\documentclass[12pt]{article}

\usepackage{microtype}
\usepackage{graphicx}
\usepackage[margin=1in]{geometry}
\usepackage{booktabs} 

\usepackage{amsthm}
\usepackage[normalem]{ulem}

\usepackage{algorithm}
\usepackage{algorithmic}
\usepackage{amsmath}
\usepackage{amssymb}
\usepackage{mathtools}
\usepackage{color}
\usepackage[colorlinks]{hyperref}
\usepackage{cleveref}
\usepackage{caption}
\usepackage{subcaption}
\usepackage{svg}
\usepackage[toc,page,header]{appendix}
\usepackage{minitoc}

\newcommand*{\poww}{\mathrm{MultPow}}

\newcommand*{\probm}{\mathcal{P}}
\newcommand*{\prob}{\mathcal{P}}
\newcommand{\cpl}{\mathcal{C}}

\newcommand{\dGW}{d_\mathrm{GW}}
\newcommand{\dW}{d_\mathrm{W}}
\newcommand{\mcms}{\mathcal{M}^\mathrm{MS}}
\newcommand{\dWLk}{d_{\mathrm{WL}}^{\scriptscriptstyle{(k)}}}

\newcommand{\dWL}{d_{\mathrm{WL}}}

\newcommand{\invs}{\mathcal{L}}
\newcommand{\diam}{\mathrm{diam}}
\newcommand{\R}{\mathbb{R}}

\newcommand*{\mbl}{\left\{\mskip-5mu\left\{}
\newcommand*{\mbr}{\right\}\mskip-5mu\right\}}

\newcommand{\lc}{\left(}
\newcommand{\rc}{\right)}

\newcommand*{\mX}{\mathcal{X}}
\newcommand*{\mY}{\mathcal{Y}}
\newcommand*{\mZ}{\mathcal{Z}}

\newcommand{\N}{\mathbb{N}}

\newcommand{\WLh}[2]{\mathfrak{l}^{\scriptscriptstyle{(#1)}}_{\scriptscriptstyle{#2}}}
\newcommand{\WLhor}[2]{\ell^{\scriptscriptstyle{(#1)}}_{\scriptscriptstyle{#2}}}

\newtheorem{theorem}{Theorem}[section]
\newtheorem{proposition}[theorem]{Proposition}
\newtheorem{corollary}[theorem]{Corollary}
\newtheorem{lemma}[theorem]{Lemma}
\newtheorem{remark}[theorem]{Remark}
\newtheorem{claim}{Claim}

\newtheorem{definition}{Definition}
\newtheorem{example}{Example}

\newcommand{\LMMCNN}[1] {{\text{\rm{MCNN}}_{#1}}}

\usepackage{authblk}

\begin{document}
\title{Weisfeiler-Lehman meets Gromov-Wasserstein}

\author[1]{Samantha Chen\thanks{\url{sac003@ucsd.edu}}}
\author[2]{Sunhyuk Lim\thanks{\url{sulim@mis.mpg.de}}}
\author[3]{Facundo Mémoli\thanks{\url{memoli@math.osu.edu}}}
\author[4]{Zhengchao Wan\thanks{\url{zcwan@ucsd.edu}}}
\author[1,4]{Yusu Wang\thanks{\url{yusuwang@ucsd.edu}}}

\affil[1]{Department of Computer Science and Engineering, University of California San Diego}
\affil[2]{Max Planck Institute for Mathematics in the Sciences, Leipzig}
\affil[3]{Department of Mathematics and Department of Computer Science and Engineering, The Ohio State University}
\affil[4]{Hal{\i}c{\i}o\u{g}lu Data Science Institute, University of California San Diego}

\date{\today}
\maketitle

\begin{abstract}
The Weisfeiler-Lehman (WL) test is a classical procedure for graph isomorphism testing.  
The WL test has also been widely used both for designing graph kernels and for analyzing graph neural networks. In this paper, we propose the \emph{Weisfeiler-Lehman (WL) distance}, a notion of distance between \emph{labeled measure Markov chains} (LMMCs), of which labeled graphs are special cases. The WL distance is polynomial time computable and is also compatible with the WL test in the sense that the former is positive if and only if the WL test can distinguish the two involved graphs. 
The WL distance captures and compares subtle structures of the underlying LMMCs and, as a consequence of this, it is more discriminating than the distance between graphs used for defining the state-of-the-art Wasserstein Weisfeiler-Lehman graph kernel. Inspired by the structure of the WL distance we identify a neural network architecture on LMMCs which turns out to be universal w.r.t. continuous functions defined on the space of all LMMCs (which includes all graphs) endowed with the WL distance. Finally, the WL distance turns out to be stable w.r.t. a natural variant of the Gromov-Wasserstein (GW) distance for comparing metric Markov chains that we identify. Hence, the WL distance can also be construed as a polynomial time lower bound for the GW distance which is in general NP-hard to compute.
\end{abstract}

\section{Introduction}\label{section:introduction}
The Weisfeiler-Lehman (WL) test \cite{leman1968reduction} is a classical procedure which provides a polynomial time proxy for testing graph isomorphism.
It is efficient and can distinguish most pairs of graphs in linear time \cite{babai1979canonical, babai1983canonical}. 
The WL test has a close relationship with graph neural networks (GNNs), both in the design of GNN architectures and in terms of characterizing their expressive power.
For example,  \cite{xu2018powerful} showed that graph isomorphism networks (GINs) have the same discriminative power as the WL test in distinguishing whether two graphs are isomorphic or not. 
Recently, \cite{azizian2020expressive} showed that message passing graph neural networks (MPNNs) are universal with respect to the continuous functions defined on the set of graphs (with the topology induced by a specific variant of the graph edit distance) that have equivalent to or less discriminative power than the WL test.

However, the WL test 
is only suitable for testing graph isomorphism  and cannot directly \emph{quantitatively} compare graphs. This state of affairs naturally suggests identifying a distance function between graphs so that two graphs have positive distance iff they can be distinguished by the WL test.
We note that there have been WL-inspired graph kernels which can quantitatively compare graphs \cite{shervashidze2011weisfeiler,togninalli2019wasserstein}. However, these either cannot handle continuous node features naturally or they do not have the same discriminative power as the WL test.

\paragraph{New work and connections to related work.} Our work provides novel connections between the WL test, GNNs and the Gromov-Wasserstein distance. The central object we define in this paper is a distance between graphs \emph{which has the same discriminative power as the WL test}. We do this by combining ideas inherent to the WL test with optimal transport (OT) \cite{villani2009optimal}. We call this distance the \emph{Weisfeiler-Lehman (WL) distance}. We show that two graphs are at zero WL distance if and only if they cannot be distinguished by the WL test.
Moreover, the WL distance can be computed in polynomial time.
Furthermore, we define our WL distance for a more general and flexible type of objects called the \emph{labeled measure Markov chains} (LMMCs), of which labeled graphs (i.e, graph with node features) are special cases. Besides graphs, the LMMC framework also encompasses continuous objects such as Riemannian manifolds and graphons.

Our definition of the WL distance is able to capture and compare subtle geometric and combinatorial structures from the underlying LMMCs. This allows us to establish various lower bounds for the WL distance which are not just useful in practical computations but also clarify its discriminating power relative to existing approaches. Furthermore, based on the hierarchy inherent to the WL distance, we are able to identify a neural network architecture on the collection of all LMMCs, which we call $\LMMCNN{}$s (for Markov chain NNs). We show that $\LMMCNN{}$s have the same discriminative power as the WL test when applied to graphs; while at the same time, they have the desired universal approximation property w.r.t. continuous functions defined on the space of all LMMCs (including the space of graphs) equipped with the WL distance.
It turns out that from $\LMMCNN{}$s, one can recover  Weisfeiler-Lehman graph kernels \cite{togninalli2019wasserstein} and in particular, we show that a slight variant of a key pseudo-distance between graphs defined in \cite{togninalli2019wasserstein} serves as a lower bound for our WL distance. This indicates that the WL distance has a stronger discriminating ability  than WWL graph kernels (see \Cref{sec:WWL} for an infinite family of examples).

Finally, we observe that our formulation of the WL distance resembles the Gromov-Wasserstein (GW) distance \cite{memoli2007on,memoli2011gromov,peyre2016gromov,sturm2012space,vayer2020fused,chowdhury2019gromov} which is a OT-based distance between \emph{metric measure spaces} and has been recently widely used in shape matching and machine learning. We hence identify a special variant of the GW distance between \emph{Markov chain metric spaces} (MCMSs) (including all graphs). Our version of the GW distance implements a certain \emph{multiscale comparison} of MCMSs, it vanishes only when the two MCMSs are isomorphic, but leads to NP-hard problems. Interestingly, it turns out that the poly-time computable WL distance is not only stable w.r.t. (i.e., upper bounded by) this variant of the GW distance, but can also be construed as a variant of the \emph{third lower bound} (TLB) of this GW distance, as in \cite{memoli2011gromov}.

Proofs of results and details can be found in the Appendix.

\section{Preliminaries}\label{sec:pre}
\subsection{The Weisfeiler-Lehman test}\label{sec:WL test}
A labeled graph is a graph $G=(V_{G},E_{G})$ endowed with a \emph{label function} $\ell_{G}:V_{G}\rightarrow Z$, where the labels (i.e., node features) are taken from some set $Z$.
Common label functions include the degree label (i.e., $\ell_G:V_G\rightarrow\N$ sends each $v\in V_G$ to its degree, denoted by $\deg_G(v)$) and the constant label (assigning a constant to all vertices).  For a node $v\in V_{G}$, let $N_{G}(v)$ denote the set of neighbors of $v$ in $G$. Below, we describe the \emph{Weisfeiler-Lehman hierarchy} for a given labeled graph $(G,\ell_{G})$.

\begin{definition}[Weisfeiler-Lehman hierarchy]\label{def:hierarchy set}
Given any labeled graph $(G,\ell_G)$, we consider the following hierarchy of multisets, which we call the Weisfeiler-Lehman hierarchy:
\begin{description}

\item[Step $1$] For each $v\in V_{G}$ we compute the pair
    \[\WLhor{1}{(G,\ell_G)}(v)\coloneqq\lc \ell_{G}(v),\mbl\ell_{G}(v'):\,v'\in N_{G}(v)\mbr\rc.\]
\item[$\cdots$]
\item[Step $k$] For each $v\in V$ we compute the pair
    \begin{align*} 
    \WLhor{k}{(G,\ell_G)}(v)\coloneqq
       \lc \WLhor{k-1}{(G,\ell_G)}(v),\mbl\WLhor{k-1}{(G,\ell_G)}(v'):\,v'\in N_{G}(v)\mbr\rc.
    \end{align*}
\end{description}
\end{definition}
Here, $\mbl\cdot \mbr$ denotes multisets. In the  literature, $\WLhor{k}{(G,\ell_G)}(v)$ is usually often mapped to a common space of labels such as $\N$ through a hash function, a step which we do not require in this paper. 
We  induce, at each step $k$, a multiset
\[L_k\!\lc(G,\ell_{G})\rc\coloneqq\mbl \WLhor{k}{(G,\ell_{G})}(v):\,v\in V_{G}\mbr.\]

\begin{definition}[Weisfeiler-Lehman test]\label{def:WL test}
For each integer $k\geq 0$, we compare $L_k\!\lc(G_1,\ell_{G_1})\rc$ with $L_k\!\lc(G_2,\ell_{G_2})\rc$. If $\exists k\geq 0$ so that $L_k\!\lc(G_1,\ell_{G_1})\rc\neq L_k\!\lc(G_2,\ell_{G_2})\rc$ then we conclude that the two label graphs are non-isomorphic; otherwise we say that the two labeled graphs pass the WL test and that the two graphs are ``possibly isomorphic''.
\end{definition}

\subsection{Probability measures and optimal transport}\label{sec:coupling and dW}
For any measurable space $Z$, we will denote by $\prob(Z)$ the collection of all probability measures on $Z$. When $Z$ is a metric space $(Z,d_Z)$, we further require that every $\alpha\in\prob(Z)$ has finite 1-moment, i.e., $\int_Zd_Z(z,z_0)\alpha(dz)<\infty$ for any $\alpha\in\probm(Z)$ and any fixed $z_0\in Z$.

\paragraph{Pushforward maps.} Given two measurable spaces $X$ and $Y$ and a measurable map $\psi:X\rightarrow Y$, the \emph{pushforward} map induced by $\psi$ is the map
$\psi_\# :\prob(X)\rightarrow \prob(Y)$
sending $\alpha$ to $\psi_\#\alpha$ where for any measurable $B\subseteq Y$, $\psi_\#\alpha (B)\coloneqq\alpha\lc \psi^{-1}(B)\rc.$ In the case when $X$ is finite and $Y$ is a metric space, $\psi_\#\alpha$ obviously has finite 1-moment and is thus an element of $\probm(Y)$.

\paragraph{Couplings and the Wasserstein distance.}
For measurable spaces $X$ and $Y$, given $\alpha\in \prob(X)$ and $\beta\in \prob(Y)$, $\gamma\in\prob(X\times Y)$ is called a \emph{coupling} between $\alpha$ and $\beta$ if $(p_X)_\#\gamma=\alpha\text{ and }(p_Y)_\#\gamma=\beta,$
where $p_X:X\times Y\rightarrow X$ and $p_Y:X\times Y\rightarrow Y$ are the canonical projections, e.g., the product measure $\alpha\otimes\beta$ is one such coupling.
Let $\mathcal{C}(\alpha,\beta)$ denote the set of all couplings between $\alpha$ and $\beta$.

Given a metric space $(Z,d_Z)$, for $\alpha,\beta\in\probm(Z)$, we define the ($\ell^1$-)Wasserstein distance between them as follows:
\[\dW(\alpha,\beta)\coloneqq\inf_{\gamma\in\mathcal{C}(\alpha,\beta)}\int_{Z\times Z}d_Z(z,z')\gamma(dz\times dz').\]
By \cite[Proposition 2.1]{villani2009optimal}, the infimum above is always achieved by some $\gamma\in\mathcal{C}(\alpha,\beta)$ which we call an \emph{optimal coupling} between $\alpha$ and $\beta$.

\paragraph{Hierarchy of probability measures.}
An important ingredient in this paper is the following construction: Given a \emph{finite} set $X$ and a \emph{metric space} $Z$, a map of the form $\psi:X\rightarrow\probm(Z)$ induces $\psi_\#:\prob(X)\rightarrow\probm(\probm(Z))$ which involves the space of probability measures over probability measures, i.e., $\probm(\probm(Z))$. Inductively, we define the family of spaces $\probm^{\circ k}(Z)$, called the \emph{hierarchy of probability measures}:
\begin{enumerate}
    \item $\probm^{\circ 1}(Z)\coloneqq\prob(Z)$;
    \item $\probm^{\circ (k+1)}(Z)\coloneqq\probm\lc\probm^{\circ k}(Z)\rc$ for $k\geq 1$.
\end{enumerate}
If $Z$ is complete and separable then, when endowed with $\dW$, $\probm(Z)$ is also complete and separable (\cite[Theorem 6.18]{villani2009optimal}). 

By induction, for each $k\in \N$, $\probm^{\circ k}(Z)$ is also a complete and separable metric space. This hierarchy will be critical  in our development of the WL distance.

\paragraph{The Gromov-Wasserstein distance.} We call a triple $\mathbf{X}=(X,d_X,\mu_X)$ a \emph{metric measure space} (MMS) if $(X,d_X)$ is a metric space and $\mu_X$ is a (Borel) probability measure on $X$ with full support. Given any $\mathbf{X}=(X,d_X,\mu_X)$ and $\mathbf{Y}=(Y,d_Y,\mu_Y)$, for any coupling $\gamma\in\cpl(\mu_X,\mu_Y)$, we define its distortion by
    \[\mathrm{dis}(\gamma)\coloneqq \int\limits_{X\times Y}\int\limits_{X\times Y}|d_X(x,x')-d_Y(y,y')|\gamma(dx\times dy)\gamma(dx'\times dy').\]

Then, the ($\ell^1$-)Gromov-Wasserstein (GW) distance between $\mathbf{X}=(X,d_X,\mu_X)$ and $\mathbf{Y}=(Y,d_Y,\mu_Y)$ is defined as follows \cite{memoli2011gromov}
\begin{equation}\label{eq:GWdist}
    \dGW(\mathbf{X},\mathbf{Y})\coloneqq
 \inf_{\gamma\in \cpl(\mu_X,\mu_Y)}\mathrm{dis}(\gamma),
\end{equation}
where we omit the usual $\frac{1}{2}$ factor for simplicity.

\subsection{Markov chains} \label{sec:markov chains}
Given a \emph{finite} set $X$, we call any map ${m_\bullet^X:X\rightarrow \prob(X)}$ a \emph{Markov kernel} on $X$. Of course Markov kernels can be represented as transition matrices but we adopt this more flexible language. A probability measure $\mu_X\in\prob(X)$ is called a \emph{stationary distribution w.r.t. $m_\bullet^X$} if for every measurable subset $A\subseteq X$ we have:
$$\mu_X(A)=\int_{ X}\,m_{x}^X(A)\mu(dx).$$
The existence of stationary distributions is guaranteed by the Perron-Frobenius Theorem \cite{saloff1997lectures}.
A \emph{measure Markov chain (MMC)} is any tuple $\mX=(X,m_\bullet^X,\mu_X)$ where $X$ is a finite set, $m_\bullet^X$ is a Markov kernel on $X$ and $\mu_X$ is a fully supported stationary distribution w.r.t. $m_\bullet^X$.

\begin{definition}[Labeled measure Markov chain]
Given any metric space $Z$, which we refer to as the \emph{metric space of labels,} a \emph{$Z$-labeled measure Markov chain} (($Z$-)LMMC for short) is a tuple $(\mX,\ell_X)$ where $\mX$ is a MMC and $\ell_X:X\rightarrow Z$ is a continuous map. For technical reasons, throughout this paper, we assume that the metric space of labels $Z$ is \emph{complete} and \emph{separable}. We let $\mathcal{M}^{L}(Z)$ denote the collection of all $Z$-LMMCs.
\end{definition}

The following definition of isomorphism between LMMCs is similar to that of labeled graph isomorphism.

\begin{definition}
Two $Z$-LMMCs $(\mX,\ell_X)$ and $(\mY,\ell_Y)$ are said to be \emph{isomorphic} if there exists a bijective map $\psi:X\rightarrow Y$ such that $\ell_X(x)=\ell_Y(\psi(x))$ and $\psi_\# m_x^X = m_{\psi(x)}^Y$ for all $x\in X$ and $ \psi_\#\mu_X=\mu_Y$. 
\end{definition}

\paragraph{Labeled graphs as LMMCs.}
Any labeled graph induces a family of LMMCs which we explain as follows.

\begin{definition}[$q$-Markov chains on graphs]\label{def:markov chain on graphs}
For any graph $G$ and parameter $q\in[0,1)$, we define the $q$-Markov chain $m^{G,q}_\bullet$ associated to $G$ as follows:  for any $v\in V_G$,
\[m^{G,q}_v\coloneqq\begin{cases}q\,\delta_v+\frac{1-q}{\mathrm{deg}_G(v)}\sum_{v'\in N_G(v)}\delta_{v'}, &N_G(v)\neq\emptyset\\
\delta_v,&N_G(v)=\emptyset\end{cases}.\]
\end{definition}

We further let $\overline{\deg}_G(v)\coloneqq\deg_G(v)$ if $N_G(v)\neq\emptyset$ and $\overline{\deg}_G(v)\coloneqq 1$ otherwise. Then, it is easy to see that
$$\mu_G\coloneqq\sum_{v\in V_G}\frac{\overline{\deg}_G(v)}{\sum_{v'\in V_G}\overline{\deg}_G(v')}\delta_v$$
is a stationary distribution for $m_\bullet^{G,q}$ for all $q\in[0,1]$.

For any $q\in[0,1)$, we let $\mX_q(G):=\lc V_G,m_\bullet^{G,q},\mu_G\rc$ and call $\lc\mX_q(G),\ell_G\rc$ a \emph{graph induced LMMC}. When $q=0$, we also let $m_\bullet^{G}\coloneqq m_\bullet^{G,q}$ and let $\mX(G)\coloneqq\mX_0(G)$.
One has the following desirable property for graph induced LMMCs.

\begin{proposition}\label{prop:iso of graphs}
For any $q\in[0,1)$,$(G_1,\ell_{G_1})$ is isomorphic to $(G_2,\ell_{G_1})$ as labeled graphs iff $\lc\mX_q(G_1),\ell_{G_1}\rc$ is isomorphic to $\lc\mX_q(G_2),\ell_{G_2}\rc$ as LMMCs.
\end{proposition}

\section{The WL distance}\label{sec:WL distance}
A non-empty finite multiset $M$ of elements from a given set $S$ encodes information about the \emph{multiplicity} of each $s\in S$ in $M$. This suggests that one might consider the \emph{probability measure} $\mu_M$ on $S$ induced by $M$:
$$\mu_M(s)\coloneqq\frac{m(s)}{\sum_{t\in S}m(t)},\,\,\forall s\in S$$
where $m(s)$ denotes the multiplicity of $s$ in $S$.
This point of view permits reinterpreting the multisets appearing in the WL hierarchy (cf. \Cref{def:hierarchy set}) through the language of probability measures, which will eventually lead us to a distance between graphs. 

\begin{definition}[Weisfeiler-Lehman measure hierarchy]\label{def:hierarchy prob}
Given any $Z$-LMMC $(\mX,\ell_X)$, we let $\WLh{0}{(\mX,\ell_X)}\coloneqq\ell_X$ and produce the following label functions whose codomains span a certain hierarchy of probability measures:
\begin{description}
\itemsep-0.2em
\item[Step $1$] For each $x\in X$, we have $(\WLh{0}{(\mX,\ell_X)})_\# m_x^X \in \probm(Z)$. Hence we in fact have the function
$$\WLh{1}{(\mX,\ell_X)}\coloneqq\lc\WLh{0}{(\mX,\ell_X)}\rc_\# m_\bullet^X:X\rightarrow \probm(Z).$$ \item[$\cdots$]
\item[Step $k$] For each integer $k\geq 2$, we inductively define
$$\WLh{k}{(\mX,\ell_X)}\coloneqq\lc\WLh{k-1}{(\mX,\ell_X)}\rc_\# m_\bullet^X:X\rightarrow \probm^{\circ k}(Z).$$
\end{description}
\end{definition}
We then induce at each step $k$ a probability measure
$$\mathfrak{L}_k\!\lc(\mX,\ell_X)\rc\coloneqq \lc \WLh{k}{(\mX,\ell_X)}\rc_\#\mu_X\in\prob^{\circ(k+1)}(Z).$$
$\WLh{k}{(\mX,\ell_X)}$ should be compared to $\WLhor{k}{(G,\ell_G)}$ and $\mathfrak{L}_k\!\lc(\mX,\ell_X)\rc$ should be compared to $L_k\!\lc(G,\ell_G)\rc$ from the WL hierarchy (cf. \Cref{def:hierarchy set}). See \Cref{fig:example_1} for an illustration of the WL measure hierarchy of a graph induced LMMC and its comparison with the corresponding WL hierarchy. We will show later that, up to certain change of labels, the WL measure hierarchy for a graph induced LMMC captures all the information contained in the WL hierarchy of the original graph (cf. \Cref{prop:WL vs dwl}).

\begin{figure*}
    \centering
    \includegraphics[scale=0.8]{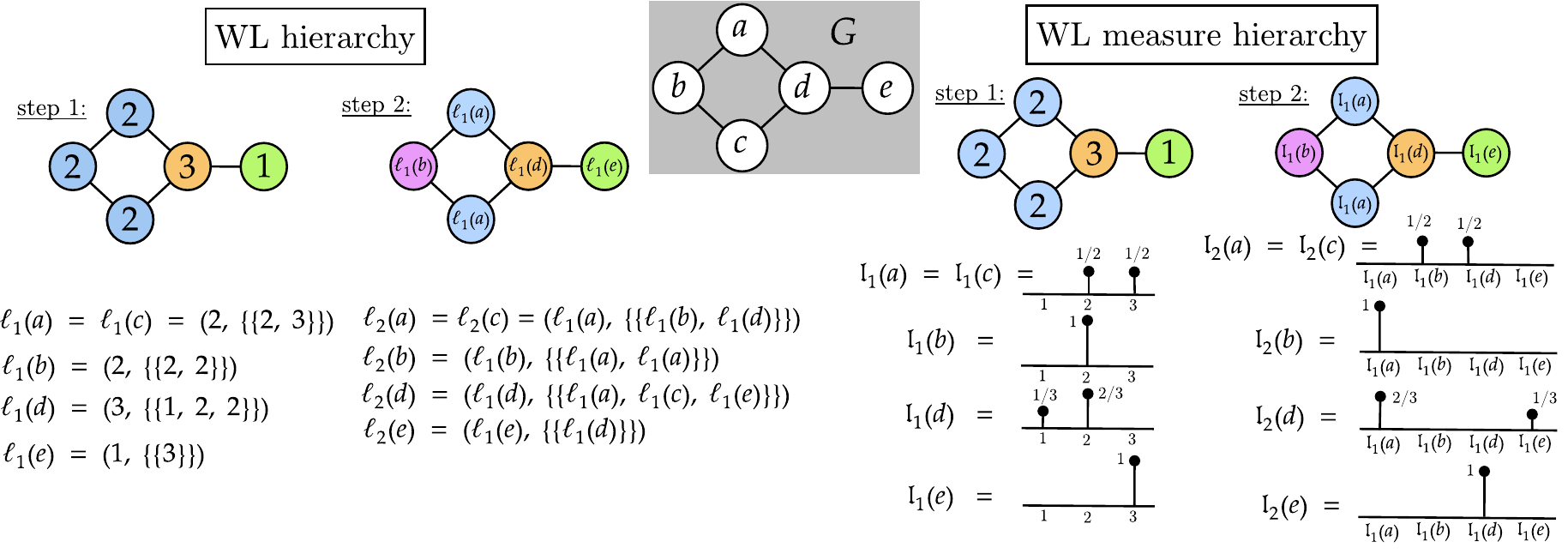}
    \vskip -0.1in
    \caption{\textbf{Illustration of the WL (measure) hierarchy.} The graph $G$ shown in the middle of the figure is assigned the degree label $\ell_G$. We explicitly present two steps of the WL hierarchy of $(G,\ell_G)$ (on the left) and of the WL measure hierarchy of $(\mX(G),\ell_G)$ (on the right). Every probability measure is represented as a histogram. For $i=1,2$, $\ell_i$ and $\mathfrak{l}_i$ are abbreviations for $\WLhor{i}{(G,\ell_G)}$ and $\WLh{i}{(\mX(G),\ell_G)}$, respectively. Notice how the WL measure hierarchy interprets the multisets from the WL hierarchy as probability measures.}
    \label{fig:example_1}
\end{figure*}

We now define the Weisfeiler-Lehman distance based on the WL measure hierarchy.

\begin{definition}[Weisfeiler-Lehman distance]\label{def:WL metric}
For each integer $k\geq 0$ and any metric space of labels $Z$ we define the \emph{Weisfeiler-Lehman (WL) distance of depth $k$} between the $Z$-LMMCs $(\mX,\ell_X)$ and $(\mY,\ell_Y)$ as 
\begin{equation}\label{eq:definition of dwlk}
    \dWLk\!\lc(\mX,\ell_X),(\mY,\ell_Y)\rc\coloneqq\dW\!\lc \mathfrak{L}_k\!\lc(\mX,\ell_X)\rc,\mathfrak{L}_k\!\lc(\mY,\ell_Y)\rc\rc
\end{equation}
where $\dW$ above takes place in $\prob^{\circ k}(Z)$. We also define the (absolute) \textbf{Weisfeiler-Lehman distance} by
$$\dWL\!\lc(\mX,\ell_X),(\mY,\ell_Y)\rc \coloneqq\sup_{k\geq 0} \dWLk\!\lc(\mX,\ell_X),(\mY,\ell_Y)\rc.$$
\end{definition}

\begin{example}\label{ex:k=0 and 1}
We write down explicit formulas for $\dWLk$ when $k=0$ and $1$.
When $k=0$, it is easy to see that 
\[\dWL^{\scriptscriptstyle{(0)}}\!\lc(\mX,\ell_X),(\mY,\ell_Y)\rc=\dW\!\lc (\ell_X)_\#\mu_X,(\ell_Y)_\#\mu_Y\rc,\]
which agrees with the Wasserstein distance between the \emph{global} label distributions $(\ell_X)_\#\mu_X$ and $(\ell_Y)_\#\mu_Y$ (cf. a similar concept for MMSs \cite{memoli2011gromov}).

When $k=1$, we have that
\[\dWL^{\scriptscriptstyle{(1)}}((\mX,\ell_X),(\mY,\ell_Y))=\inf_{\gamma\in \cpl(\mu_X,\mu_Y)} \int\limits_{X\times Y}\dW\!\lc (\ell_X)_\#m_x^{X},(\ell_Y)_\#m_y^{Y}\rc \,\gamma(dx\times dy),\]
 implementing the comparison of \emph{local} label distributions.
\end{example}

The following proposition states that the WL distance becomes more discriminating as the depth increases.

\begin{proposition}\label{coro:hierarchy dwlk}
Let $k\geq 0$ be any integer. Given any two $Z$-LMMCs $(\mX,\ell_X)$ and $(\mY,\ell_Y)$, we have that
$d_{\mathrm{WL}}^{\scriptscriptstyle{(k)}}\!\lc(\mX,\ell_X),(\mY,\ell_Y)\rc\leq d_{\mathrm{WL}}^{\scriptscriptstyle{(k+1)}}\!\lc(\mX,\ell_X),(\mY,\ell_Y)\rc$.
\end{proposition}

As a first step towards understanding the WL distance, we show 
that $\dWL$ is a pseudo-distance. We discuss its relationship with the WL test in \Cref{prop:WL vs dwl} below.  
\begin{proposition}\label{prop:dwl is a pseudometric}
$\dWL$ (resp. $\dWLk$ for $k\geq 0$) defines a pseudo-distance\footnote{By pseudo-distance, we mean that $\dWL$ (resp. $\dWLk$) is symmetric and satisfies the triangle inequality, but non-isomorphic LMMCs can have zero $\dWL$ (resp. $\dWLk$) distance.} on the collection $\mathcal{M}^{L}(Z)$.
\end{proposition}

\subsection{Comparison with the WL test}\label{sec:comparison with WL test}
Given the apparent similarity between the WL hierarchy and the WL measure hierarchy, it should not be surprising that (as we will show later in \Cref{prop:WL vs dwl}), on graphs, $\dWL$ essentially has the same discriminative power as the WL test, i.e., those pairs of graphs which can be distinguished by the WL test are the same as those with $\dWL>0$. The WL distance therefore should be interpreted as a quantification of the \emph{degree} to which the graphs fail to pass the WL test.

However, there is an apparent loss of information in the WL measure hierarchy due to the normalization inherent in probability measures. This could result in certain cases when $\dWL((\mX_q(G_1),\ell_{G_1}),(\mX_q(G_2),\ell_{G_2}))=0$ but $(G_1,\ell_{G_1})$ and $(G_2,\ell_{G_2})$ are distinguished by the WL test. See the examples in \Cref{sec:examples}. However, as we explain next, with appropriate label transformations, the discriminative power of $\dWL$ is the same as that of the WL test.

For any metric space of labels $Z$, consider any \emph{injective} map $g:Z\times \mathbb{N}\times \mathbb{N}\rightarrow Z_1$ where $Z_1$ is another metric space of labels.
A trivial example of such injective $g$ is given by letting $Z_1\coloneqq Z\times \mathbb{N}\times \mathbb{N}$ and letting $g$ be the identity map.

Now, given any labeled graph $(G,\ell_G:V_G\rightarrow Z)$, we generate a new label function $\ell_G^g\coloneqq g(\ell_G,\deg_G(\bullet),|V_G|):V_G\rightarrow Z_1$.
Intuitively, this is understood as relabeling $G$ via the map $g$.
Modulo this change of label function, we establish that $\dWL$ has the \emph{same} discriminative power as the WL test.

\begin{proposition}\label{prop:WL vs dwl}
For any $q\in(\frac{1}{2},1)$, the WL test distinguishes two labeled graphs $(G_1,\ell_{G_1})$ and $(G_2,\ell_{G_2})$ iff $\dWL\!\lc \lc\mX_q(G_1),\ell_{G_1}^g\rc,\lc\mX_q(G_2),\ell_{G_2}^g\rc\rc > 0$.
\end{proposition}

Although it may seem that we are injecting more information into labels, this extra relabeling, 
in fact, does not affect the outcome of the WL test:
\begin{lemma}\label{lm:two labels are the same}
The WL test distinguishes $(G_1,\ell_{G_1})$ and $(G_2,\ell_{G_2})$ iff it distinguishes $(G_1,\ell_{G_1}^g)$ and $(G_2,\ell_{G_2}^g)$.
\end{lemma}

By \Cref{prop:WL vs dwl} and a convergence result pertaining to the WL test \cite{krebs2015universal}, one has the following convergence result for $\dWLk$ (which implies that to determine whether two graph induced LMMCs satisfy $\dWL=0$ one only needs to inspect $\dWLk$ for finitely many $k$.)

\begin{corollary}\label{cor:dWL convergence}
For any $q\in(\frac{1}{2},1)$ and any two labeled graphs $(G_1,\ell_{G_1})$ and $(G_2,\ell_{G_2})$, if 
    $\dWLk\!\lc \lc\mX_q(G_1),\ell_{G_1}^g\rc,\lc\mX_q(G_2),\ell_{G_2}^g\rc\rc = 0$
holds for each $k=0,\ldots, \left|V_{G_1}\right|+\left|V_{G_2}\right|$, we then have that 
$\dWL\!\lc \lc\mX_q(G_1),\ell_{G_1}^g\rc,\lc\mX_q(G_2),\ell_{G_2}^g\rc\rc = 0$.
\end{corollary}

Results from \cite{babai1979canonical} imply that the WL test can certify isomorphism of random graphs (with degree labels) with high probability.
Then, immediately by \Cref{prop:WL vs dwl}, we have that with high probability $\dWL$ generates positive distance for non-isomorphic random-graph-induced LMMCs. 

\subsection{A lower bound for $\dWLk$}\label{sec:1 lower bound for dwlk}
The WL measure hierarchy, defined through consecutive steps of pushforward maps, can be related to a certain sequence of Markov kernels which we explain next. Given a MMC $\mX=(X,m_\bullet^X,\mu_X)$ and any $k\in\mathbb{N}$, the $k$-step Markov kernel, denoted by $m_\bullet^{X,\otimes k}$, is defined inductively as follows: for any $x \in X$, when $k = 1$, $m_x^{X \otimes 1} \coloneqq m_x^{X}$ and when $k \geq 2$, for $A \subseteq X$,
\[m_x^{X,\otimes k}(A)\coloneqq\int_Xm_{x'}^{X,\otimes(k-1)}(A)\,m_x^X(dx').
\]
If we represent $m_\bullet^X$ by a transition matrix $M_{\mX }$, then the  matrix corresponding to $m_\bullet^{X,\otimes k}$ is the $k$-power of $M_{\mX }$.

Recall the formula for $\dWL^{\scriptscriptstyle{(1)}}$ from \Cref{ex:k=0 and 1}. We define a certain quantity (for each $k\in\N$) which arises by replacing the Markov kernels in that formula with $k$-step Markov kernels:
\[d_{\mathrm{WLLB}}^{\scriptscriptstyle{(k)}}((\mX,\ell_X),(\mY,\ell_Y))\coloneqq\inf_{\gamma\in \cpl(\mu_X,\mu_Y)} \int\limits_{X\times Y}\dW\!\lc (\ell_X)_\#m_x^{X,\otimes k},(\ell_Y)_\#m_y^{Y,\otimes k}\rc \gamma(dx\times dy).\]
Notice that $\dW$ above takes place in $\prob(Z)$ whereas $\dW$ in \Cref{eq:definition of dwlk} for defining $\dWLk$ takes place in $\prob^{\circ k}(Z)$.
Of course, we have that
$d_{\mathrm{WLLB}}^{\scriptscriptstyle{(1)}}=\dWL^{\scriptscriptstyle{(1)}}$.  
It turns out that for each $k\geq1$, $d_{\mathrm{WLLB}}^{\scriptscriptstyle{(k)}}$ is a lower bound for $\dWLk$.

\begin{proposition}\label{prop:lower bound}
For any $(\mX,\ell_X),(\mY,\ell_Y) \in \mathcal{M}^{L}(Z)$ and any integer $k\geq 1$ we have that
$ d_{\mathrm{WLLB}}^{\scriptscriptstyle{(k)}}\!\lc(\mX,\ell_X),(\mY,\ell_Y)\rc\leq\dWLk\!\lc(\mX,\ell_X),(\mY,\ell_Y)\rc.$
\end{proposition}

For any fixed $k\in\N$ and any two finite $\R$-LMMCs, $(\mX, \ell_X)$ and $(\mY, \ell_Y)$, if we let $n \coloneqq \max(|X|, |Y|)$, then computing $\dWLk((\mX, \ell_X),(\mY, \ell_Y))$ can be done in $O(n^5 \log(n)\,k)$ time whereas the total time complexity for computing $d_{\mathrm{WLLB}}^{\scriptscriptstyle{(k)}}((\mX, \ell_X),(\mY, \ell_Y))$ is $O(n^3 \log(nk))$. Hence it is far more efficient to compute $d_{\mathrm{WLLB}}^{\scriptscriptstyle{(k)}}$ than $\dWLk$. Details can be found in \Cref{sec:algorithm and analysis}.

\section{WL distance inspired neural networks}\label{sec:NN}
We now focus on the case when the metric space of labels $Z$ is Euclidean, i.e., $Z=\R^d$ and define a family of real functions on $\mathcal{M}^{L}(\R^d)$ called \emph{Markov chain neural networks ($\LMMCNN{}$s)}. We both study the discriminative power and establish a universality result for this family of functions.

For any \emph{Lipschitz} function $\varphi:\R^i\rightarrow\R^j$, we define the map 
$q_\varphi:\probm(\R^i)\rightarrow\R^j$
sending $\alpha\in \probm(\R^i)$ to the average $\int_{\R^i}\varphi(x)\alpha(dx)$. Based on $q_\varphi$, we define two types of maps:

(1) $F_\varphi:\mathcal{M}^{L}(\R^i)\rightarrow \mathcal{M}^{L}(\R^j)$
sending $(\mX,\ell_X)$ to $(\mX,\ell_X^\varphi)$, where $\ell_X^\varphi:X\rightarrow \R^j$ is defined by $x\mapsto q_\varphi((\ell_X)_\#m_x^X)$.

(2) $S_\varphi:\mathcal{M}^{L}(\R^i)\rightarrow\R^j$ sending $(\mX,\ell_X)$ to $q_\varphi((\ell_X)_\#\mu_X)$. 

Then, for any sequence of \emph{Lipschitz} maps $\varphi_i:\R^{d_{i-1}}\rightarrow\R^{d_{i}}$ for $i=1,\ldots,k+1$, and any \emph{continuous} map $\psi:\R^{d_{k+1}}\rightarrow \R$, we define a map of the following form, which we call a $k$-layer \emph{Markov chain neural network ($\LMMCNN{k}$)}: 
\begin{equation}\label{eq:LMMCNN definition}
   \psi\circ S_{\varphi_{k+1}}\circ F_{\varphi_k}\circ \cdots\circ F_{\varphi_1}:\mathcal{M}^{L}(\R^d)\rightarrow \mathbb{R}.
\end{equation}

Note the resemblance between our MCNNs and message passing neural networks (MPNNs) for graphs \cite{gilmer2017neural}:  Specifically, $(\ell_X)_\#m_x^X$ is analogous to the {\sc Aggregation} operation, $q_\phi$ is analogous to the {\sc Update} operation, and $\psi\circ S_\phi$ corresponds to the readout function that appear in the context of MPNN. 

\begin{example}[Relation with WWL graph kernels]\label{ex:relation with wwl}
MCNNs recover the framework of  Wasserstein Weisfeiler-Lehman (WWL) graph kernels w.r.t. continuous attributes \cite{togninalli2019wasserstein}:
Consider any labeled graph $(G,\ell_G:V_G\rightarrow\R^d)$ and any $q\in[0,1)$. Let $\varphi:\R^d\rightarrow\R^d$ be any continuous map. Applying $q_{\varphi}$ to $(\mX_q(G),\ell_{G})$ (cf. \Cref{def:markov chain on graphs}), then for any $v\in V_G$ such that $N_G(v)\neq\emptyset$, we have
\begin{align*}
    &\ell_G^{\varphi}(v)=\int_{\R^d}\varphi(t)\,(\ell_G)_\#m_v^{G,q}(dt)\\
    &=q\,\varphi(\ell_G(v))+\frac{1-q}{\deg_G(v)}\sum_{v'\in N_G(v)}\!\!\varphi(\ell_G(v')).
\end{align*}

Notice that if we further let $q=\frac{1}{2}$ and $\varphi:\R^d\rightarrow\R^d$ be the identity map $\mathrm{id}$, then we have
\begin{equation}\label{eq:WWL}
   \ell_G^{\mathrm{id}}(v)=\frac{1}{2}\!\lc\,\ell_G(v)+\frac{1}{\deg_G(v)}\sum_{v'\in N_G(v)}\!\!\ell_G(v')\rc. 
\end{equation}
This is exactly how labels are updated in the WWL graph kernel framework. A slight modification of the ground distance computation in the WWL framework generates a lower bound for $\dWLk$, which implies that the WL distance is more capable at discriminating labeled graphs than WWL graph kernels. This is confirmed by the examples in \Cref{sec:WWL}. 
\end{example}

We let $\mathcal{N\!N}_k(\R^d)$ denote the collection of all $\LMMCNN{k}$ (cf. \Cref{eq:LMMCNN definition}). Below, we show that $\mathcal{N\!N}_k(\R^d)$ has the same discriminative power as the WL distance. 

\begin{proposition}\label{prop:zero set of NN}
Given any $(\mX,\ell_X),(\mY,\ell_Y)\in \mathcal{M}^{L}(\R^d)$, 
\begin{enumerate}
    \item if $\dWLk\!\lc(\mX,\ell_X),(\mY,\ell_Y)\rc =0$, then for every $h\in \mathcal{N\!N}_k(\R^d)$ one has that $h((\mX,\ell_X))=h((\mY,\ell_Y))$;
    \item if $\dWLk\!\lc(\mX,\ell_X),(\mY,\ell_Y)\rc >0$, then there exists $h\in \mathcal{N\!N}_k(\R^d)$ such that $h((\mX,\ell_X))\neq h((\mY,\ell_Y))$.
\end{enumerate}
\end{proposition}

Recall from \Cref{cor:dWL convergence} that given any $q\in(\frac{1}{2},1)$, any injective map $g$ and any labeled graphs $(G_1,\ell_{G_1})$ and $(G_2,\ell_{G_2})$, we need at most $2n$ steps to determine whether $\dWL(\lc\mX_q(G_1),\ell_G^g\rc, \lc\mX_q(G_2),\ell_{G_2}^g\rc) = 0$, where $n=\max\lc\left|V_{G_1}\right|,\left|V_{G_2}\right|\rc$. Consequently,  $\LMMCNN{}$s have the same discriminative power as the WL test:
\begin{corollary}\label{cor:separationgraphs}
For any $\frac{1}{2}<q<1$, the WL test distinguishes the labeled graphs $(G_1, \ell_{G_1})$ and $(G_2, \ell_{G_2})$ iff there exists $h \in \mathcal{N\!N}_{2n}(\R^d)$ for which $h\!\lc\lc\mX_q(G_1),\ell_{G_1}^g\rc\rc \neq h\!\lc\lc\mX_q(G_2),\ell_{G_2}^g\rc\rc$. 
\end{corollary}
Since MPNNs also have the same discriminative power as the WL test \cite{xu2018powerful}, we know that our $\LMMCNN{}$s can separate all pairs of graphs that MPNNs can separate.

In general, a pseudometric space canonically induces a metric space by identifying points at 0 distance (cf. \cite[Proposition 1.1.5]{burago2001course}). We let $\mathcal{M}^{L}_k(\R^d)$ denote the metric space induced by the pseudometric space $(\mathcal{M}^{L}(\R^d),\dWLk)$. As a direct consequence of \Cref{prop:zero set of NN}, every $h\in\mathcal{N\!N}_k(\R^d)$ induces a real function (which we still denote by $h$) in $\mathcal{M}^{L}_k(\R^d)$. Then, our $\LMMCNN{}$s are actually universal w.r.t. continuous functions defined on $\mathcal{M}^{L}_k(\R^d)$ (see the proof in \Cref{app:proof universal}). 
\begin{theorem}\label{thm:universality}
For any $k\in\N$, let $\mathcal{K}\subseteq\mathcal{M}^{L}_k(\R^d)$ be any compact subspace. Then\footnote{For simplicity of notation, we still use $\mathcal{N\!N}_k(\R^d)$ to denote induced functions on $\mathcal{M}^{L}_k(\R^d)$ with domain restricted to $\mathcal{K}$.},
$\overline{\mathcal{N\!N}_k(\R^d)}=C(\mathcal{K},\mathbb{R})$.
\end{theorem}
Our universality result resembles the one established in \cite{azizian2020expressive} for message passing neural networks (MPNNs) which proves that MPNNs can universally approximate continuous functions on graphs with bounded size which are less or equally as discriminative as the WL test. Compared with their result, we remark that our universality result applies to the collection of all LMMCs (and hence all graphs) \emph{with no restriction on their size}. 
Moreover, although for simplicity LMMCs are restricted to finite spaces  throughout the paper, our MCNNs and universality result can potentially be extended to more general LMMCs including continuous objects such as manifolds and graphons.

\section{Relationship with the GW distance}\label{sec:dGW}
Given a LMMC $(\mX,\ell_X)$, the label $\ell_X$ induces the following pseudo-distance on $X$:
$d_X(x,x')\coloneqq d_Z(\ell_X(x),\ell_X(x'))$ for $x,x'\in X.$
This suggests  structures which are closely related to LMMCs: \emph{Markov chain metric spaces} (MCMSs for short). A MCMS is any tuple $(\mX,d_X)$ where $\mX=(X,m_\bullet^X,\mu_X)$ is a finite MMC and $d_X$ is a proper distance on $X$.
Obviously, endowing a MCMS $(\mX,d_X)$ with a label function $\ell_X$ and forgetting $d_X$ produces a LMMC $(\mX,\ell_X)$.
We  let $\mcms$ denote the collection of all MCMSs. We now construct a Gromov-Wasserstein type distance between MCMSs and study its relationship with the WL distance.

\subsection{The GW distance between MCMSs}
Recall from \Cref{eq:GWdist} the definition of the standard GW distance between MMSs.
Intuitively, in order to identify a suitable GW-like distance between MCMSs, we would like to incorporate a comparison between Markov kernels into \Cref{eq:GWdist}. Towards this goal, for each $k\in\N$, we consider a special type of maps $\nu^{\scriptscriptstyle{(k)}}_{\bullet,\bullet}:X\times Y\rightarrow\prob(X\times Y)$ which are defined similarly to how we define $k$-step Markov kernels and satisfy that for any $x\in X$ and $y\in Y$, $\nu^{\scriptscriptstyle{(k)}}_{x,y}$ is a coupling between the $k$-step Markov kernels $m_x^{X,\otimes k}$ and $m_y^{Y,\otimes k}$;  see \Cref{sec: k fold dW} for the precise definition. We refer to $\nu^{\scriptscriptstyle{(k)}}_{\bullet,\bullet}$ as a ``$k$-step coupling'' between $k$-step Markov kernels. We let $\cpl^{\scriptscriptstyle{(k)}}\!\lc m_\bullet^X,m_\bullet^Y\rc$ denote the collection of all such $k$-step couplings $\nu^{\scriptscriptstyle{(k)}}_{\bullet,\bullet}$.

\begin{definition}\label{def:MCMsGW}
For any $k\geq 1$ and any MCMSs $(\mX,d_X)$ and $(\mY,d_Y)$, we define the $k$-distortion of any pair $(\gamma,\nu^{\scriptscriptstyle{(k)}}_{\bullet,\bullet})$ where $\gamma\in \cpl(\mu_X,\mu_Y)$ and $\nu^{\scriptscriptstyle{(k)}}_{\bullet,\bullet}\in\cpl^{\scriptscriptstyle{(k)}}\!\lc m_\bullet^X,m_\bullet^Y\rc$ as:
\begin{multline*}
    \mathrm{dis}^{\scriptscriptstyle{(k)}}\!\lc\gamma,\nu_{\bullet,\bullet}^{\scriptscriptstyle{(k)}}\rc\coloneqq\int\limits_{X\times Y}\int\limits_{X\times Y}\int\limits_{X\times Y}\!\!|d_X(x,x')-d_Y(y,y')|\\
 \nu_{x'',y''}^{\scriptscriptstyle{(k)}}(dx'\times dy')\,\gamma(dx''\times dy'')\,\gamma(dx\times dy).
\end{multline*}
This notion of distortion implements a multiscale reweighting of the coupling $\gamma$ through the $k$-step coupling $\nu_{\bullet,\bullet}^{\scriptscriptstyle{(k)}}$.
Then, the \emph{$k$-Gromov-Wasserstein distance} between the MCMSs $(\mX,d_X)$ and $(\mY,d_Y)$ is defined by
\[\dGW^{\scriptscriptstyle{(k)}}\!\lc(\mX,d_X),(\mY,d_Y)\rc \coloneqq
\!\!\!\!\!\!\inf_{\substack{\gamma\in \cpl(\mu_X,\mu_Y)\\\nu^{\scriptscriptstyle{(k)}}_{\bullet,\bullet}\in\cpl^{\scriptscriptstyle{(k)}}\!\lc m_\bullet^X,m_\bullet^Y\rc}}\!\!\!\!\mathrm{dis}^{\scriptscriptstyle{(k)}}\!\lc\gamma,\nu^{\scriptscriptstyle{(k)}}_{\bullet,\bullet}\rc.\]

We then define the (absolute) \textbf{Gromov-Wasserstein distance} between MCMSs by
\[\dGW^\mathrm{MCMS}((\mX,d_X),(\mY,d_Y))\!\!\coloneqq\sup_k\dGW^{\scriptscriptstyle{(k)}}\!\lc(\mX,d_X),(\mY,d_Y)\rc\!.\]
\end{definition}

\begin{proposition}\label{prop:kGWmetric}
$\dGW^\mathrm{MCMS}$ defines a proper\footnote{Unlike the case for $\dWL$, two MCMSs have zero $\dGW^\mathrm{MCMS}$ distance iff they are isomorphic. The precise definition of isomorphism between MCMSs is postponed to \Cref{def:MCMS isomorphism} in \Cref{sec:proof iso}.} distance on the collection $\mcms$ modulo isomorphism of MCMSs.
\end{proposition}

\begin{example}[MMS induced MCMS]\label{prop:Haibin}
Given a metric measure space $\mathbf{X}=(X,d_X,\mu_X)$, we produce a MCMS $\mathcal{M}(\mathbf{X})\coloneqq(\mX,d_X)$ where $\mX\coloneqq (X,m_\bullet^X,\mu_X)$ by letting $m_\bullet^X\coloneqq\mu_X$ be the constant Markov kernel. It is easy to check that $\mu_X$ is a stationary distribution w.r.t. $m_\bullet^X$. Then, for any two metric measure spaces $\mathbf{X}=(X,d_X,\mu_X)$, $\mathbf{Y}=(Y,d_Y,\mu_Y)$, and $k\geq 1$, we have that $$\dGW^{\scriptscriptstyle{(k)}}(\mathcal{M}(\mathbf{X}),\mathcal{M}(\mathbf{Y}))=\dGW^{\mathrm{bi}}(\mathbf{X},\mathbf{Y})$$
where $\dGW^{\mathrm{bi}}$ denotes a ``decoupled" version of the Gromov-Wasserstein distance between metric measure spaces (see \Cref{sec:decGW}) which is of course independent of $k$.
\end{example}

$\dGW^{\mathrm{bi}}$ is in general NP-hard to compute \cite{scetbon2021linear}, which (via \Cref{prop:Haibin}) implies that $\dGW^\mathrm{MCMS}$ is also NP-hard to compute.
See \Cref{sec:basic lower bound} for  a basic computable lower bound estimate of $\dGW^\mathrm{MCMS}$. In the next section, we establish more sophisticated lower bounds for $\dGW^\mathrm{MCMS}$ involving the WL distance.

\begin{table}[htb!]
\caption{1-Nearest Neighbor classification accuracy.} 
\label{tab:nn experiments}
\vskip 0.15in
\begin{center}
\resizebox{\columnwidth}{!}{\begin{tabular}{lccccccc}
\toprule
Method & MUTAG  & PROTEINS & PTC-FM & PTC-MR & IMDB-B & IMDB-M & COX2\\
\midrule 
$\dWLk$ & \textbf{92.1 $\pm$ 6.3} & 63.0 $\pm$ 3.5 & \textbf{62.2 $\pm$ 8.5} & 56.2 $\pm$ 6.3 & 70.0 $\pm$ 4.3 & \textbf{41.3 $\pm$ 4.8} & 76.1 $\pm$ 5.5\\
\midrule
$d_{\mathrm{WLLB}}^{\scriptscriptstyle{(k)}}$ & 87.3 $\pm$ 1.9 & \textbf{66.2 $\pm$ 2.2} & 62.5 $\pm$ 8.5 & \textbf{57.8 $\pm$ 6.8} & \textbf{69.9 $\pm$ 2.5} & 40.6 $\pm$ 3.8 & \textbf{81.2 $\pm$ 5.3}\\
\bottomrule
WWL & 85.1 $\pm$ 6.5 & 64.7 $\pm$ 2.8 & 58.2 $\pm$ 8.5 & 54.3 $\pm$ 7.9 & 65.0 $\pm$ 3.3 & 40.0 $\pm$ 3.3 &  76.1 $\pm$ 5.6 \\
\bottomrule
\end{tabular}}
\end{center}
\vskip -0.1in
\end{table}

\begin{table*}[htb!]
\caption{SVM classification accuracy.}
\label{tab:svm experiments}
\vskip 0.15in
\begin{center}
\resizebox{\columnwidth}{!}{\begin{tabular}{lcccccccr}
\toprule
Method & MUTAG  & PROTEINS & PTC-FM  & PTC-MR & IMDB-B & IMDB-M & COX2\\
\midrule
$\dWLk$ & 89.9 $\pm$ 6.4 & 72.6 $\pm$ 3.1 & 62.1 $\pm$ 3.9 & 57.9 $\pm$ 7.9 & \textbf{75.9 $\pm$ 2.7} & 51.6 $\pm$ 4.0 & 78.1 $\pm$ 0.8 \\
\midrule
$d_{\mathrm{WLLB}}^{\scriptscriptstyle{(k)}}$ & \textbf{90.0 $\pm$ 5.6} & 68.9 $\pm$ 1.9 & 59.6 $\pm$ 6.4 & 59.0 $\pm$ 8.3 & 75.1 $\pm$ 2.2 & \textbf{52.0 $\pm$ 1.8} & 78.1 $\pm$ 0.8\\
\bottomrule
WWL & 85.3 $\pm$ 7.3 &\textbf{ 72.9 $\pm$ 3.6 }& \textbf{62.2 $\pm$ 6.1} & \textbf{63.0 $\pm$ 7.4} & 70.8 $\pm$ 5.4& 50.0 $\pm$ 5.3 & 78.2 $\pm$ 0.8\\
WL & 85.5$\pm$ 1.6 & 71.6 $\pm$ 0.6 & 56.6 $\pm$ 2.1 & 56.2 $\pm$ 2.0 & 72.4 $\pm$ 0.7 & 50.9 $\pm$ 0.4 & 78.4 $\pm$ 1.1\\
WL-OA & 86.3 $\pm$ 2.1 & 72.6 $\pm$ 0.7 & 58.4 $\pm$ 2.0 & 54.2 $\pm$ 1.6 & 73.0 $\pm$ 1.1 & 50.2 $\pm$ 1.1 & \textbf{78.8 $\pm$ 1.3}\\
\bottomrule
\end{tabular}}
\end{center}
\vskip -0.1in
\end{table*}

\subsection{The WL distance v.s. the GW distance}
Given a MCMS $(\mX,d_X)$, fix any $x\in X$. Then, one can endow $(\mX,d_X)$ with the label function $ d_X(x,\bullet):X\rightarrow \R$. This gives rise to the LMMC
$(\mX,d_X(x,\bullet)).$
Then, we have the following lower bound of $\dGW^{\scriptscriptstyle{(k)}}$ in terms of $\dWLk$:

\begin{proposition}\label{prop:lower bound for dGW}
For each $k\geq 1$ and for any MCMSs $(\mX,d_X)$ and $(\mY,d_Y)$, one has that
\[  \dGW^{\scriptscriptstyle{(k)}}((\mX,d_X),(\mY,d_Y))\geq\inf_{\gamma\in\cpl(\mu_X,\mu_Y)}
    \int\limits_{X\times Y}\dWLk\!\big((\mX,d_X(x,\bullet)),(\mY,d_Y(y,\bullet))\big)\gamma(dx\times dy).\]
\end{proposition}
\begin{remark}\label{rmk:TLB}
When $(\mX,d_X)$ and $(\mY,d_Y)$ are induced from MMSs $\mathbf{X}$ and $\mathbf{Y}$, as shown in \Cref{prop:Haibin}, the left-hand side of the above inequality coincides with the decoupled GW distance. We point out that the right-hand side is also independent of $k$ and actually coincides with the \emph{third lower bound} (TLB) for the GW distance as defined in \cite{memoli2011gromov}. See \Cref{sec: rmk tlb} for more details. Hence, the proposition above can be viewed as a generalization of the TLB to the setting of MCMS.
\end{remark}
In general, a MCMS $(\mX,d_X)$ endowed with any label function $\ell_X:X\rightarrow Z$ induces a LMMC $(X,m_\bullet^X,\mu_X,\ell_X)$ which we denote by $(\mX,\ell_X)$.
If we assign label functions to all MCMSs in a suitable coherent way, we obtain that the WL distance between the induced LMMCs is stable w.r.t. the GW distance between corresponding MCMSs.

\begin{definition}[Label invariant of MCMCs]
Given any metric space of labels $Z$, a $Z$-valued \textbf{label invariant} of MCMSs is a map $\ell_\bullet:\mcms\rightarrow Z^\bullet$ which by definition sends each MCMS $(\mX,d_X)$ into a label function $\ell_X:X\rightarrow Z.$ One such label invariant $\ell_\bullet$ will be said to be \emph{stable} if for all $(\mX,d_X),(\mY,d_Y)\in\mcms,\,k\in\mathbb{N}$, and for any $\gamma\in \cpl(\mu_X,\mu_Y),\nu^{\scriptscriptstyle{(k)}}_{\bullet,\bullet}\in\cpl^{\scriptscriptstyle{(k)}}\!\lc m_\bullet^X,m_\bullet^Y\rc$ we have
\[ \int\limits_{X\times Y}\int\limits_{X\times Y}d_Z(\ell_X(x'),\ell_Y(y'))\,\nu_{x,y}^{\scriptscriptstyle{(k)}}(dx'\times dy')\gamma(dx\times dy)
\leq \mathrm{dis}^{\scriptscriptstyle{(k)}}\!\lc\gamma,\nu^{\scriptscriptstyle{(k)}}_{\bullet,\bullet}\rc.\]
By $\invs(Z)$ we will denote the collection of all stable $Z$-valued label invariants. 
\end{definition}

\begin{example}\label{ex:eccentricity}
One immediate example of the stable label invariant is the \emph{eccentricity function} $\mathrm{ecc}_\bullet$ (see \Cref{sec:proof ecc} for a proof): for any MCMS $(\mX,d_X)$,  $\mathrm{ecc}_X(x)\coloneqq\int_Xd_X(x,x')\,\mu_X(dx')$ for  $x\in X$.
\end{example}
Then, if one assigns stable labels to MCMSs, one has that the WL distance between the induced LMMCs is stable w.r.t. the GW distance.
\begin{proposition}\label{prop:stable WL}
For every stable label invariant $\ell_\bullet \in \invs(Z)$, $k\in\N$ and $(\mX,d_X),(\mY,d_Y)\in \mcms$ we have that
$\dWLk\!\lc(\mX,\ell_X),(\mY,\ell_Y)\rc \leq \dGW^{\scriptscriptstyle{(k)}}\!\lc(\mX,d_X),(\mY,d_Y)\rc $.
\end{proposition}


\section{Experimental Results}\label{sec:experiment}
We provide some 
results showing the effectiveness of our WL distance in terms of comparing graphs. 
We conduct both 1-NN and SVM graph classification experiments and evaluate the performance of both our lower bound, $d_{\mathrm{WLLB}}^{\scriptscriptstyle{(k)}}$, and our WL distance, $\dWLk$, against the WWL kernel/distance \cite{togninalli2019wasserstein}, the WL kernel and the WL optimal assignment (WL-OA) \cite{kriege2016valid} kernel. 
We note that the WWL kernel of \cite{togninalli2019wasserstein} is a state-of-the-art graph kernel. 
We use the degree label for both $\dWLk$ and $d_{\mathrm{WLLB}}^{\scriptscriptstyle{(k)}}$. More details on the experimental setup and extra experiments can be found in \Cref{appendix:experiments}.

\paragraph{1-NN classification.}
In this case, both $d_{\mathrm{WLLB}}^{\scriptscriptstyle{(k)}}$ and $\dWLk$ slightly outperform the WWL distance on all datasets we tested; see \Cref{tab:nn experiments}. Overall, the classification accuracies of $d_{\mathrm{WLLB}}^{\scriptscriptstyle{(k)}}$ and $\dWLk$ were close to those of the WWL distance. These results illustrate the close relationship between $\dWLk$ and the WWL distance that was outlined in  \Cref{sec:NN}. 

\paragraph{SVM classification.}
See \Cref{tab:svm experiments}. First, we observe that our lower bound kernel slightly outperforms the WWL kernel for MUTAG, IMDB-B, and IMDB-M. For the other datasets, $d_{\mathrm{WLLB}}^{\scriptscriptstyle{(k)}}$ had comparable classification accuracy with the other methods, coming within one to two percent of WWL and WWL-OA. The $\dWLk$ kernel had similar classification accuracy to $d_{\mathrm{WLLB}}^{\scriptscriptstyle{(k)}}$ but only outperformed the $d_{\mathrm{WLLB}}^{\scriptscriptstyle{(k)}}$ kernel on PTC-FM, PROTEINS, and IMDB-B.

Note that our lower bound distance $d_{\mathrm{WLLB}}$ performs similarly to our WL distance $\dWL$, but is more efficient to compute. See \Cref{subsec:time comparison} for the runtime comparison.

\section{Conclusion and future directions}
In this paper, we proposed the WL distance -- a quantitative extension of the WL test -- for measuring the dissimilarity between objects in a fairly general family called LMMCs. The WL distance  possesses interesting connections with graph kernels, GNNs and the GW distance.
In order to more directly compare the WL test with the WL distance without resorting to relabeling, one future direction is to redefine the WL distance via positive measures and  ``unbalanced" (Gromov-)Wasserstein distances \cite{liero2018optimal,sejourne2020unbalanced,de2020entropy}.
Whereas our  paper focuses on the application of our WL distance to the graph setting, LMMCs can be used to model not just graphs but also far more general objects such as Riemannian manifolds (equipped with heat kernels) or graphons. Then, our neural network architecture (MCNN) has the potential to be applied to point sets sampled from manifolds too, as well as serving as the limiting object when studying the convergence of GNNs.
It is also interesting to extend our WL distance to a higher order version that is analogous to the high order $k$-WL test and $k$-GNNs. We conjecture that a suitable notion of $k$-WL distance will converge to the Gromov-Wasserstein distance for MCMSs as $k$ tends to infinity. 

\paragraph{Acknowledgements.} This work is partially supported by  NSF-DMS-1723003, NSF-CCF-1740761, NSF-RI-1901360, NSF-CCF-1839356, NSF-IIS-2050360, NSF-2112665 and BSF-2020124.

\bibliography{dwl}

\newpage
\appendix
\onecolumn

\section{Extra details}
\subsection{Useful facts about couplings}
Here we collect some useful facts about couplings  which will be used in subsequent proofs.

\begin{lemma}\label{lm:push forward of coupling dW}
Let $X,Y$ be finite metric spaces and let $Z$ be a complete and separable metric space. Let $\varphi_X:X\rightarrow Z$ and $\varphi_Y:Y\rightarrow Z$ be measurable maps. Consider any $\mu_X\in\prob(X)$ and $\mu_Y\in\prob(Y)$. Then, we have that
\[\dW((\varphi_X)_\#\mu_X,(\varphi_Y)_\#\mu_Y)=\inf_{\gamma\in\cpl\lc \mu_X,\mu_Y\rc }\int\limits_{X\times Y}d_Z(\varphi_X(x),\varphi_Y(y))\gamma(dx\times dy).\]
\end{lemma}
\begin{proof}
The proof is based on the following result.
\begin{lemma}\label{lm:push forward of coupling}
Let $X,Y$ be finite metric spaces and let $Z$ be a metric space of labels. Let $\varphi_X:X\rightarrow Z$ and $\varphi_Y:Y\rightarrow Z$ be measurable maps. Consider any $\mu_X\in\prob(X)$ and $\mu_Y\in\prob(Y)$. If we let $\varphi\coloneqq\varphi_X\times \varphi_Y$, then we have that
\[\varphi_\#\mathcal{C}(\mu_X,\mu_Y)=\cpl\lc (\varphi_X)_\#\mu_X,(\varphi_Y)_\#\mu_Y\rc \]
\end{lemma}
\begin{proof}[Proof of \Cref{lm:push forward of coupling}]
Since $X$ and $Y$ are finite, $\varphi_X(X)$ and $\varphi_Y(Y)$ are discrete sets. Then, the lemma follows directly from Proposition 4.5 in \cite{schmitzer2013modelling}.
\end{proof}
Hence, 
\begin{align*}
    \dW((\varphi_X)_\#\mu_X,(\varphi_Y)_\#\mu_Y)&=\inf_{\gamma\in\cpl\lc (\varphi_X)_\#\mu_X,(\varphi_Y)_\#\mu_Y\rc }\int\limits_{X\times Y}d_Z(z_1,z_2)\gamma(dz_1\times dz_2)\\
    &=\inf_{\gamma\in\cpl\lc \mu_X,\mu_Y\rc }\int\limits_{X\times Y}d_Z(z_1,z_2)\,\varphi_\#\gamma(dz_1\times dz_2)\\
    &=\inf_{\gamma\in\cpl\lc \mu_X,\mu_Y\rc }\int\limits_{X\times Y}d_Z(\varphi_X(x),\varphi_Y(y))\gamma(dx\times dy).
\end{align*}
\end{proof}

The following lemma is a direct consequence of  \cite[Corollary 5.22]{villani2009optimal}
\begin{lemma}\label{lm:mble optimal coupling}
For any complete and separable metric space $Z$, there exists a measurable map $\varphi:\probm(Z)\times \probm(Z)\rightarrow\prob(Z\times Z)$ so that for every $\alpha,\beta\in\probm(Z)$, $\varphi(\alpha,\beta)$ is an optimal coupling between $\alpha$ and $\beta$.
\end{lemma}

\subsection{A characterization of the WL distance}\label{sec: k fold dW}

Recall from \Cref{def:WL metric} that the WL distance of depth $k$ is the Wasserstein distance between the local distributions of labels generated at the $k$th step of the WL hierarchy. In this section, we prove that in fact $\dWLk$ can be characterized through a {novel variant} of the Wasserstein distance between the distributions of initial labels (i.e., $(\ell_X)_\#\mu_X$ and $(\ell_Y)_\#\mu_Y$).

\subsubsection{$k$-step couplings}

We introduce a convenient notation which will be used in the sequel.

\begin{definition}
Suppose a measurable space $Z$, a probability measure $\gamma\in\prob(Z)$, and a measurable map $\nu_{\bullet}:Z\longrightarrow\prob(Z)$ are given. Then, we define the average of $\nu_\bullet$ under $\gamma$, denoted by $\nu_{\bullet}\odot\gamma$, which is still a probability measure on $Z$:
$$\mbox{For any measurable }A\subseteq Z,\,\,\nu_{\bullet}\odot\gamma(A)\coloneqq\int_{Z}\nu_{z}(A)\,\gamma(dz).$$
\end{definition}

The operation $\odot$ will be useful for constructing a special type of couplings between Markov kernels.

Given two MMCs $\mX$ and $\mY$, we introduce the following notion of \textbf{$k$-step coupling between Markov kernels $m_\bullet^{X}$ and $m_\bullet^{Y}$}:

\begin{itemize}
\item{$k=1$:} A 1-step coupling between $m_\bullet^{X}$ and $m_\bullet^{Y}$ is any measurable map
\[\nu_{\bullet,\bullet}^{\scriptscriptstyle{(1)}}:X\times Y\rightarrow\prob(X\times Y)\]
such that $\nu_{x,y}^{\scriptscriptstyle{(1)}}\in\cpl(m_x^X,m_y^Y)$ for any $x\in X$ and $y\in Y$.
\item{$k\geq 2$:} We say a map
\[\nu_{\bullet,\bullet}^{\scriptscriptstyle{(k)}}:X\times Y\rightarrow\prob(X\times Y)\]
is a \textbf{$k$-step coupling} between the Markov kernels $m_\bullet^{X}$ and $m_\bullet^{Y}$ if there exist a $(k-1)$-step coupling $\nu^{\scriptscriptstyle{(k-1)}}_{\bullet,\bullet}$ and a 1-step coupling $\nu^{\scriptscriptstyle{(1)}}_{\bullet,\bullet}$ such that
\[\nu_{x,y}^{\scriptscriptstyle{(k)}}=\nu_{\bullet,\bullet}^{\scriptscriptstyle{(k-1)}}\odot\nu_{x,y}^{\scriptscriptstyle{(1)}},\,\,\forall x\in X,\,\,y\in Y.\]
\end{itemize}

\begin{lemma}\label{lm:k-fold coupling well defined}
For any $k$-step coupling $\nu^{\scriptscriptstyle{(k)}}_{\bullet,\bullet}$, one has that  $\nu^{\scriptscriptstyle{(k)}}_{x,y}\in\cpl\!\left(m_x^{X,\otimes k},m_y^{Y,\otimes k}\right)$ for every $x\in X$ and $y\in Y$.
\end{lemma}
\begin{proof}
We prove the statement by induction on $k$. When $k=1$, the statement is trivially true. Assume that the statement is true for some $k\geq 1$. Then, for any measurable $A\subseteq X$, we have that
\begin{align*}
    \nu_{x,y}^{\scriptscriptstyle{(k+1)}}(A\times Y)&=\int\limits_{X\times Y}\nu_{x',y'}^{\scriptscriptstyle{(k)}}(A\times Y)\,\nu_{x,y}^{\scriptscriptstyle{(1)}}(dx'\times dy')\\
    &=\int\limits_{X\times Y}m_{x'}^{X,\otimes k}(A)\,\nu_{x,y}^{\scriptscriptstyle{(1)}}(dx'\times dy')\\
    &=\int_{X}m_{x'}^{X,\otimes k}(A)\,m_{x}^X(dx')\\
    &=m_{x}^{X,\otimes (k+1)}(A).
\end{align*}
Similarly, for any measurable $B\subseteq Y$, we have that 
\[\nu_{x,y}^{\scriptscriptstyle{(k+1)}}(X\times B)=m_{y}^{Y,\otimes (k+1)}(B).\]
Hence $\nu_{x,y}^{\scriptscriptstyle{(k+1)}}\in\cpl\lc m_{x}^{X,\otimes (k+1)},m_{y}^{Y,\otimes (k+1)}\rc $ and thus we conclude the proof.
\end{proof}

Henceforth, we denote by $\cpl^{\scriptscriptstyle{(k)}}\!\lc m_\bullet^X,m_\bullet^Y\rc$ the collection of all $k$-step couplings between Markov kernels $m_\bullet^X$ and $m_\bullet^Y$.
\begin{definition}
Given any $k\in\N$, any coupling $\gamma\in\cpl(\mu_X,\mu_Y)$ and any $k$-step coupling $\nu_{\bullet,\bullet}^{\scriptscriptstyle{(k)}}\in\cpl^{\scriptscriptstyle{(k)}}\!\lc m_\bullet^{X},m_\bullet^Y\rc$, we define a probability measure on $X\times Y$ as follows:
\begin{equation}\label{eq:k-fold stationary}
\mu^{\scriptscriptstyle{(k)}}\coloneqq\nu_{\bullet,\bullet}^{\scriptscriptstyle{(k)}}\odot\gamma.
\end{equation}
We call $\mu^{\scriptscriptstyle{(k)}}$ defined as above a \textbf{$k$-step coupling between $\mu_X$ and $\mu_Y$}. We let $\cpl^{\scriptscriptstyle{(k)}}(\mu_X,\mu_Y)$ denote the collection of all $k$-step couplings between $\mu_X$ and $\mu_Y$ and we also let $\cpl^{\scriptscriptstyle{(0)}}(\mu_X,\mu_Y)\coloneqq\cpl(\mu_X,\mu_Y)$.
\end{definition}

As defined above,  $k$-step couplings are indeed couplings.

\begin{lemma}
Any $\mu^{\scriptscriptstyle{(k)}}\in\cpl^{\scriptscriptstyle{(k)}}(\mu_X,\mu_Y)$ is a coupling between $\mu_X$ and $\mu_Y$.
\end{lemma}

\begin{proof}
For any measurable $A\subseteq X$, we have that
\begin{align*}
    \mu^{\scriptscriptstyle{(k)}}(A\times Y)&=\int\limits_{X\times Y}\nu_{x,y}^{\scriptscriptstyle{(k)}}(A\times Y)\,\gamma(dx\times dy)\\
    &=\int\limits_{X\times Y}m_{x}^{X,\otimes k}(A)\,\gamma(dx\times dy)\\
    &=\int_{X}m_{x}^{X,\otimes k}(A)\,\mu_X(dx)\\
    &=\mu_X(A).
\end{align*}
Similarly, for any measurable $B\subseteq Y$, we have that $\mu^{\scriptscriptstyle{(k)}}(X\times B)=\mu_Y(B)$. Therefore, $\mu^{\scriptscriptstyle{(k)}}\in\cpl(\mu_X,\mu_Y)$.
\end{proof}

\begin{lemma}\label{lm:hierarchy of k-fold couplings}
We have the following hierarchy of $k$-step couplings:
\[\cpl^{\scriptscriptstyle{(0)}}(\mu_X,\mu_Y)\supseteq\cpl^{\scriptscriptstyle{(1)}}(\mu_X,\mu_Y)\supseteq\cpl^{\scriptscriptstyle{(2)}}(\mu_X,\mu_Y)\supseteq\cdots\]
\end{lemma}

\begin{proof}
We prove the following inclusion by induction on $k=0,1,\ldots$ :
\begin{equation}\label{eq:inclusion}
  \cpl^{\scriptscriptstyle{(k)}}(\mu_X,\mu_Y)\supseteq\cpl^{\scriptscriptstyle{(k+1)}}(\mu_X,\mu_Y)  
\end{equation}
When $k=0$, we only need to check that any $\gamma^{\scriptscriptstyle{(1)}}\in\cpl^{\scriptscriptstyle{(1)}}(\mu_X,\mu_Y)$ is a coupling between $\mu_X$ and $\mu_Y$. Assume that
\[\gamma^{\scriptscriptstyle{(1)}}=\nu_{\bullet,\bullet}^{\scriptscriptstyle{(1)}}\odot\gamma\]
for some $\nu_{\bullet,\bullet}^{\scriptscriptstyle{(1)}}\in\cpl^{\scriptscriptstyle{(1)}}\!\lc m_\bullet^X,m_\bullet^Y\rc$ and $\gamma\in\cpl(\mu_X,\mu_Y)$.
For any measurable set $A\subseteq X$, we have that
\begin{align*}
    \gamma^{\scriptscriptstyle{(1)}}(A\times Y)&=\int\limits_{X\times Y}\nu_{x,y}^{\scriptscriptstyle{(1)}}(A\times Y)\,\gamma(dx\times dy)\\
    &=\int\limits_{X\times Y}m_x^X(A)\,\gamma(dx\times dy)\\
    &=\int_{X}m_x^X(A)\,\mu_X(dx)\\
    &=\mu_X(A).
\end{align*}
Similarly, for any measurable $B\subseteq Y$ we have that
\[\gamma^{\scriptscriptstyle{(1)}}(X\times B)=\mu_Y(B).\]
Hence, $\gamma^{\scriptscriptstyle{(1)}}\in\cpl(\mu_X,\mu_Y)$.

Now, assume that Equation \eqref{eq:inclusion} holds for some $k\geq 0$. For any $\gamma^{\scriptscriptstyle{(k+2)}}\in\cpl^{\scriptscriptstyle{(k+2)}}(\mu_X,\mu_Y)$, we assume that
\[\gamma^{\scriptscriptstyle{(k+2)}}=\nu_{\bullet,\bullet}^{\scriptscriptstyle{(k+2)}}\odot\gamma,\]
where $\nu_{\bullet,\bullet}^{\scriptscriptstyle{(k+2)}}\in\cpl^{\scriptscriptstyle{(k+2)}}\!\lc m_\bullet^X,m_\bullet^Y\rc$ and $\gamma\in\cpl(\mu_X,\mu_Y)$.
Then, there exist $\nu_{\bullet,\bullet}^{\scriptscriptstyle{(k+1)}}\in \cpl^{\scriptscriptstyle{(k+1)}}\!\lc m_\bullet^X,m_\bullet^Y\rc$ and $\nu_{\bullet,\bullet}^{\scriptscriptstyle{(1)}}\in \cpl^{\scriptscriptstyle{(1)}}\!\lc m_\bullet^X,m_\bullet^Y\rc$ such that
\[\nu_{x,y}^{\scriptscriptstyle{(k+2)}}=\nu_{\bullet,\bullet}^{\scriptscriptstyle{(k+1)}}\odot\nu_{x,y}^{\scriptscriptstyle{(1)}},\,\,\forall x\in X,\,\,y\in Y.\]
Hence, 
\begin{align*}
    \gamma^{\scriptscriptstyle{(k+2)}}&=\int\limits_{X\times Y}\int\limits_{X\times Y}\nu_{x',y'}^{\scriptscriptstyle{(k+1)}}\nu_{x,y}^{\scriptscriptstyle{(1)}}(dx'\times dy')\,\gamma(dx\times dy)\\
    &=\int\limits_{X\times Y}\nu_{x',y'}^{\scriptscriptstyle{(k+1)}}\int\limits_{X\times Y}\nu_{x,y}^{\scriptscriptstyle{(1)}}(dx'\times dy')\,\gamma(dx\times dy)\\
    &=\int\limits_{X\times Y}\nu_{x',y'}^{\scriptscriptstyle{(k+1)}}\,\gamma^{\scriptscriptstyle{(1)}}(dx'\times dy').
\end{align*}
Here that $\gamma^{\scriptscriptstyle{(1)}}\coloneqq\nu_{\bullet,\bullet}^{\scriptscriptstyle{(1)}}\odot\gamma$ belongs to $\cpl(\mu_X,\mu_Y)$ follows from the case $k=0$. Hence, by the induction assumption, $\gamma^{\scriptscriptstyle{(k+2)}}\in\cpl^{\scriptscriptstyle{(k+1)}}(\mu_X,\mu_Y)$ which concludes the proof.
\end{proof}

\subsubsection{A characterization of $\dWLk$ via $k$-step couplings}
Now, we characterize $\dWLk$ via $k$-step couplings defined in the previous section.

\begin{theorem}\label{thm:dwl= dwk}
Given any integer $k\geq 0$ and any two $Z$-LMMC $(\mX,\ell_X)$ and $(\mY,\ell_Y)$, we have that
\[\dWLk\!\lc(\mX,\ell_X),(\mY,\ell_Y)\rc =\inf_{\gamma^{\scriptscriptstyle{(k)}}\in \cpl^{\scriptscriptstyle{(k)}}(\mu_X,\mu_Y)} \int\limits_{X\times Y}d_Z(\ell_X(x),\ell_Y(y))\gamma^{\scriptscriptstyle{(k)}}(dx\times dy).\]
\end{theorem}
\begin{proof}[Proof of \Cref{thm:dwl= dwk}]
The case $k=0$ holds trivially. Now, for any $k\geq 1$ and for any $x\in X$ and $y\in Y$, by \Cref{lm:push forward of coupling dW} we have that

\[\dW\!\lc \WLh{k}{(\mX,\ell_X)}(x),\WLh{k}{(\mY,\ell_Y)}(y)\rc=\inf_{\nu_{x,y}\in\cpl(m_x^X,m_y^Y)}\int\limits_{X\times Y}\dW\!\lc \WLh{k-1}{(\mX,\ell_X)}(x'),\WLh{k-1}{(\mY,\ell_Y)}(y')\rc \nu_{x,y}(dx'\times dy').\]

Since $(x,y)\mapsto (m_x^X,m_y^Y)$ is measurable by definition of  Markov kernels, by \Cref{lm:mble optimal coupling} we have that there exists a measurable map $\nu_{\bullet,\bullet}:X\times Y\rightarrow\prob(X\times Y)$ such that for every $x\in X$ and $y\in Y$, $\nu_{x,y}$ is optimal, i.e., 
\begin{equation}\label{eq:optimal nu}
    \dW\!\lc \WLh{k}{(\mX,\ell_X)}(x),\WLh{k}{(\mY,\ell_Y)}(y)\rc=\int\limits_{X\times Y}\dW\!\lc \WLh{k-1}{(\mX,\ell_X)}(x'),\WLh{k-1}{(\mY,\ell_Y)}(y')\rc \nu_{x,y}(dx'\times dy').
\end{equation}

Hence, we have the following formulas.
\begin{align*}
&\dWLk\!\lc(\mX,\ell_X),(\mY,\ell_Y)\rc\\
    =&\dW\!\lc \lc \WLh{k}{(\mX,\ell_X)}\rc_\#\mu_X,\lc \WLh{k}{(\mY,\ell_Y)}\rc_\#\mu_Y\rc \\
    =&\int\limits_{X\times Y}\dW\!\lc \WLh{k}{(\mX,\ell_X)}(x),\WLh{k}{(\mY,\ell_Y)}(y)\rc \gamma(dx\times dy),\quad \text{here }\gamma\in\mathcal{C}(\mu_X,\mu_Y)\text{ is chosen to be optimal}\\
    =&\int\limits_{X\times Y}\int\limits_{X\times Y}\dW\!\lc \WLh{k-1}{(\mX,\ell_X)}(x_1),\WLh{k-1}{(\mY,\ell_Y)}(y_1)\rc (\nu_1)_{x,y}(dx_1\times dy_1)\gamma(dx\times dy),\\
    =&\int\limits_{X\times Y}\cdots\int\limits_{X\times Y}\dW\!\lc \WLh{0}{(\mX,\ell_X)}(x_k),\WLh{0}{(\mY,\ell_Y)}(y_k)\rc (\nu_k)_{x_{k-1},y_{k-1}}(dx_k\times dy_k)\cdots(\nu_1)_{x,y}(dx_1\times dy_1)\gamma(dx\times dy).
\end{align*}
Here, each $(\nu_i)_{\bullet,\bullet}\in\cpl^{\scriptscriptstyle{(1)}}(m_\bullet^X,m_\bullet^Y)$ for $i=1,\ldots,k$ is optimal in the sense of \Cref{eq:optimal nu}.

From the above equations, we identify a probability measure on $X\times Y$ for every $x\in X$ and $y\in Y$ as follows
\[\nu_{x,y}^{\scriptscriptstyle{(k)}}\coloneqq\int\limits_{X\times Y}\cdots\int\limits_{X\times Y}(\nu_k)_{x_{k-1},y_{k-1}}\,(\nu_{k-1})_{x_{k-2},y_{k-2}}(dx_{k-1}\times dy_{k-1})\cdots(\nu_1)_{x,y}(dx_1\times dy_1).\]

It is obvious that $\nu_{\bullet,\bullet}^{\scriptscriptstyle{(k)}}$ is a $k$-step coupling and thus
\begin{align*}
    \dWLk\!\lc(\mX,\ell_X),(\mY,\ell_Y)\rc  =&\dW\!\lc \lc \WLh{k}{(\mX,\ell_X)}\rc_\#\mu_X,\lc \WLh{k}{(\mY,\ell_Y)}\rc_\#\mu_Y\rc \\
    =&\int\limits_{X\times Y}\int\limits_{X\times Y}\dW\!\lc \WLh{0}{(\mX,\ell_X)}(x'),\WLh{0}{(\mY,\ell_Y)}(y')\rc \nu_{x,y}^{\scriptscriptstyle{(k)}}(dx'\times dy')\gamma(dx\times dy)\\
    \geq&\inf_{\gamma^{\scriptscriptstyle{(k)}}\in\cpl^{\scriptscriptstyle{(k)}}(\mu_X,\mu_Y)}\int\limits_{X\times Y}\dW\!\lc \ell_X(x),\ell_Y(y)\rc\gamma^{\scriptscriptstyle{(k)}}(dx\times dy),
\end{align*}

Conversely, we have that given any $\gamma^{\scriptscriptstyle{(k)}}\coloneqq\nu_{\bullet,\bullet}^{\scriptscriptstyle{(k)}}\odot\gamma\in \cpl^{\scriptscriptstyle{(k)}}(\mu_X,\mu_Y)$, where $\gamma\in\cpl(\mu_X,\mu_Y)$ and  $\nu_{\bullet,\bullet}^{\scriptscriptstyle{(k)}}\in\cpl^{\scriptscriptstyle{(k)}}\!\lc m_\bullet^X,m_\bullet^Y\rc$ can be written for every $x\in X$ and $y\in Y$ as follows
\[\nu_{x,y}^{\scriptscriptstyle{(k)}}\coloneqq\int\limits_{X\times Y}\cdots\int\limits_{X\times Y}(\nu_k)_{x_{k-1},y_{k-1}}\,(\nu_{k-1})_{x_{k-2},y_{k-2}}(dx_{k-1}\times dy_{k-1})\cdots(\nu_1)_{x,y}(dx_1\times dy_1),\]
the following inequalities hold:
\begin{align*}
&\int\limits_{X\times Y}\dW\!\lc \ell_X(x),\ell_Y(y)\rc\gamma^{\scriptscriptstyle{(k)}}(dx\times dy)\\
    =&\int\limits_{X\times Y}\int\limits_{X\times Y}\dW\!\lc \WLh{0}{(\mX,\ell_X)}(x_k),\WLh{0}{(\mY,\ell_Y)}(y_k)\rc \nu_{x,y}^{\scriptscriptstyle{(k)}}(dx_k\times dy_k)\gamma(dx\times dy)\\
    =&\int\limits_{X\times Y}\cdots\int\limits_{X\times Y}\dW\!\lc \WLh{0}{(\mX,\ell_X)}(x_k),\WLh{0}{(\mY,\ell_Y)}(y_k)\rc (\nu_k)_{x_{k-1},y_{k-1}}(dx_k\times dy_k)\cdots(\nu_1)_{x,y}(dx_1\times dy_1)\gamma(dx\times dy)\\
    \geq & \int\limits_{X\times Y}\cdots\int\limits_{X\times Y}\dW\!\lc \WLh{1}{(\mX,\ell_X)}(x_{k-1}),\WLh{1}{(\mY,\ell_Y)}(y_{k-1})\rc (\nu_{k-1})_{x_{k-2},y_{k-2}}(dx_{k-1}\times dy_{k-1})\cdots\gamma(dx\times dy)\\
    \cdots&\\
    \geq &\int\limits_{X\times Y}\int\limits_{X\times Y}\dW\!\lc \WLh{k-1}{(\mX,\ell_X)}(x_1),\WLh{k-1}{(\mY,\ell_Y)}(y_1)\rc (\nu_1)_{x,y}(dx_1\times dy_1)\gamma(dx\times dy)\\
    \geq&\int\limits_{X\times Y}\dW\!\lc \WLh{k}{(\mX,\ell_X)}(x),\WLh{k}{(\mY,\ell_Y)}(y)\rc \gamma(dx\times dy)\\
    \geq &\dW\!\lc \lc \WLh{k}{(\mX,\ell_X)}\rc_\#\mu_X,\lc \WLh{k}{(\mY,\ell_Y)}\rc_\#\mu_Y\rc.
\end{align*}
Infimizing over all $\gamma$ and $\nu_{\bullet,\bullet}^{\scriptscriptstyle{(k)}}$, one concludes the proof.
\end{proof}

\subsection{The Wasserstein Weisfeiler-Lehman graph kernel and its relationship with $\dWL$}\label{sec:WWL}

The Wasserstein Weisfeiler-Lehman graph kernel deals with graphs with either categorical or `continuous' (i.e., Euclidean) labels \cite{togninalli2019wasserstein}. We describe their framework w.r.t. continuous labels as follows. For technical reasons, we assume that all graphs involved in this section are such that all their connected components have cardinality at least 2 (i.e. no graph contains an isolated vertex).

Given a labeled graph $(G,\ell_G:V_G\rightarrow\R^d)$, the label function is updated for a fixed number $k$ of iterations according to the equation below for $i=0,\ldots,k-1$, where $\ell^{0}_G\coloneqq\ell_G$.
\[\forall v\in V_G,\,\,\,\,\ell_G^{i+1}(v)\coloneqq\frac{1}{2}\!\lc\,\ell_G^{i}(v)+\frac{1}{\deg_G(v)}\sum_{v'\in N_G(v)}\ell_G^{i}(v')\rc. \]

Then, for each $i=0,\ldots,k$ there is a label function $\WLhor{i}{(G,\ell_G)}:V_G\rightarrow\R^d$. Define the stacked label function $L^k_G$ as follows:
$$L^k_G\coloneqq\lc\ell^{0}_{G},\ldots,\ell_G^{k}\rc:V_G\rightarrow \R^{d\times(k+1)}.$$ 
Now, given any two labeled graphs $(G_1,\ell_{G_1})$ and $(G_2,\ell_{G_2})$, \cite{togninalli2019wasserstein} first computed $L_{G_1}^k$ and $L_{G_2}^k$, then computed the Wasserstein distance between their induced distributions and finally, built a kernel upon this Wasserstein distance. If we let $\lambda_{G_i}$ denote the uniform measure on $V_{G_i}$, then we express their Wasserstein distance via pushforward of uniform measures as follows.
 \begin{equation}\label{eq:WWL distance}
     D^{\scriptscriptstyle{(k)}}\!\lc(G_1,\ell_{G_1}),(G_2,\ell_{G_2})\rc\coloneqq\dW\!\lc \lc L^k_{G_1}\rc_\#\lambda_{G_1},\lc L^k_{G_2}\rc_\#\lambda_{G_2} \rc.
 \end{equation}

Now, if we instead of uniform measures consider the stationary distributions $\mu_{G_1}$ and $\mu_{G_2}$ w.r.t. $m_\bullet^{G_1,\frac{1}{2}}$ and $m_\bullet^{G_2,\frac{1}{2}}$, respectively, we define the following variant of $D^{\scriptscriptstyle{(k)}}$:
\begin{equation}\label{eq:WWL distance variant}
    \hat{D}^{\scriptscriptstyle{(k)}}((G_1,\ell_{G_1}),(G_2,\ell_{G_2}))\coloneqq\dW\!\lc \lc L^k_{G_1}\rc_\#\mu_{G_1},\lc L^k_{G_2}\rc_\#\mu_{G_2} \rc,
\end{equation}
which is the distance which we will relate to our WL distance next. In fact, we then  prove that $\hat{D}^{\scriptscriptstyle{(k)}}((G_1,\ell_{G_1}),(G_2,\ell_{G_2}))$ actually provides a lower bound for $\dWLk((\mX_q(G_1),\ell_{G_1}),(\mX_q(G_2),\ell_{G_2}))$ when $q=\frac{1}{2}$.

\begin{proposition}\label{prop:wwl<dwlk}
For any two labeled graphs $(G_1,\ell_{G_1}:V_{G_1}\rightarrow\R^d)$ and $(G_2,\ell_{G_2}:V_{G_2}\rightarrow\R^d)$, one has that for $q=\frac{1}{2}$ and any $k\in\N$,
\[\hat{D}^{\scriptscriptstyle{(k)}}((G_1,\ell_{G_1}),(G_2,\ell_{G_2}))\leq k\cdot\dWLk((\mX_q(G_1),\ell_{G_1}),(\mX_q(G_2),\ell_{G_2})).\]
\end{proposition}

The proposition will be proved after we provide some examples and remarks.
\begin{figure}
    \centering
    \includegraphics[scale=0.7]{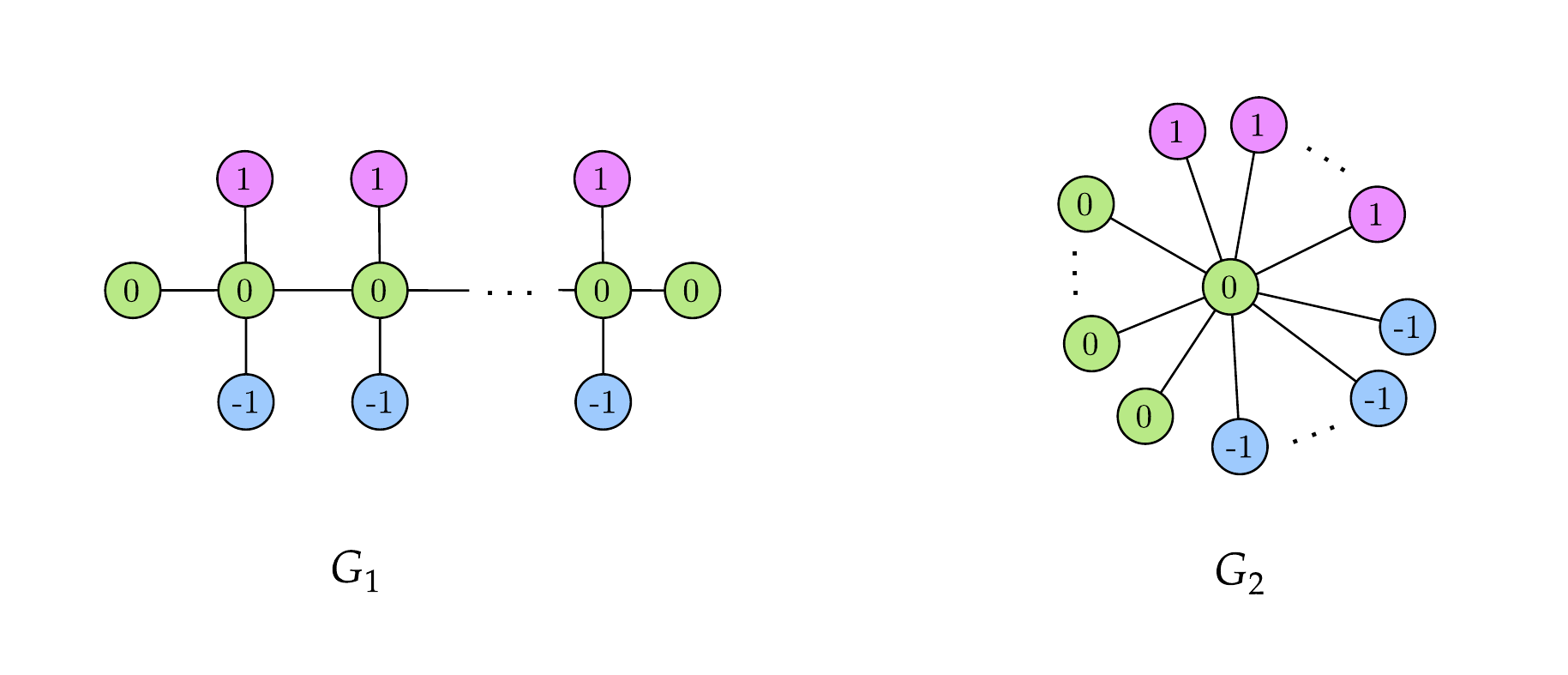}
    \caption{In this figure we show two labeled graphs and each of them has $3n+2$ vertices. Each of the graphs has $n$ vertices with label $1$, $n$ vertices with label $-1$, and $n+2$ vertices with label 0.}
    \label{fig:example 6}
\end{figure}

\begin{example}[$\dWLk$ is more discriminating than $\hat{D}^{\scriptscriptstyle{(k)}}$]\label{ex:better than WWL}
In this example, we construct a family of pairs of graphs so that $\hat{D}^{\scriptscriptstyle{(k)}}$ between the pairs is always zero but $\dWL$ between the pairs is positive.
For any $n\geq 2$, consider the two $(3n+2)$-point labeled graphs shown in \Cref{fig:example 6}. It is easy to see that for each $i=1,2$, and any $v\in V_{G_i}$, 
\begin{enumerate}
    \item if $\ell_{G_i}(v)=0$, then $\ell_{G_i}^{k}(v)=0$ for all $k=0,1,\ldots$;
    \item if $\ell_{G_i}(v)=\pm 1$, then $\ell_{G_i}^{k}(v)=\pm\frac{1}{2^k}$ for all $k=0,1,\ldots$.
\end{enumerate}
Hence, for any $k=0,\ldots$, we have that
\[\lc L_{G_1}^k\rc_\#\mu_{G_1}=\lc L_{G_2}^k\rc_\#\mu_{G_2}=\frac{4n+2}{6n+2}\delta_{(0,\cdots,0)}+\frac{n}{6n+2}\delta_{(1,\cdots,2^{-k})}+\frac{n}{6n+2}\delta_{(-1,\cdots,-2^{-k})}.\]
Therefore, $\hat{D}^{\scriptscriptstyle{(k)}}((G_1,\ell_{G_1}),(G_2,\ell_{G_2}))=0$ for all $k=0,\ldots$.

For proving that $\dWL>0$, we first analyze $\WLh{1}{(\mX(G_1),\ell_{G_1})}(v_1)$ for any $v_1\in V_{G_1}$.
\begin{enumerate}
    \item If $\ell_{G_1}(v_1)=\pm1$, then \[\WLh{1}{(\mX(G_1),\ell_{G_1})}(v_1)=\frac{1}{2}\delta_0+\frac{1}{2}\delta_{\pm1}.\]
    \item If $\ell_{G_1}(v_1)=0$ and $v_1$ is neither the leftmost nor the rightmost vertex, then
    \[\WLh{1}{(\mX(G_1),\ell_{G_1})}(v_1)=\frac{6}{8}\delta_0+\frac{1}{8}\delta_1+\frac{1}{8}\delta_{-1}.\]
    \item If $v_1$ is either the leftmost or the rightmost vertex, then
    \[\WLh{1}{(\mX(G_1),\ell_{G_1})}(v_1)=\delta_0.\]
\end{enumerate}
We then analyze $\WLh{1}{(\mX(G_2),\ell_{G_2})}(v_2)$ for any $v_2\in V_{G_2}$.
\begin{enumerate}
    \item If $\ell_{G_2}(v_2)=\pm1$, then \[\WLh{1}{(\mX(G_2),\ell_{G_2})}(v_2)=\frac{1}{2}\delta_0+\frac{1}{2}\delta_{\pm1}.\]
    \item If $v_2$ is the center vertex, then
    \[\WLh{1}{(\mX(G_2),\ell_{G_2})}(v_2)=\frac{2n+1}{3n+1}\delta_0+\frac{n}{2(3n+1)}\delta_1+\frac{n}{2(3n+1)}\delta_{-1}.\]
    \item If $\ell_{G_2}(v_2)=0$ and $v_2$ is not the center vertex, then
    \[\WLh{1}{(\mX(G_2),\ell_{G_2})}(v_2)=\delta_0.\]
\end{enumerate}

Hence, it is clear that when $n>1$
\begin{align*}
    \dWL^{\scriptscriptstyle{(1)}}((\mX_q(G_1),\ell_{G_1}),(\mX_q(G_2),\ell_{G_2}))=\dW\!\lc \lc \WLh{1}{(\mX(G_1),\ell_{G_1})}\rc_\#\mu_{G_1},\lc \WLh{1}{(\mX(G_2),\ell_{G_2})}\rc_\#\mu_{G_2}\rc>0.
\end{align*}
\end{example}

\begin{remark}
Our WL distance formulation is flexible and we can of course relax it by allowing the comparison of measure Markov chains with general reference probability measures which are not necessarily stationary. In that case, we can directly compare $D^{\scriptscriptstyle{(k)}}$ defined in \Cref{eq:WWL distance} with $\dWLk$. More precisely, for any graph $G$, we can replace the stationary distribution inherent to $\mX_q(G)$ with the uniform measure and hence obtain a new measure Markov chain $\mX_q^\mathrm{u}(G)$. Then, the same proof technique used for proving \Cref{prop:wwl<dwlk} can be used for proving that
\[{D}^{\scriptscriptstyle{(k)}}((G_1,\ell_{G_1}),(G_2,\ell_{G_2}))\leq k\cdot\dWLk((\mX_q^\mathrm{u}(G_1),\ell_{G_1}),(\mX_q^\mathrm{u}(G_2),\ell_{G_2})).\]
Moreover, we can show that $\dWLk$ is strictly more discriminating than $D^{\scriptscriptstyle{(k)}}$ via the same pairs of graphs as in \Cref{ex:better than WWL}.
\end{remark}

The proof of \Cref{prop:wwl<dwlk} is based on the following basic fact about the Wasserstein distance:

\begin{lemma}\label{lm:product dW}
Let $Z$ be a complete and separable metric space. Endow $Z\times Z$ with any product metric $d_{Z\times Z}$ such that
$$d_{Z\times Z}((z_1,z_2),(z_3,z_4))\leq{d_Z(z_1,z_3)+d_Z(z_2,z_4)}\,\,\,\forall z_1,z_2,z_3,z_4\in Z.$$
For example, one can let
$$d_{Z\times Z}((z_1,z_2),(z_3,z_4))\coloneqq\sqrt{\lc d_Z(z_1,z_3)\rc^2+\lc d_Z(z_2,z_4)\rc^2}.$$
Given any complete and separable metric space $X$, for any $i=1,2$ and any measurable maps $f_i,g_i:X\rightarrow Z$, if we let $h_i\coloneqq(f_i,g_i):X\rightarrow Z\times Z$, then for any $\mu_1,\mu_2\in\probm(X)$
\[\dW\!\lc (h_1)_\#\mu_1,(h_2)_\#\mu_2\rc\leq \dW\!\lc (f_1)_\#\mu_1,(f_2)_\#\mu_2\rc+\dW\!\lc (g_1)_\#\mu_1,(g_2)_\#\mu_2\rc.\]
\end{lemma}
\begin{proof}[Proof of \Cref{lm:product dW}]
For any $\gamma_f,\gamma_g\in\cpl(\mu_1,\mu_2)$, we define a probability measure $\nu\in\prob(Z\times Z\times Z\times Z)$ as follows: for any measurable $A,A',B,B'\subseteq Z$
\[\nu(A\times A'\times  B\times B')\coloneqq (f_1\times f_2)_\#\gamma_f(A\times B)\cdot (g_1\times g_2)_\#\gamma_g(A'\times B').\]
It is easy to show that $\nu\in\cpl\lc (h_1)_\#\mu_1,(h_2)_\#\mu_2\rc$. Then,
\begin{align*}
    &\dW\!\lc (h_1)_\#\mu_1,(h_2)_\#\mu_2\rc\\
    &\leq \int\limits_{Z\times Z}\int\limits_{Z\times Z}d_{Z\times Z}((z_1,z_2),(z_3,z_4))\,\nu(dz_1\times dz_2\times dz_3\times dz_4)\\
    &\leq  \int\limits_{Z\times Z}\int\limits_{Z\times Z}\lc d_{Z}(z_1,z_3)+d_{Z}(z_2,z_4)\rc\,\nu(dz_1\times dz_2\times dz_3\times dz_4)\\
    &\leq  \int\limits_{Z\times Z}d_{Z}(z_1,z_3)\,(f_1\times f_2)_\#\gamma_f(dz_1\times dz_3)+\int\limits_{Z\times Z}d_{Z}(z_2,z_4)\,(g_1\times g_2)_\#\gamma_g(dz_2\times dz_4)
\end{align*}
By \Cref{lm:push forward of coupling dW}, infimizing over all $\gamma_f,\gamma_g\in\cpl(\mu_1,\mu_2)$, we obtain the conclusion.
\end{proof}

\begin{proof}[Proof of \Cref{prop:wwl<dwlk}]
Since $L^k_G\coloneqq(\ell^{0}_{G},\ldots,\ell_{G}^{k})$, by inductively applying \Cref{lm:product dW}, we have that
\[\hat{D}^{\scriptscriptstyle{(k)}}(G_1,G_2)=\dW\!\lc \lc L^k_{G_1}\rc_\#\mu_{G_1},\lc L^k_{G_2}\rc_\#\mu_{G_2} \rc\leq \sum_{i=1}^k\dW\!\lc \lc \ell_{G_1}^{i}\rc_\#\mu_{G_1},\lc \ell_{G_2}^{i}\rc_\#\mu_{G_2} \rc.\]
Choose $\varphi_j\coloneqq\mathrm{id}:\R^d\rightarrow\R^d$ to be the identity map for each $j=1,\ldots,k$, then using notation from \Cref{sec:proof zero set}, we have that
\[\ell_{G_i}^j=\ell_{G_i}^{\scriptscriptstyle{(\varphi,j)}},\,\,\,\,\forall i=1,2\mbox{ and }\forall j=1,\ldots,k.\]
Then, by \Cref{eq:lip ineq}, we conclude that
\begin{align*}
    \hat{D}^{\scriptscriptstyle{(k)}}(G_1,G_2)&\leq \sum_{i=1}^k\dW\!\lc \lc \WLh{i}{(\mX_q(G_1),\ell_{G_1})}\rc_\#\mu_{G_1},\lc\WLh{i}{(\mX_q(G_2),\ell_{G_2})}\rc_\#\mu_{G_2} \rc\\
    &=\sum_{i=1}^k\dWL^{\scriptscriptstyle{(i)}}\!\lc {(\mX_q(G_1),\ell_{G_1})},{(\mX_q(G_2),\ell_{G_2})} \rc.
\end{align*}
By \Cref{coro:hierarchy dwlk}, we have that
\[\hat{D}^{\scriptscriptstyle{(k)}}(G_1,G_2)\leq k\cdot \dWLk\!\lc {(\mX_q(G_1),\ell_{G_1})},{(\mX_q(G_2),\ell_{G_2})} \rc.\]
\end{proof}
\subsection{A decoupled version of the Gromov-Wasserstein distance}\label{sec:decGW}

For simplicity, in this section we will assume that the cardinality of all underlying spaces to always be finite.

The Gromov-Wasserstein distance $d_{\mathrm{GW}}$ was proposed as a measure of dissimilarity between two metric measure spaces; see \Cref{sec:coupling and dW} for its definition and also  \cite{memoli2011gromov} for its more general version involving a parameter $p\in[1,\infty]$. Note that one can define a variant of the standard GW distance by considering two coupling measures $\gamma,\gamma'$ \emph{independently}, and use $\gamma\otimes\gamma'$ instead of $\gamma\otimes\gamma$ in \Cref{eq:GWdist}. This version of the GW distance was implicit in the optimization procedure followed in \cite{memoli2011gromov} and has been explicitly considered in \cite{sejourne2020unbalanced,redko2020co}, and this is closely connected to our GW distance between MCMSs (see \Cref{def:MCMsGW}) as shown in \Cref{prop:Haibin}. 

Here we give the definition of this ``decoupled" variant of the GW distance.

\begin{definition}[Decoupled Gromov-Wasserstein distance]
Suppose two metric measure spaces $\mathbf{X}=(X,d_X,\mu_X)$ , $\mathbf{Y}=(Y,d_Y,\mu_Y)$ are given. We define the decoupled Gromov-Wasserstein distance $\dGW^{\mathrm{bi}}(\mathbf{X},\mathbf{Y})$ in the following way:
$$\dGW^{\mathrm{bi}}(\mathbf{X},\mathbf{Y})\coloneqq\inf_{\gamma,\gamma'\in\cpl(\mu_X,\mu_Y)}\int\limits_{X\times Y}\int\limits_{X\times Y}|d_X(x,x')-d_Y(y,y')|\gamma'(dx'\times dy')\gamma(dx\times dy).$$
\end{definition}

Obviously, $\dGW^{\mathrm{bi}}(\mathbf{X},\mathbf{Y})\leq \dGW(\mathbf{X},\mathbf{Y})$ in general. Furthermore, this inequality is actually tight as one can see in the following remark.

\begin{remark}
We let $\Gamma_{X,Y}(x,y,x',y')\coloneqq|d_X(x,x')-d_Y(y,y')|$ for any $x,x'\in X$ and $y,y'\in Y$.
If the kernel $\Gamma_{X,Y}:X\times Y\times X\times Y\rightarrow\R$ is negative semi-definite, then one can show that $\dGW^{\mathrm{bi}}(\mathbf{X},\mathbf{Y})=\dGW(\mathbf{X},\mathbf{Y})$ by invoking  \cite[Theorem 4]{sejourne2020unbalanced}. More precisely, if $\gamma,\gamma'$ are the optimal coupling measures achieving the infimum in the definition of $\dGW^{\mathrm{bi}}(\mathbf{X},\mathbf{Y})$, then both $\gamma$ and $\gamma'$ are optimal for $\dGW$, i.e.,
\begin{align*}
    \dGW^{\mathrm{bi}}(\mathbf{X},\mathbf{Y})=\Vert \Gamma_{X,Y}\Vert_{L^1(\gamma\otimes\gamma')}=\Vert \Gamma_{X,Y}\Vert_{L^1(\gamma\otimes\gamma)}=\Vert \Gamma_{X,Y}\Vert_{L^1(\gamma'\otimes\gamma')}=\dGW(\mathbf{X},\mathbf{Y}).
\end{align*}

\end{remark}

Just like the original version, $\dGW^{\mathrm{bi}}$ also becomes a legitimate metric on the collection of metric measure spaces. This is another contribution of our work.
\begin{proposition}
The decoupled Gromov Wasserstein distance $\dGW^{\mathrm{bi}}$ is a legitimate metric on $\mcms$.
\end{proposition}
\begin{proof}
Symmetry is obvious. We need to prove the triangle inequality plus the fact that $\dGW^{\mathrm{bi}}(\mathbf{X},\mathbf{Y})=0$ happens if and only if $\mathbf{X}$ and $\mathbf{Y}$ are isomorphic. The ``if" part is trivial. For the other direction we proceed as follows. Suppose that $\dGW^{\mathrm{bi}}(\mathbf{X},\mathbf{Y})=0$. By \Cref{lemma:Facundolemma} and the compactness of $\cpl(\mu_X,\mu_Y)$ for the weak topology (see \cite[p.49]{villani2021topics}), there must be  optimal couplings $\gamma,\gamma'\in\cpl(\mu_X,\mu_Y)$ such that
\begin{equation}\label{eq:0}
    \sum_{(x,y)\in X\times Y}\sum_{(x',y')\in X\times Y}\vert d_X(x,x')-d_Y(y,y')\vert\,\gamma'(x',y')\,\gamma(x,y)=0.
\end{equation}
\begin{claim}\label{claim:isometry exist}
There exists an isometry $\phi:X\rightarrow Y$ such that
$$\{(x,\phi(x)):x\in X\}=\mathrm{supp}(\gamma)=\mathrm{supp}(\gamma').$$
\end{claim}
\begin{proof}[Proof of \Cref{claim:isometry exist}]
By \Cref{eq:0}, we have that 
\begin{equation}\label{eq:dbilemma}
    d_X(x,x')=d_Y(y,y')
\end{equation}
for any $(x,y)\in\mathrm{supp}(\gamma)$ and $(x',y')\in\mathrm{supp}(\gamma')$.

Fix an arbitrary $x\in X$. Then, since both $\mu_X$ and $\mu_Y$ are fully supported and $X,Y$ are finite, there must exist $y,y'\in Y$ such that $(x,y)\in\mathrm{supp}(\gamma)$ and $(x,y')\in\mathrm{supp}(\gamma')$. Then, $y=y'$ by \Cref{eq:dbilemma}. Now, if there exists $y''\in Y$ such that $(x,y'')\in\mathrm{supp}(\gamma)$, then similarly, we have that $y''=y'$ and thus $y''=y$. In other words, for each $x\in X$, there exists a unique $y\in Y$ such that $(x,y)\in\mathrm{supp}(\gamma)$. Similarly, this same $y\in Y$ is unique such that $(x,y)\in\mathrm{supp}(\gamma')$. Hence, we define $\phi:X\rightarrow Y$ by letting $\phi(x)$ be the unique $y\in Y$ such that $(x,\phi(x))\in\mathrm{supp}(\gamma)$. It is obvious that $\phi$ is bijective and satisfies that $\{(x,\phi(x)):x\in X\}=\mathrm{supp}(\gamma)=\mathrm{supp}(\gamma')$. By \Cref{eq:dbilemma}, we conclude that $\phi$ is an isometry.
\end{proof}

Based on the claim above, consider an arbitrary Borel subset $A\subseteq X$. Then,
$$\mu_X(A)=\gamma(A\times Y)=\gamma(A\times Y\cap A\times\phi(A))=\gamma(A\times\phi(A))=\mu_Y(\phi(A)).$$
Hence, $\phi$ is a isomorphism between $\mathbf{X}$ and $\mathbf{Y}$.

Finally, for the triangle inequality, fix finite metric measure space $\mathbf{X}$,$\mathbf{Y}$, and $\mathbf{Z}$. Notice first that for all $x,x'\in X$, $y,y'\in Y$, and $z,z'\in Z$,

$$\Gamma_{X,Y}(x,y,x',y')\leq\Gamma_{X,Z}(z,x,z',x')+\Gamma_{Z,Y}(y,z,y',z').$$

Next, fix arbitrary coupling measures $\gamma_1,\gamma_1'\in\cpl(\mu_X,\mu_Z)$ and $\gamma_2,\gamma_2'\in\cpl(\mu_Z,\mu_Y)$. By the Gluing Lemma (see \cite[Lemma 7.6]{villani2021topics}), there exist probability measures $\pi,\pi'\in\prob(X\times Y\times Z)$ with marginals $\gamma_1,\gamma_1'$ on $X\times Z$ and $\gamma_2,\gamma_2'$ on $Z\times Y$. Let $\gamma_3,\gamma_3'$ be the marginal of $\pi,\pi'$ on $X\times Y$. Then, by the triangle inequality of $L^1$ norm,
\begin{align*}
\dGW^{\mathrm{bi}}(\mathbf{X},\mathbf{Y})&\leq\Vert \Gamma_{X,Y}\Vert_{L^1(\gamma_3\otimes\gamma_3')}\\
&=\Vert \Gamma_{X,Y}\Vert_{L^1(\pi\otimes\pi')}\\
&\leq\Vert \Gamma_{X,Z}\Vert_{L^1(\pi\otimes\pi')}+\Vert \Gamma_{Z,Y}\Vert_{L^1(\pi\otimes\pi')}\\
&=\Vert \Gamma_{X,Z}\Vert_{L^1(\gamma_1\otimes\gamma_1')}+\Vert \Gamma_{Z,Y}\Vert_{L^1(\gamma_2\otimes\gamma_2')}.
\end{align*}
Since the choice of $\gamma_1,\gamma_1',\gamma_2,\gamma_2'$ are arbitrary, by taking the infimum one can conclude
$$\dGW^{\mathrm{bi}}(\mathbf{X},\mathbf{Y})\leq \dGW^{\mathrm{bi}}(\mathbf{X},\mathbf{Z})+\dGW^{\mathrm{bi}}(\mathbf{Z},\mathbf{Y}).$$
\end{proof}
\subsection{Examples when $\dWL$ fails to separate graphs}\label{sec:examples}
\begin{example}[Constant labels]\label{ex:constant label}
Let $G_1$ be a claw and $G_2$ be a path with four nodes; see \Cref{fig:example2}. Let the label functions $\ell_{G_i}$ for $i=1,2$ be constant and equal to $1$ for both graphs. 

In the first step of the WL test, we find
\[ L_1((G_1,\ell_{G_1})) = 
    \mbl (1, \mbl 1, 1, 1\mbr), (1, \mbl 1\mbr), 
    (1, \mbl 1\mbr),
    (1, \mbl 1\mbr)\mbr\]
and 
\[   L_1((G_2,\ell_{G_2})) = 
    \mbl (1, \mbl 1\mbr), (1, \mbl 1\mbr), (1, \mbl 1, 1\mbr), (1, \mbl 1, 1\mbr)\mbr.\]
 Since $L_1((G_1,\ell_{G_1})) \neq L_1((G_2,\ell_{G_2}))$, $(G_1,\ell_{G_1})$ and $(G_2,\ell_{G_2})$ are recognized as non-isomorphic by the WL test. Notice that within the first step, the WL test collects degree information and comparing $L_1((G_1,\ell_{G_1}))$ and $L_1((G_2,\ell_{G_2}))$ is equivalent to comparing the multisets of degrees w.r.t. $G_1$ and $G_2$. However, for $\dWL\!\lc (\mX(G_1),\ell_{G_1}),(\mX(G_2),\ell_{G_2})\rc $ (abbreviated to $\dWL\!\lc \mX(G_1),\mX(G_2)\rc $ in \Cref{fig:example2}), because of the normalization inherent to the Markov chains $m_\bullet^{G_1}$ and $m_\bullet^{G_2}$, each step inside the hierarchy pertaining to the WL distance cannot collect degree information when the labels are constant.
\end{example}

\begin{example}[Degree label]\label{ex:degree label}
Let $G_1$ be a two-point graph consisting of a single edge with the vertex set $\{v_1, v_2\}$. Let $G_2$ be a four-point graph consisting of two disjoint edges denoted by $\{u_1, u_2\}$ and $\{u_3, u_4\}$; see \Cref{fig:example3}. For each $i=1,2$, let $\ell_{G_i}$ be the degree label function for both graphs.  

In the first step of the WL test, 
$$L_1((G_1,\ell_{G_1})) = \mbl (1, \mbl 1\mbr), (1, \mbl 1\mbr)\mbr$$ 
and 
\[  L_1((G_2,\ell_{G_2})) = 
    \mbl(1, \mbl 1\mbr), (1, \mbl 1\mbr), (1, \mbl 1\mbr), (1, \mbl 1\mbr)\mbr.\]
Then in the first step of the WL test, the two labeled graphs are already distinguished as non-isomorphic. 

In the case of $\dWL\!\lc (\mX(G_1),\ell_{G_1}),(\mX(G_2),\ell_{G_2})\rc $, notice that  $(\ell_{G_1})_\# m_x^{G_1}(z) = 1$ if $z = 1$ and $0$ otherwise for both $x = v_1$ and $x= v_2$. Hence, $\WLh{1}{(\mX(G_1),\ell_{G_1})}(v_1) = \WLh{1}{(\mX(G_1),\ell_{G_1})}(v_2)$. Similarly, for $G_2$, 
$$\WLh{1}{(\mX(G_2),\ell_{G_2})}(u_1) = \cdots =  \WLh{1}{(\mX(G_2),\ell_{G_2})}(u_4) = \WLh{1}{(\mX(G_1),\ell_{G_1})}(v_1).$$

It is not hard to show inductively that for each $k\in\N$,
\[
    \WLh{k}{(\mX(G_1),\ell_{G_1})} (v_1) = \WLh{k}{(\mX(G_1),\ell_{G_1})} (v_2)=
    \WLh{k}{(\mX(G_2),\ell_{G_2})} (u_1) = \cdots = \WLh{k}{(\mX(G_2),\ell_{G_2})} (u_4).
\] 
Then, for each $k \in\N$, $$\mathfrak{L}_1((\mX(G_1),\ell_{G_1}))=\lc\WLh{k}{(\mX(G_1),\ell_{G_1})}\rc_\# \mu_{G_1} = \lc\WLh{k}{(\mX(G_2),\ell_{G_2})}\rc_\# \mu_{G_2}=\mathfrak{L}_1((\mX(G_2),\ell_{G_2}))$$
and thus $\dWLk((\mX(G_1),\ell_{G_1}),(\mX(G_2),\ell_{G_2})) = 0$ which implies that $\dWL((\mX(G_1),\ell_{G_1}), (\mX(G_2),\ell_{G_2}) )= 0$. 

Notice that the standard WL test with degree labels is able to capture (and therefore compare) information about the number of nodes in the graph. On the other hand, $(\mX(G_1),\ell_{G_1})$ and $(\mX(G_2),\ell_{G_2})$ cannot be distinguished by the WL distance because of the normalization of the reference measures, $\mu_{G_1}$ and $\mu_{G_2}$.
\end{example}
\begin{figure*}
\centering
  \includegraphics[width=0.8\linewidth]{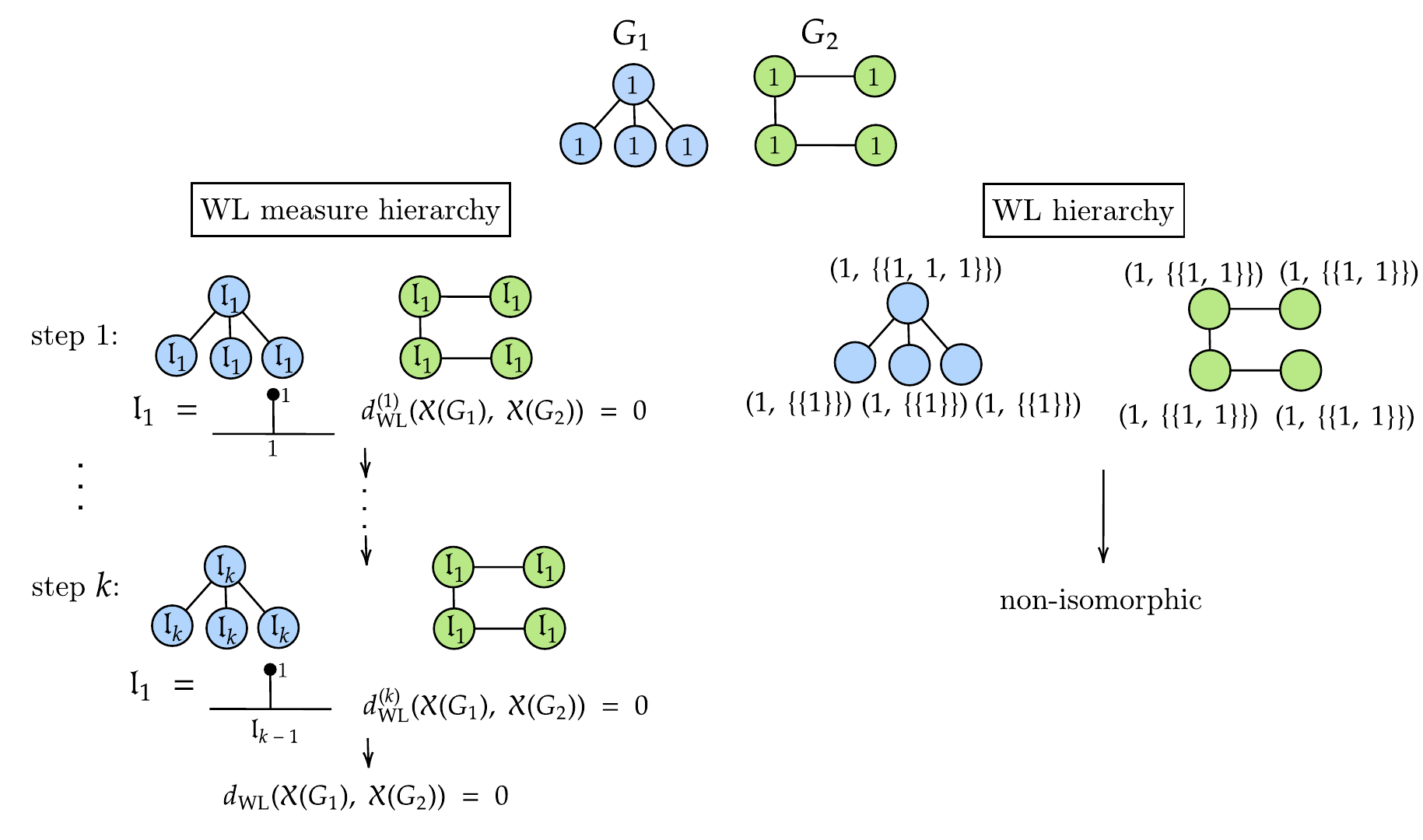}
  \caption{\textbf{Illustration of \Cref{ex:constant label}}. Notice that if we start from constant labels, one step of the WL hierarchy will collect degree information for the vertices. In contrast, because of the normalization of the Markov kernel, a single step of the WL measure hierarchy with constant labels will not be able to accumulate the same information.}
  \label{fig:example2}
\end{figure*}

\begin{figure*}
\centering
  \includegraphics[width=0.8\linewidth]{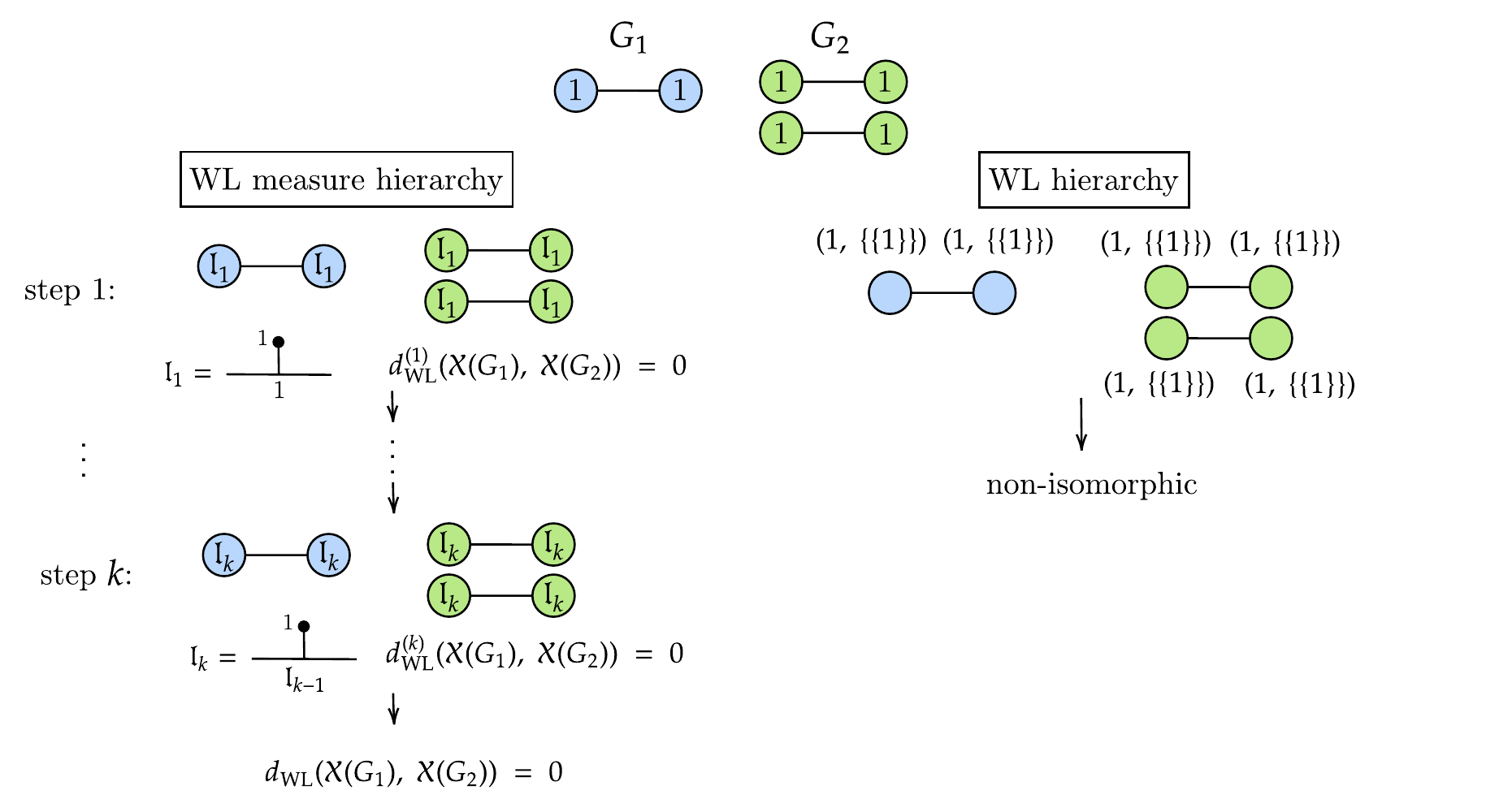}
  \caption{\textbf{Illustration of \Cref{ex:degree label}}. One step of the WL hierarchy with degree labels can distinguish graphs of different sizes whereas the normalization of $\mu_{G_1}$ and $\mu_{G_2}$ does not allow graph size to be distinguished when the label function is degree. }
  \label{fig:example3}
\end{figure*}
\subsection{A basic lower bound for $\dGW^\mathrm{MCMS}$}\label{sec:basic lower bound}
One can produce some basic lower bounds for $\dGW^\mathrm{MCMS}$ by invoking the notion of \emph{diameter} for MCMSs which we define below. We first introduce the one point MCMS.
\begin{example}
The one point MCMS is the tuple $*\coloneqq(\{*\},(0),\delta_\ast,\delta_*)$.
\end{example}

\begin{definition}[MCMS diameter]
For each $k\geq 1$ and a MCMS $(\mX,d_X)$, we define
\begin{align*}
    \diam_{\mathrm{MCMS}}^{\scriptscriptstyle{(k)}}((\mX,d_X))&\coloneqq\dGW^{\scriptscriptstyle{(k)}}((\mX,\ell_X),\ast),
\end{align*}
and
\begin{align*}
    \diam_{\mathrm{MCMS}}((\mX,d_X))&\coloneqq\dGW^\mathrm{MCMS}((\mX,d_X),\ast).
\end{align*}
\end{definition}

Notice that $\cpl(\mu_X,\mu_*)=\{\mu_X\otimes\delta_*\}$ and $\cpl^{\scriptscriptstyle{(k)}}\!\left(m_\bullet^X,\delta_*\right)=\left\{m_\bullet^{X,\otimes k}\otimes\delta_*\right\}$. Then, it turns out that the diameter of $\mX$ is independent of $k$:
\begin{align*}
    \diam_{\mathrm{MCMS}}^{\scriptscriptstyle{(k)}}((\mX,d_X))&=\dGW^{\scriptscriptstyle{(k)}}((\mX,d_X),\ast)=\int_X\int_X\int_X d_X(x,x')\,m_{x''}^{X,\otimes k}(dx')\mu_X(dx'')\mu_X(dx)\\
    &=\int_X\int_X d_X(x,x')\mu_X(dx')\mu_X(dx)
\end{align*}

By the triangle inequality, one can prove the following result.

\begin{proposition}
For all two MCMSs $(\mX,d_X)$,$(\mY,d_Y)$ and $k\geq 1$, we have
$$\left\vert\int_X\int_X d_X(x,x')\mu_X(dx')\mu_X(dx)-\int_Y\int_Y d_Y(y,y')\mu_Y(dy')\mu_Y(dy)\right\vert\leq\dGW^{\scriptscriptstyle{(k)}}\!\lc(\mX,d_X),(\mY,d_Y)\rc ,$$
and
$$\left\vert\int_X\int_X d_X(x,x')\mu_X(dx')\mu_X(dx)-\int_Y\int_Y d_Y(y,y')\mu_Y(dy')\mu_Y(dy)\right\vert\leq\dGW^\mathrm{MCMS}\!\lc(\mX,d_X),(\mY,d_Y)\rc.$$
\end{proposition}

\subsection{More details on the complexity of computing the WL distance}\label{sec:algorithm and analysis}

In the following subsections, we provide an algorithm for computing the WL distance and complexity analysis for both computing the WL distance and its lower bound defined in \Cref{sec:1 lower bound for dwlk}. 

\subsubsection{Computation of the WL distance}
In this section, we devise an algorithm (with pseudocode in \Cref{alg:dWL}) for computing $\dWLk$ and establish the following complexity analysis.

\begin{proposition}
For any fixed $k\in\N$, computing $\dWLk$ between any LMMCs $(\mX,\ell_X)$ and $(\mY,\ell_Y)$ can be achieved in time at most $O(k\,n^5 \log(n))$ where $n = \max(|X|, |Y|)$.
\end{proposition}

Recall from \Cref{eq:definition of dwlk} that the WL distance of depth $k$ is defined as 
\begin{align*}
    \dWLk\!\lc(\mX,\ell_X),(\mY,\ell_Y)\rc &\coloneqq\dW\!\lc \lc \WLh{k}{(\mX,\ell_X)}\rc_\#\mu_X,\lc \WLh{k}{(\mY,\ell_Y)}\rc_\#\mu_Y\rc\\
    &=\inf_{\gamma\in\cpl(\mu_X,\mu_Y)}\int\limits_{X\times Y}\dW\!\lc \WLh{k}{(\mX,\ell_X)}(x), \WLh{k}{(\mY,\ell_Y)}(y)\rc\gamma(dx\times dy).
\end{align*}
In order to compute $\dWLk\!\lc(\mX,\ell_X),(\mY,\ell_Y)\rc $, we must first compute $\dW(\WLh{k}{(\mX,\ell_X)}(x), \WLh{k}{(\mY,\ell_Y)}(y))$ for each $x \in X$ and $y \in Y$. To do this, we introduce some notation. For each $i=1,\ldots,k$, we let $C_i$ denote the $|X|\times |Y|$ matrix such that for each $x \in X$ and $y \in Y$, 
$$C_{i}(x,y)\coloneqq\dW\!\lc\WLh{i}{(\mX,\ell_X)}(x), \WLh{i}{(\mY,\ell_Y)}(y)\rc.$$
We also let $C_0$ denote the matrix such that $C_{0}(x,y)\coloneqq\left\|\ell_X(x)-\ell_Y(y)\right\|$ for each $x\in X$ and $y\in Y$.
Then, our task is to compute the matrix $C_k$. For this purpose, we consecutively compute the matrix $C_i$ for $i=1,\ldots,k$: Given matrix $C_{i-1}$, since $\WLh{i}{(\mX,\ell_X)}(x) = \lc\WLh{i-1}{(\mX,\ell_X)}\rc_\# m_x^X$ and $\WLh{i}{(\mY,\ell_Y)}(y) = \lc\WLh{i-1}{(\mY,\ell_Y)}\rc_\# m_y^Y$, computing $$\dW\!\lc\WLh{i}{(\mX,\ell_X)}(x), \WLh{i}{(\mY,\ell_Y)}(y)\rc=\inf_{\gamma\in\cpl(m_x^X,m_y^Y)}\int\limits_{X\times Y}\dW\!\lc \WLh{i-1}{(\mX,\ell_X)}(x), \WLh{i-1}{(\mY,\ell_Y)}(y)\rc\gamma(dx\times dy).$$ is reduced to solving the optimal transport problem with $C_{i - 1}$ as the cost matrix and $m_x^X$ and $m_y^Y$ as the source and target distributions, which can be done in $O(n^3\log(n))$ time \cite{pele2009fast}.  Thus, for each $i$, computing $C_i$ given that we know $C_{ i- 1}$, requires $O(n^2 \cdot n^3 \log(n))$. Finally, we need $O(n^3\log(n))$ time to compute $\dWLk\!\lc(\mX,\ell_X),(\mY,\ell_Y)\rc$ based on solving an optimal transport problem with cost matrix $C_k$ and with $\mu_X$ and $\mu_Y$ being the source and target distributions, respectively.

Therefore, the total time needed to compute $\dWLk\!\lc(\mX,\ell_X),(\mY,\ell_Y)\rc$ is 
$$k\cdot O(n^5 \log(n))+O(n^3\log(n))=O(k\,n^5 \log(n)).$$

For any $n\in \N$, $\dWL^{\scriptscriptstyle{(2n)}}$ generates a distance between graph induced LMMCs with size bounded by $n$. By \Cref{cor:dWL convergence}, $\dWL^{\scriptscriptstyle{(2n)}}$ has the same discriminating power as the WL test in separating graphs with size bounded by $n$. Now, given labeled graphs $(G_1,\ell_{G_1})$ and $(G_2,\ell_{G_2})$ so that $\max\lc|V_{G_1}|,|V_{G_2}|\rc\leq n$, computing $\dWL^{\scriptscriptstyle{(2n)}}\!\lc(\mX_q(G_1),\ell_{G_1}),(\mX_q(G_2),\ell_{G_2})\rc$ takes time at most $O(n^6 \log(n))$.

\begin{algorithm}[htb]
\caption{$\dWLk$ computation}
\begin{algorithmic}[1]
\STATE \textbf{Input:} The depth $k\in\N$, and two finite LMMCs $\lc X = \left\{ x_1, x_2, ..., x_n \right\},m_\bullet^X,\mu_X,\ell_X:X\rightarrow\R\rc$ and $\lc Y = \left\{ y_1, ..., y_m \right\},m_\bullet^Y,\mu_Y,\ell_Y:Y\rightarrow\R\rc$
\STATE \textbf{Initialization:} $P = C =\mathrm{zeros}(n,m)$
\FOR{$i\in[n],j\in[m]$}
\STATE $P(i,j) = |\ell_X(x_i)-\ell_Y(y_j)|$
\ENDFOR

\FOR{$l\in[k]$}
\FOR{$i\in[n],j\in[m]$}
\STATE $C(i,j)=\inf_{\gamma\in\cpl\lc m_{x_i}^X,m_{y_j}^Y\rc}\sum_{a\in[n],b\in[m]} P(a,b)\gamma(a,b)$
\ENDFOR

\STATE $P=C$
\ENDFOR

$D=\inf_{\gamma\in\cpl\lc \mu_X,\mu_Y\rc}\sum_{i\in[n],j\in[m]} C(i,j)\gamma(i,j)$

\STATE \textbf{Output:}  $D$
\end{algorithmic}
\label{alg:dWL}
\end{algorithm}

\subsubsection{Computation of the lower bound distance}

Recall from \Cref{sec:1 lower bound for dwlk} that the WL lower bound distance was defined as 
\begin{equation*}
    d_{\mathrm{WLLB}}^{\scriptscriptstyle{(k)}}\!\lc(\mX,\ell_X),(\mY,\ell_Y)\rc\coloneqq \!\inf_{\gamma\in \cpl(\mu_X,\mu_Y)}\!\!\! \int\limits_{X\times Y}\!\!\!\!\dW\!\lc (\ell_X)_\#m_x^{X,\otimes k},(\ell_Y)_\#m_y^{Y,\otimes k}\rc \!\gamma(dx\times dy).
\end{equation*}
Given two finite LMMCs,
$$\lc\mX,\ell_X:X\rightarrow\R\rc\text{ where }\mX=\lc X=\left\{ x_1, x_2, ..., x_n \right\},m_\bullet^X,\mu_X\rc$$
and
$$\lc\mY,\ell_Y:Y\rightarrow\R\rc\text{ where }\mY=\lc Y = \left\{ y_1,y_2, ..., y_m \right\},m_\bullet^Y,\mu_Y\rc$$
we represent their Markov kernels as two transition matrices, $M_{\mX }$ and $M_{\mY}$, respectively. Then, $k$-Markov kernels $m_\bullet^{X,\otimes k}$ and $m_\bullet^{Y,\otimes k}$ are expressed as matrices $M_{\mX }^k$ and $M_{\mY}^k$, respectively. Assume that $n \geq m$. Then computing the $k$-Markov kernels of $\mX$ and $ \mY$ will require $O(n^3\log(k))$ time where $O(n^3)$ is time needed for matrix multiplication. Then since $(\ell_X)_\# m_x^{X, \otimes k}$ and $(\ell_Y)_\# m_y^{Y, \otimes k}$ are both distributions in $\R$, by \cite{vallender1974calculation}, $\dW\!\lc (\ell_X)_\#m_x^{X,\otimes k},(\ell_Y)_\#m_y^{Y,\otimes k}\rc$ can be computed in $O(n)$ time for each $x \in X$ and $y \in Y$. Finally, computing $d_{\mathrm{WLLB}}^{\scriptscriptstyle{(k)}}$ can be formulated as finding the optimal transport cost where each entry of the cost matrix is defined as $\dW\!\lc (\ell_X)_\#m_x^{X,\otimes k},(\ell_Y)_\#m_y^{Y,\otimes k}\rc$ and the source and target distributions are $\mu_X, \mu_Y$ respectively. Recall from the previous section that $\mu_X$ and $\mu_Y$ are normalized degree distributions for $\mX$ and $\mY$. Therefore, the overall time complexity is 
$$O(n^3\log(k) ) + O(n^3 \log(n))=O(n^3 \log(kn)).$$

\subsection{Experiments} \label{appendix:experiments}

\subsubsection{Experimental setup}
We use several publicly available graph benchmark datasets 
from TUDatasets \cite{Morris+2020} and evaluate the performance of our WL distance $\dWLk$ distance as well as $d_{\mathrm{WLLB}}^{\scriptscriptstyle{(k)}}$ (lower bound of our $\dWL$ which is more efficient to compute) through two types of graph classification experiments compared with several representative methods. Note that for all of our experiments, we use $q = 0.6$ to transform every graph $G$ into the Markov chain $\mX_q(G)$.
We use 1-Nearest Neighbors classifier in the first graph classification experiment and for the second experiment, we use support vector machines (SVM). For the first experiment, we compare classification accuracies with the WWL distance \cite{togninalli2019wasserstein} (cf. \Cref{eq:WWL distance}). For the second graph classification task, we run an SVM using the indefinite kernel matrices $\exp\lc-\gamma d_{\mathrm{WLLB}}^{\scriptscriptstyle{(k)}}\rc$ and $\exp\lc-\gamma \dWLk\rc$, which are seen as noisy observations of the true positive semi-definite kernels \cite{luss2009support}. Additionally, for the SVM method, we cross validate the parameter $C \in \{10^{-3}, \dots, 10^3\}$ and the parameter $\gamma \in \{10^{-3}, \dots, 10^3\}$. We compare classification accuracies with the WWL kernel \cite{togninalli2019wasserstein}, the WL kernel \cite{shervashidze2011weisfeiler}, and the Weisfeiler-Lehman optimal assignment kernel (WL-OA) \cite{kriege2016valid}. Note that we only use WWL distance in the 1-NN graph classification experiment since the WL-OA and WL kernels are not defined in terms of a distance unlike the WWL kernel.

In addition to the full accuracies for $k \in \{1, 2, 3, 4\}$ for $\dWLk$ and $d_{\mathrm{WLLB}}^{\scriptscriptstyle{(k)}}$ with the degree label, call this $f_1$, we also evaluate $\dWLk$ and $d_{\mathrm{WLLB}}^{\scriptscriptstyle{(k)}}$ with the label function $f_2(G,v) = \frac{1}{|V_G|} + \deg_G(v)$ for any graph $G$ and vertex $v\in V_G$. Note that $f_2$ is a relabeling of any label function, assigning any constant $c$ to each vertex, via the injective map $g:\{c\}\times\N\times\N\rightarrow\R$ sending $(c,n_1,n_2)$ to $n_1+\frac{1}{n_2}$ as described in \Cref{sec:comparison with WL test}. So under $f_2$, $\dWLk$ is as discriminative as the $k$-step WL test. Thus, we also evaluate the performance of the WWL distance/kernel, WL, and WL-OA kernels using only degree label. Additionally, we report the best accuracies for WWL, WL, and WL-OA for iterations $1, \dots, 4$.

\subsubsection{Extra experimental results}
In \Cref{tab:full nn experiments} and \Cref{tab:full svm}, we have included the 1-NN and SVM classification accuracies for $k \in \{1, 2, 3, 4\}$, respectively. 
\begin{table*}[t]
\caption{1-Nearest Neighbor classification accuracy. Let  $f_1(G,v) = \deg_G(v)$, $f_2(G,v) = \frac{1}{|V_G|} + \deg_G(v)$}
\label{tab:full nn experiments}
\vskip 0.15in
\begin{center}
\resizebox{\columnwidth}{!}{\begin{tabular}{lccccccc}
\toprule
Method & MUTAG  & PROTEINS & PTC-FM & PTC-MR & IMDB-B & IMDB-M & COX2\\
\midrule
$d_{\mathrm{WL}}^{\scriptscriptstyle{(1)}}$, $f_1$ & 90.5 $\pm$ 6.5 & 61.8 $\pm$ 4.3 & 60.0 $\pm$ 8.5 & 53.9 $\pm$ 7.1 & 70.1 $\pm$ 4.7 & 41.1 $\pm$ 3.9 & 73.8 $\pm$ 3.6\\
$d_{\mathrm{WL}}^{\scriptscriptstyle{(2)}}$, $f_1$ & 92.1 $\pm$ 6.3 & 60.8 $\pm$ 4.4 & 62.2 $\pm$ 7.4 & 56.2 $\pm$ 6.3 & 69.9 $\pm$ 4.2 & 41.1 $\pm$ 4.7 & 74.2 $\pm$ 4.5\\
$d_{\mathrm{WL}}^{\scriptscriptstyle{(3)}}$, $f_1$ & 91.1 $\pm$ 4.3 & 60.8 $\pm$ 3.5 & 59.4 $\pm$ 8.2 & 54.0 $\pm$ 7.7 & 69.4 $\pm$ 3.9 & 41.0 $\pm$ 4.8 & 74.2 $\pm$ 3.9\\
$d_{\mathrm{WL}}^{\scriptscriptstyle{(4)}}$, $f_1$ & 90.1 $\pm$ 4.8 & 63.0 $\pm$ 3.8 & 59.1 $\pm$ 8.3 & 54.2 $\pm$ 6.8 & 70.2 $\pm$ 4.3 & 41.3 $\pm$ 4.8 & 76.1 $\pm$ 5.5 \\
\midrule
$d_{\mathrm{WL}}^{\scriptscriptstyle{(1)}}$, $f_2$ & 91.6 $\pm$ 7.1 & 63.3 $\pm$ 4.4 & 57.5 $\pm$ 6.0 & 51.9 $\pm$ 9.3 & \textbf{71.4 $\pm$ 4.5} & 40.6 $\pm$ 5.3 & 72.5 $\pm$ 4.5\\
$d_{\mathrm{WL}}^{\scriptscriptstyle{(2)}}$, $f_1$ & 91.1 $\pm$ 5.8 & 62.4 $\pm$ 3.4 & 58.2 $\pm$ 8.2 & 56.2 $\pm$ 7.6 & 70.4 $\pm$ 4.5 & \textbf{41.6 $\pm$ 4.3} & 74.0 $\pm$ 4.7\\
$d_{\mathrm{WL}}^{\scriptscriptstyle{(3)}}$, $f_1$ & 91.5 $\pm$ 5.8 & 63.4 $\pm$ 3.9 & 58.5 $\pm$ 7.9 & 53.4 $\pm$ 8.4 & \textbf{71.4 $\pm$ 5.9} & 40.6 $\pm$ 4.3 & 74.6 $\pm$ 4.4\\
$d_{\mathrm{WL}}^{\scriptscriptstyle{(4)}}$, $f_1$ & \textbf{92.6 $\pm$ 4.8} & 63.3 $\pm$ 4.9 & 58.5 $\pm$ 8.0 & 54.8 $\pm$ 7.9 & 71.2 $\pm$ 5.1 & 40.7 $\pm$ 4.8 & 75.9 $\pm$ 4.9 \\
\midrule
$d_{\mathrm{WLLB}}^{\scriptscriptstyle{(1)}}$, $f_1$ & 87.3 $\pm$ 1.9 & 64.0 $\pm$ 2.3 & \textbf{62.5 $\pm$ 8.5} & 57.4 $\pm$ 6.8 & 69.0 $\pm$ 3.9 & 40.6 $\pm$ 3.8 & 75.1 $\pm$ 3.8\\
$d_{\mathrm{WLLB}}^{\scriptscriptstyle{(2)}}$, $f_1$ & 86.8 $\pm$ 3.7 &  \textbf{66.2 $\pm$ 2.2} & 60.0 $\pm$ 8.1 & 53.4 $\pm$ 6.4 & 69.4 $\pm$ 3.2 & 40.1 $\pm$ 3.6 & 75.1 $\pm$ 3.8\\
$d_{\mathrm{WLLB}}^{\scriptscriptstyle{(3)}}$, $f_1$ & 85.2 $\pm$ 3.5 & 64.6 $\pm$ 2.2 & 58.0 $\pm$ 1.1 & 54.5 $\pm$ 9.1 & 69.8 $\pm$ 3.3 & 40.1 $\pm$3.9 & \textbf{81.2 $\pm$ 5.3}\\
$d_{\mathrm{WLLB}}^{\scriptscriptstyle{(4)}}$, $f_1$ & 84.7 $\pm$ 3.1 & 65.4 $\pm$ 2.3 & 58.0 $\pm$ 1.1 & 52.0 $\pm$ 9.1 & 69.9 $\pm$ 2.5 & 40.1 $\pm$3.6 & 80.4 $\pm$ 2.3 \\
\midrule
$d_{\mathrm{WLLB}}^{\scriptscriptstyle{(1)}}$, $f_2$ & 87.3 $\pm$ 2.5 & 64.7 $\pm$ 1.4 & \textbf{62.5 $\pm$ 7.4} & \textbf{57.8 $\pm$ 6.8} & 69.0 $\pm$ 3.9 & 40.4 $\pm$ 3.6 & 75.5 $\pm$ 3.7\\
$d_{\mathrm{WLLB}}^{\scriptscriptstyle{(2)}}$, $f_2$ & 86.3 $\pm$ 3.6 & 65.6 $\pm$ 2.2 & 60.0 $\pm$ 8.1 & 53.4 $\pm$ 6.4 & 69.2 $\pm$ 3.2 & 40.2 $\pm$ 3.6 & 77.0 $\pm$ 4.9\\
$d_{\mathrm{WLLB}}^{\scriptscriptstyle{(3)}}$, $f_2$ & 85.3 $\pm$ 3.6 & 64.3 $\pm$ 1.1 & 58.0 $\pm$ 10.8 & 54.7 $\pm$ 9.1 & 69.7 $\pm$ 3.1 & 40.1 $\pm$ 3.9 & 80.4 $\pm$ 4.4\\
$d_{\mathrm{WLLB}}^{\scriptscriptstyle{(4)}}$, $f_2$ & 84.7 $\pm$ 3.0 & 64.8 $\pm$ 1.8 & 58.0 $\pm$ 10.9 & 52.0 $\pm$ 9.2 & 69.4 $\pm$ 2.5 & 40.2 $\pm$ 3.8 & 80.4 $\pm$ 4.4\\
\bottomrule
WWL & 85.1 $\pm$ 6.5 & 64.7 $\pm$ 2.8 & 58.2 $\pm$ 8.5 & 54.3 $\pm$ 7.9 & 65.0 $\pm$ 3.3 & 40.0 $\pm$ 3.3 &  76.1 $\pm$ 5.6 \\
\bottomrule
\end{tabular}}
\end{center}
\vskip -0.1in
\end{table*}

\begin{table*}[t]
\caption{SVM classification accuracy. Let $f_1(G,v) = \deg_G(v)$, $f_2(G,v) = \frac{1}{|V_G|} + \deg_G(v)$}
\label{tab:full svm}
\vskip 0.15in
\begin{center}
\resizebox{\columnwidth}{!}{\begin{tabular}{lcccccccr}
\toprule
Method & MUTAG  & PROTEINS & PTC-FM  & PTC-MR & IMDB-B & IMDB-M & COX2\\
\midrule
$d_{\mathrm{WL}}^{\scriptscriptstyle{(1)}}$, $f_1$ & 87.7 $\pm$ 6.2 & 71.7 $\pm$ 3.3 & 57.6 $\pm$ 6.4 & 55.5 $\pm$ 4.5 & 74.5 $\pm$ 4.1 & 51.3 $\pm$ 3.0 & 78.1 $\pm$ 0.8\\
$d_{\mathrm{WL}}^{\scriptscriptstyle{(2)}}$, $f_1$ & 89.9 $\pm$ 6.4 & 71.3 $\pm$ 3.3 & 59.3 $\pm$ 3.7 & 54.9 $\pm$ 6.3 & 75.0 $\pm$ 3.0 & 51.4 $\pm$ 3.4 & 78.1 $\pm$ 0.8\\
$d_{\mathrm{WL}}^{\scriptscriptstyle{(3)}}$, $f_1$ & 87.6 $\pm$ 8.8 & 72.6 $\pm$ 3.1 & 59.0 $\pm$ 3.7 & 57.8 $\pm$ 7.9 & 74.9 $\pm$ 5.1 & 51.6 $\pm$ 4.0 & 78.1 $\pm$ 0.8\\
$d_{\mathrm{WL}}^{\scriptscriptstyle{(4)}}$, $f_1$ & 87.7 $\pm$ 4.1 & 72.4 $\pm$ 4.1 & 62.1 $\pm$ 3.9 & 56.7 $\pm$ 3.7 & 75.9 $\pm$ 2.7 & 51.4 $\pm$ 3.2 & 78.1 $\pm$ 0.8\\
\midrule
$d_{\mathrm{WL}}^{\scriptscriptstyle{(1)}}$, $f_3$ & 87.3 $\pm$ 8.2 & 71.1 $\pm$ 3.2 & 58.7 $\pm$ 7.6 & 55.2 $\pm$ 5.5 & 74.1 $\pm$ 4.1 & 50.5 $\pm$ 4.5 & 78.1 $\pm$ 0.8\\
$d_{\mathrm{WL}}^{\scriptscriptstyle{(2)}}$, $f_3$ & 86.2 $\pm$ 7.4 & 73.5 $\pm$ 2.8  & 60.2 $\pm$ 5.3 & 54.0 $\pm$ 6.4 & 75.0 $\pm$ 4.5 & 51.4 $\pm$ 3.9 & 78.1 $\pm$ 0.8\\
$d_{\mathrm{WL}}^{\scriptscriptstyle{(3)}}$, $f_3$ & 88.8 $\pm$ 5.4 & \textbf{74.5 $\pm$ 2.9} & 61.6 $\pm$ 5.3 & 59.2 $\pm$ 6.4 & 75.4 $\pm$ 5.4 & 50.8 $\pm$ 4.0 & 77.0 $\pm$ 1.5\\
$d_{\mathrm{WL}}^{\scriptscriptstyle{(4)}}$, $f_3$ & 87.2 $\pm$ 5.8 & 73.9 $\pm$ 3.5 & 60.4 $\pm$ 5.1 & 54.3 $\pm$ 7.7 & 75.7 $\pm$ 3.7 & 50.8 $\pm$ 3.2 & 78.1 $\pm$ 0.8\\
\midrule
$d_{\mathrm{WLLB}}^{\scriptscriptstyle{(1)}}$, $f_1$ & 87.9 $\pm$ 5.9 &  68.0 $\pm$ 1.3 & 59.6 $\pm$ 6.4 & 57.4 $\pm$ 8.1 & 74.7 $\pm$ 2.5 & \textbf{52.0 $\pm$ 1.8}  & 78.1 $\pm$ 0.8\\
$d_{\mathrm{WLLB}}^{\scriptscriptstyle{(2)}}$, $f_1$ & 89.4 $\pm$ 5.2 &  68.9 $\pm$ 1.9 & 58.6 $\pm$ 5.7 & 59.0 $\pm$ 8.3 & \textbf{75.1 $\pm$ 2.2} & 50.8 $\pm$ 1.6 & 78.1 $\pm$ 0.8\\
$d_{\mathrm{WLLB}}^{\scriptscriptstyle{(3)}}$, $f_1$ & \textbf{90.0 $\pm$ 5.6} &  68.6 $\pm$ 1.6 & 57.3 $\pm$ 6.2 & 58.7 $\pm$ 8.1 & 75.2 $\pm$ 2.1 & 51.0 $\pm$ 1.6 & 78.1 $\pm$ 0.8\\
$d_{\mathrm{WLLB}}^{\scriptscriptstyle{(4)}}$, $f_1$ & 89.4 $\pm$ 5.2 &  66.7 $\pm$ 2.0 & 58.2 $\pm$ 6.1 & 56.6 $\pm$ 7.2 & 74.5 $\pm$ 2.0 & 50.3 $\pm$ 1.4 & 77.5 $\pm$ 2.1\\
\midrule
$d_{\mathrm{WLLB}}^{\scriptscriptstyle{(1)}}$, $f_2$ & 88.9 $\pm$ 4.5 & 70.5 $\pm$ 1.0 & 60.5 $\pm$ 5.4 & 56.6 $\pm$ 8.0 & 75.0 $\pm$ 2.5 & \textbf{52.0 $\pm$ 1.8} & 78.1 $\pm$ 0.8\\
$d_{\mathrm{WLLB}}^{\scriptscriptstyle{(2)}}$, $f_2$ & \textbf{90.0 $\pm$ 4.2} & 70.0 $\pm$ 1.4 & 58.0 $\pm$ 5.6 & 58.5 $\pm$ 8.9 & 75.1 $\pm$ 2.2 & 50.8 $\pm$ 1.6 & 78.1 $\pm$ 0.8 \\
$d_{\mathrm{WLLB}}^{\scriptscriptstyle{(3)}}$, $f_2$ & \textbf{90.0 $\pm$ 4.2} & 70.3 $\pm$ 2.4 & 56.4 $\pm$ 6.2 & 58.8 $\pm$ 7.8 & \textbf{75.2 $\pm$ 2.1} & 51.0 $\pm$ 1.6 & 77.9 $\pm$ 1.3 \\
$d_{\mathrm{WLLB}}^{\scriptscriptstyle{(4)}}$, $f_2$ & \textbf{90.0 $\pm$ 4.2} & 70.2 $\pm$ 1.8 & 58.3 $\pm$ 5.6 & 58.2 $\pm$ 6.9 & 74.5 $\pm$ 2.0 & 50.3 $\pm$ 1.4 & 78.2 $\pm$ 0.8 \\
\bottomrule
WWL & 85.3 $\pm$ 7.3 & 72.9 $\pm$ 3.6 & \textbf{62.2 $\pm$ 6.1} & \textbf{63.0 $\pm$ 7.4} & 72.5 $\pm$ 3.7 & 50.0 $\pm$ 5.3 & 78.2 $\pm$ 0.8\\
WL & 85.5$\pm$ 1.6 & 71.6 $\pm$ 0.6 & 56.6 $\pm$ 2.1 & 56.2 $\pm$ 2.0 & 72.4 $\pm$ 0.7 & 50.9 $\pm$ 0.4 & 78.4 $\pm$ 1.1\\
WL-OA & 86.3 $\pm$ 2.1 & 72.6 $\pm$ 0.7 & 58.4 $\pm$ 2.0 & 54.2 $\pm$ 1.6 & 73.0 $\pm$ 1.1 & 50.2 $\pm$ 1.1 & 78.8 $\pm$ 1.3\\
\bottomrule
\end{tabular}}
\end{center}
\vskip -0.1in
\end{table*}

\subsubsection{Time comparison} \label{subsec:time comparison}
We compare the runtimes of $\dWLk$ and $d_{\mathrm{WLLB}}^{\scriptscriptstyle{(k)}}$ for $k = 1, 2$. For our runtime comparisons, we use LMMCs induced by Erd\"os-Renyi graphs of sizes varying from 5 nodes to 100 nodes (with the degree label function and $q=0.6$). Note that while the runtime for $d_{\mathrm{WLLB}}^{\scriptscriptstyle{(k)}}$ does not change much between $k = 1$ and $k = 2$, the $\dWLk$ distance shows a significant increase in the time needed to compute distance between two graphs from $k = 1$ to $k = 2$.

\begin{figure}
    \centering
    \includegraphics[scale=0.5]{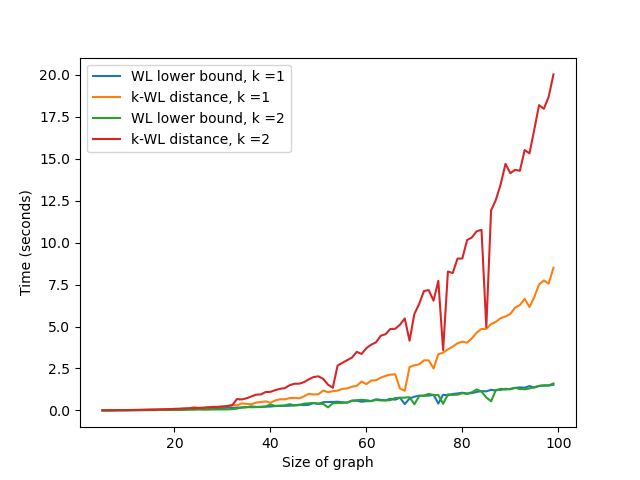}
    \caption{Comparison of runtime of WWL distance against $\dWLk$ and its lower bound $d_\mathrm{WLLB}^{\scriptscriptstyle{(k)}}$. }
    \label{fig:time comparison}
\end{figure}

\section{Proofs}
\subsection{Proofs from \Cref{sec:pre}}
\subsubsection{Proof of \Cref{prop:iso of graphs}}
The ``only if" part is obvious. To prove the ``if" part, we assume $(\mX_q(G_1),\ell_{G_1})$ is isomorphic to $(\mX_q(G_2),\ell_{G_2})$. Then, there exists a bijective map $\psi:V_{G_1}\rightarrow V_{G_2}$ such that $\psi_\# m_v^{G_1,q} = m_{\psi(v)}^{G_2,q},\, \psi_\#\mu_{G_1}=\mu_{G_2}$ and $\ell_{G_1}(v)=\ell_{G_2}(\psi(v))$ for all $v\in V_{G_1}$. Now, by the definition of $m_\bullet^{G_1,q}$ and $m_\bullet^{G_2,q}$ (cf. \Cref{def:markov chain on graphs}), one can easily check that
$$\deg_{G_1}(v)=0\Leftrightarrow m_v^{G_1,q}(v)=m_{\psi(v)}^{G_2,q}(\psi(v))=1\Leftrightarrow\deg_{G_2}(\psi(v))=0.$$

So, consider the case when $m_v^{G_1,q}(v)<1$. This implies $\deg_{G_1}(v)>0$ and $\deg_{G_2}(\psi(v))>0$. In this case, again by the definition of $m_\bullet^{G_1,q}$ and $m_\bullet^{G_2,q}$, one can show that
$$v,v'\in V_{G_1}\text{ are adjacent}\Leftrightarrow m_v^{G_1,q}(v')=m_{\psi(v)}^{G_2,q}(\psi(v'))>0\Leftrightarrow \psi(v),\psi(v')\in V_{G_2}\text{ are adjacent.}$$

Hence, $G_1$ and $G_2$ are isomorphic as we required.
\subsection{Proofs from \Cref{sec:WL distance}}

\subsubsection{Proof of the claim in \Cref{ex:k=0 and 1}}
By \Cref{lm:push forward of coupling dW} we have that
\begin{align*}
    \dWL^{\scriptscriptstyle{(1)}}((\mX,\ell_X),(\mY,\ell_Y))&=\dW\!\lc\mathfrak{L}_1((\mX,\ell_X)),\mathfrak{L}_1((\mY,\ell_Y))\rc\\
    &=\dW\!\lc\lc\WLh{1}{(\mX,\ell_X)}\rc_\#\mu_X,\lc\WLh{1}{(\mY,\ell_Y)}\rc_\#\mu_Y\rc\\
    &=\inf_{\gamma\in \cpl(\mu_X,\mu_Y)} \int\limits_{X\times Y}\dW\!\lc \WLh{1}{(\mX,\ell_X)}(x),\WLh{1}{(\mY,\ell_Y)}(y)\rc \,\gamma(dx\times dy)\\
    &=\inf_{\gamma\in \cpl(\mu_X,\mu_Y)} \int\limits_{X\times Y}\dW\!\lc (\ell_X)_\#m_x^X,(\ell_Y)_\#m_y^Y\rc \,\gamma(dx\times dy).
\end{align*}
\subsubsection{Proof of \Cref{coro:hierarchy dwlk}}
This proposition follows directly from \Cref{lm:hierarchy of k-fold couplings} and \Cref{thm:dwl= dwk}.

\subsubsection{Proof of \Cref{prop:dwl is a pseudometric}}
It is obvious that when $(\mX,\ell_X)$ is isomorphic to $(\mY,\ell_Y)$, $\dWLk\!\lc(\mX,\ell_X),(\mY,\ell_Y)\rc =0$ for all $k\in\mathbb{N}$ and thus $\dWL\!\lc(\mX,\ell_X),(\mY,\ell_Y)\rc =0$. It follows directly from Equation \eqref{eq:definition of dwlk} that $\dWLk$ satisfies the triangle inequality. Hence, $\dWL:=\sup_{k\geq 1}\dWLk$ also satisfies the triangle inequality.

\subsubsection{Proof of \Cref{lm:two labels are the same}}
We first assume that the WL test cannot distinguish $(G_1,\ell_{G_1})$ and $(G_2,\ell_{G_2})$, i.e., $L_k((G_1,\ell_{G_1}))=L_k((G_2,\ell_{G_2}))$ for all $k=0,1,\ldots$. 
We then prove that $L_k((G_1,\ell_{G_1}^g))=L_k((G_2,\ell_{G_2}^g))$ for all $k=0,1,\ldots$. The assumption $L_k((G_1,\ell_{G_1}))=L_k((G_2,\ell_{G_2}))$ for all $k=0,1,\ldots$ immediately implies that $|V_{G_1}|=|V_{G_2}|$. Then, it suffices to show that for any $v_1\in V_{G_1}$ and $v_2\in V_{G_2}$
\begin{equation}\label{eq:lkg=lkg}
\WLhor{k+1}{(G_1,\ell_{G_1})}(v_1)=\WLhor{k+1}{(G_2,\ell_{G_2})}(v_2)\Longrightarrow \WLhor{k}{(G_1,\ell_{G_1}^g)}(v_1)=\WLhor{k}{(G_2,\ell_{G_2}^g)}(v_2),\,\,\forall k=0,1,\ldots.
\end{equation}

We prove Equation \eqref{eq:lkg=lkg} by induction on $k$. When $k=0$, for any $v_1\in V_{G_1}$ and $v_2\in V_{G_2}$, if $\ell^1_{(G_1,\ell_{G_1})}(v_1)=\ell^1_{(G_2,\ell_{G_2})}(v_2)$, then
$$(\ell_{G_1}(v_1),\mbl \ell_{G_1}(v),\,v\in N_{G_1}(v_1)\mbr)=(\ell_{G_2}(v_2),\mbl \ell_{G_2}(v),\,v\in N_{G_2}(v_2)\mbr).$$
It follows that $\ell_{G_1}(v_1)=\ell_{G_2}(v_2)$ and $\deg_{G_1}(v_1)=\deg_{G_2}(v_2)$. Then, by injectivity of $g$, one has that $\ell_{G_1}^g(v_1)=\ell_{G_2}^g(v_2)$. 

Now, we assume that Equation \eqref{eq:lkg=lkg} holds for some $k\geq 0$. For the case of $k+1$, note that $\WLhor{k+2}{(G_1,\ell_{G_1})}(v_1)=\WLhor{k+2}{(G_2,\ell_{G_2})}(v_2)$ implies that
\[\lc \WLhor{k+1}{(G_1,\ell_{G_1})}(v_1),\mbl \WLhor{k+1}{(G_1,\ell_{G_1})}(v),\,v\in N_{G_1}(v_1)\mbr\rc =\lc \WLhor{k+1}{(G_2,\ell_{G_2})}(v_2),\mbl \WLhor{k+1}{(G_2,\ell_{G_2})}(v),\,v\in N_{G_2}(v_2)\mbr\rc.\]
Hence, $\WLhor{k+1}{(G_1,\ell_{G_1})}(v_1)=\WLhor{k+1}{(G_2,\ell_{G_2})}(v_2)$ and there exists a bijection $\psi:N_{G_1}(v_1)\rightarrow N_{G_2}(v_2)$ such that $\WLhor{k+1}{(G_1,\ell_{G_1})}(v)=\WLhor{k+1}{(G_2,\ell_{G_2})}(\psi(v))$ for any $v\in N_{G_1}(v_1)$. By the induction assumption, we then have that $\WLhor{k}{(G_1,\ell_{G_1}^g)}(v_1)=\WLhor{k}{(G_2,\ell_{G_2}^g)}(v_2)$ and  $\WLhor{k}{(G_1,\ell_{G_1}^g)}(v)=\WLhor{k}{(G_2,\ell_{G_2}^g)}(\psi(v))$ for any $v\in N_{G_1}(v_1)$. This implies that
\[\lc \WLhor{k}{(G_1,\ell_{G_1}^g)}(v_1),\mbl \WLhor{k}{(G_1,\ell_{G_1}^g)}(v),\,v\in N_{G_1}(v_1)\mbr\rc =\lc \WLhor{k}{(G_2,\ell_{G_2}^g)}(v_2),\mbl \WLhor{k}{(G_2,\ell^g_{G_2})}(v),\,v\in N_{G_2}(v_2)\mbr\rc \]
and thus $\WLhor{k+1}{(G_1,\ell_{G_1}^g)}(v_1)=\WLhor{k+1}{(G_2,\ell_{G_2}^g)}(v_2)$. Therefore, $L_k\!\lc(G_1,\ell_{G_1}^g)\rc=L_k\!\lc(G_2,\ell_{G_2}^g)\rc$ for all $k=0,1,\ldots$ and thus the WL test cannot distinguish  $(G_1,\ell_{G_1}^g)$ and $(G_2,\ell_{G_2}^g)$.

Conversely, we assume that the WL test cannot distinguish $(G_1,\ell_{G_1}^g)$ and $(G_2,\ell_{G_2}^g)$, i.e., $L_k\!\lc(G_1,\ell_{G_1}^g)\rc=L_k\!\lc(G_2,\ell_{G_2}^g)\rc$ for all $k=0,1,\ldots$. 
We then prove that $L_k\!\lc(G_1,\ell_{G_1})\rc=L_k\!\lc(G_2,\ell_{G_2})\rc$ for all $k=0,1,\ldots$. The proof is similar to the one for the other direction. First, the assumption $L_k\!\lc(G_1,\ell_{G_1}^g)\rc=L_k\!\lc(G_2,\ell_{G_2}^g)\rc$ for all $k=0,1,\ldots$ implies that $|V_{G_1}|=|V_{G_2}|$. Then, it suffices to show that for any $v_1\in V_{G_1}$ and $v_2\in V_{G_2}$
\begin{equation}\label{eq:converse lkg=lkg}
\WLhor{k}{(G_1,\ell_{G_1}^g)}(v_1)=\WLhor{k}{(G_2,\ell_{G_2}^g)}(v_2)\Longrightarrow \WLhor{k}{(G_1,\ell_{G_1})}(v_1)=\WLhor{k}{(G_2,\ell_{G_2})}(v_2),\,\,\forall k=0,1,\ldots.
\end{equation}

We prove Equation \eqref{eq:converse lkg=lkg} by induction on $k$. When $k=0$, for any $v_1\in V_{G_1}$ and $v_2\in V_{G_2}$, if $\ell_{G_1}^g(v_1)=\ell_{G_2}^g(v_2)$, then by injectivity of $g$, we have that $\ell_{G_1}(v_1)=\ell_{G_2}(v_2)$. 

Now, we assume that Equation \eqref{eq:converse lkg=lkg} holds for some $k\geq 0$. For the case of $k+1$, note that $\WLhor{k+1}{(G_1,\ell_{G_1}^g)}(v_1)=\WLhor{k+1}{(G_2,\ell_{G_2}^g)}(v_2)$ implies that
\[\lc \WLhor{k}{(G_1,\ell^g_{G_1})}(v_1),\mbl \WLhor{k}{(G_1,\ell^g_{G_1})}(v),\,v\in N_{G_1}(v_1)\mbr\rc =\lc \WLhor{k}{(G_2,\ell^g_{G_2})}(v_2),\mbl \WLhor{k}{(G_2,\ell^g_{G_2})}(v),\,v\in N_{G_2}(v_2)\mbr\rc.\]

By the induction assumption, it is easy to see that
\[\lc \WLhor{k}{(G_1,\ell_{G_1})}(v_1),\mbl \WLhor{k}{(G_1,\ell_{G_1})}(v),\,v\in N_{G_1}(v_1)\mbr\rc =\lc \WLhor{k}{(G_2,\ell_{G_2})}(v_2),\mbl \WLhor{k}{(G_2,\ell_{G_2})}(v),\,v\in N_{G_2}(v_2)\mbr\rc \]
and thus $\WLhor{k+1}{(G_1,\ell_{G_1})}(v_1)=\WLhor{k+1}{(G_2,\ell_{G_2})}(v_2)$. Therefore, $L_k\!\lc(G_1,\ell_{G_1})\rc=L_k\!\lc(G_2,\ell_{G_2})\rc$ for all $k=0,1,\ldots$ and thus the WL test cannot distinguish  $(G_1,\ell_{G_1})$ and $(G_2,\ell_{G_2})$.

\subsubsection{Proof of \Cref{prop:WL vs dwl}}\label{app:WL vs dwl}
By \Cref{lm:two labels are the same}, we only need to prove that the WL test cannot distinguish $(G_1,\ell_{G_1}^g)$ and $(G_2,\ell_{G_2}^g)$ iff $\dWL\!\lc \lc\mX_q(G_1),\ell_G^g\rc,\lc\mX_q(G_2),\ell_{G_2}^g\rc\rc = 0$. For this purpose, we need to introduce some new notions.

For any metric space $Z$, we let $\poww(Z)$ denote the collection of all finite multisets of $Z$ (including the empty set). We inductively define a family of sets $Z^k$ as follows:
\begin{enumerate}
    \item $Z^1\coloneqq Z\times \poww(Z)$;
    \item for $k\geq 1$, $Z^{k+1}\coloneqq Z^k\times \poww(Z^k)$.
\end{enumerate}

Then, we inductively define a family of maps $\varphi^k_q:Z^k\rightarrow\prob^{\circ k}(Z)$ as follows:
\begin{enumerate}
    \item define $\varphi^1_q:Z^1\rightarrow\prob(Z)$ by 
    \[(z,A)\in Z\times \poww(Z)\mapsto\begin{cases}q\delta_z+\frac{1-q}{|A|}\sum_{z'\in A}\delta_{z'},&A\neq\emptyset\\
    \delta_z,&A=\emptyset\end{cases} ;\]
    \item for $k\geq 1$, define $\varphi^{k+1}_q:Z^{k+1}\rightarrow\probm^{\circ (k+1)}(Z)$ by 
    \[(z,A)\in Z^k\times \poww(Z^k)\mapsto\begin{cases} q\delta_{\varphi_q^k(z)}+\frac{1-q}{|A|}\sum_{z'\in A}\delta_{\varphi_q^k(z')},&A\neq\emptyset\\
    \delta_{\varphi_q^k(z)},&A=\emptyset\end{cases}.\]
\end{enumerate}

\begin{lemma}\label{lm:mathfrak l = varphi l}
For any labeled graph $(G,\ell_G:V_G\rightarrow Z)$ and any $q\in[0,1]$, one has that for any $k\in\mathbb{N}$
\begin{equation}\label{eq:l as composition of l}
\WLh{k}{(\mX_q(G),\ell_{G})}=\varphi_q^{k}\circ\WLhor{k}{(G,\ell_G)}:V_G\rightarrow\prob^{\circ k}(Z).
\end{equation}
\end{lemma}
\begin{proof}[Proof of \Cref{lm:mathfrak l = varphi l}]
We prove by induction on $k$.

When $k=1$, for any $v\in V$, if $N_G(v)\neq\emptyset$ we have that
\begin{align*}
    \WLh{1}{(\mX_q(G),\ell_G)}(v)&=(\ell_G)_\#m_v^{G,q}=q\,\delta_{\ell_G(v)}+\frac{1-q}{\mathrm{deg}(v)}\sum_{v'\in N_G(v)}\delta_{\ell_G(v')}\\
    &=\varphi_q^1\lc (\ell_G(v),\mbl \ell_G(v'):v'\in N_G(v)\mbr)\rc \\
    &=\varphi_q^1(\WLhor{1}{(G,\ell_G)}(v)).
\end{align*}
If $N_G(v)=\emptyset$ then we have that
\begin{align*}
    \WLh{1}{(\mX_q(G),\ell_G)}(v)&=(\ell_G)_\#m_v^{G,q}=\delta_{\ell_G(v)}=\varphi_q^1\lc (\ell_G(v),\emptyset)\rc =\varphi_q^1(\WLhor{1}{(G,\ell_G)}(v)).
\end{align*}

Now, we assume that \Cref{eq:l as composition of l} holds for some $k\geq 1$. Then, for $k+1$ and for any $v\in V$, if $N_G(v)\neq\emptyset$, we have that
\begin{align*}
    \WLh{k+1}{(\mX_q(G),\ell_G)}(v)&=\lc \WLh{k}{(\mX_q(G),\ell_G)}\rc_\#m_v^{G,q}=q\,\delta_{\WLh{k}{(\mX_q(G),\ell_G)}(v)}+\frac{1-q}{\mathrm{deg}(v)}\sum_{v'\in N_G(v)}\delta_{\WLh{k}{(\mX_q(G),\ell_G)}(v')}\\
    &=q\,\delta_{\varphi_q^k\lc \WLhor{k}{(G,\ell_G)}(v)\rc }+\frac{1-q}{\mathrm{deg}(v)}\sum_{v'\in N_G(v)}\delta_{\varphi_q^k\lc \WLhor{k}{(G,\ell_G)}(v')\rc }\\
    &=\varphi_q^{k+1}\!\lc \lc \WLhor{k}{(G,\ell_G)}(v),\mbl \WLhor{k}{(G,\ell_G)}(v'):v'\in N_G(v)\mbr\rc \rc \\
    &=\varphi_q^{k+1}\!\lc \WLhor{k+1}{(G,\ell_G)}(v)\rc.
\end{align*}

If $N_G(v)=\emptyset$ then we have that
\begin{align*}
    \WLh{k+1}{(\mX_q(G),\ell_G)}(v)&=\lc \WLh{k}{(\mX_q(G),\ell_G)}\rc_\#m_v^{G,q}=\delta_{\WLh{k}{(\mX_q(G),\ell_G)}(v)}=\delta_{\varphi_q^k\lc \WLhor{k}{(G,\ell_G)}(v)\rc }\\
    &=\varphi_q^{k+1}\!\lc \lc \WLhor{k}{(G,\ell_G)}(v),\emptyset\rc \rc =\varphi_q^{k+1}\!\lc \WLhor{k+1}{(G,\ell_G)}(v)\rc.
\end{align*}
This concludes the proof.
\end{proof}

\begin{lemma}\label{lm:ptwise equivalence}
Fix any $\frac{1}{2}<q<1$ and any labeled graphs $(G_1,\ell_{G_1})$ and $(G_2,\ell_{G_2})$. Assume that the labels satisfy that for any $v_1\in V_{G_1}$ and $v_2\in V_{G_2}$, we have that $\ell_{G_1}(v_1)=\ell_{G_2}(v_2)$ implies $\mathrm{deg}(v_1)=\mathrm{deg}(v_2)$. Then, one has that for any $v_1\in V_{G_1},v_2\in V_{G_2}$,
\begin{equation}\label{eq:l=l}
    \WLhor{k}{(G_1,\ell_{G_1})}(v_1)=\WLhor{k}{(G_2,\ell_{G_2})}(v_2) \,\,\text{ iff }\,\, \WLh{k}{(\mX_q(G_1),\ell_{G_1})}(v_1)=\WLh{k}{(\mX_q(G_2),\ell_{G_2})}(v_2).
\end{equation}
\end{lemma}
\begin{proof}[Proof of \Cref{lm:ptwise equivalence}]
By \Cref{lm:mathfrak l = varphi l}, we have that $\WLhor{k}{(G_1,\ell_{G_1})}(v_1)=\WLhor{k}{(G_2,\ell_{G_2})}(v_2)$ implies that $\WLh{k}{(\mX_q(G_1),\ell_{G_1})}(v_1)=\WLh{k}{(\mX_q(G_2),\ell_{G_2})}(v_2).$ For the other direction, we prove by induction on $k$.

When $k=1$, we first note that $\WLh{1}{(\mX_q(G_1),\ell_{G_1})}(v_1)=q\,\delta_{\ell_{G_1}(v_1)}+\frac{1-q}{\mathrm{deg}(v_1)}\sum_{v\in N_{G_1}(v_1)}\delta_{\ell_{G_1}(v)}$ if $N_{G_1}(v_1)\neq\emptyset$ and $\WLh{1}{(\mX_q(G_1),\ell_{G_1})}(v_1)=\delta_{\ell_{G_1}(v_1)}$ otherwise. Since $\frac{1}{2}<q<1$, $\WLh{1}{(\mX_q(G_1),\ell_{G_1})}(v_1)=\WLh{1}{(\mX_q(G_2),\ell_{G_2})}(v_2)$ implies that $\delta_{\ell_{G_1}(v_1)}=\delta_{\ell_{G_2}(v_2)}$. Hence $\ell_{G_1}(v_1)=\ell_{G_2}(v_2)$ and thus $\deg_{G_1}(v_1)=\deg_{G_2}(v_2)$. This implies that $N_{G_1}(v_1)=\emptyset$ iff $N_{G_2}(v_2)=\emptyset$. If $N_{G_1}(v_1)=\emptyset$, then obviously, we have that
$$\WLhor{1}{(G_1,\ell_{G_1})}(v_1)=\lc \ell_{G_1}(v_1),\emptyset \rc=\lc \ell_{G_2}(v_2),\emptyset \rc=\WLhor{1}{(G_2,\ell_{G_2})}(v_2).$$
If otherwise $N_{G_1}(v_1)\neq\emptyset$, then $\WLh{1}{(\mX_q(G_1),\ell_{G_1})}(v_1)=\WLh{1}{(\mX_q(G_2),\ell_{G_2})}(v_2)$ again implies that 
$$\frac{1-q}{\mathrm{deg}(v_1)}\sum_{v\in N_{G_1}(v_1)}\delta_{\ell_{G_1}(v)}=\frac{1-q}{\mathrm{deg}(v_2)}\sum_{v\in N_{G_2}(v_2)}\delta_{\ell_{G_2}(v)}.$$ 
Hence,  $\sum_{v\in N_{G_1}(v_1)}\delta_{\ell_{G_1}(v)}=\sum_{v\in N_{G_2}(v_2)}\delta_{\ell_{G_2}(v)}$ and thus 
$$\mbl \ell_{G_1}(v):v\in N_{G_1}(v_1)\mbr=\mbl \ell_{G_2}(v):v\in N_{G_2}(v_2)\mbr.$$
Therefore, $\WLhor{1}{(G_1,\ell_{G_1})}(v_1)=\WLhor{1}{(G_2,\ell_{G_2})}(v_2)$.

Now, we assume that  \Cref{eq:l=l} holds for some $k\geq 1$. Note that 
$$\WLh{k+1}{(\mX_q(G_1),\ell_{G_1})}(v_1)=\begin{cases}q\,\delta_{\WLh{k}{(\mX_q(G_1),\ell_{G_1})}(v_1)}+\frac{1-q}{\mathrm{deg}(v_1)}\sum_{v\in N_{G_1}(v_1)}\delta_{\WLh{k}{(\mX_q(G_1),\ell_{G_1})}(v)},&N_{G_1}(v_1)\neq\emptyset\\
\delta_{\WLh{k}{(\mX_q(G_1),\ell_{G_1})}(v_1)},&N_{G_1}(v_1)=\emptyset\end{cases}$$
Then, for $k+1$, the assumptions $\frac{1}{2}<q<1$ and  $\WLh{k+1}{(\mX_q(G_1),\ell_{G_1})}(v_1)=\WLh{k+1}{(\mX_q(G_2),\ell_{G_2})}(v_2)$ imply that $\delta_{\WLh{k}{(\mX_q(G_1),\ell_{G_1})}(v_1)}=\delta_{\WLh{k}{(\mX_q(G_2),\ell_{G_2})}(v_2)}$. Hence $\WLh{k}{(\mX_q(G_1),\ell_{G_1})}(v_1)=\WLh{k}{(\mX_q(G_2),\ell_{G_2})}(v_2)$. By the induction assumption we have that $\WLhor{k}{(G_1,\ell_{G_1})}(v_1)=\WLhor{k}{(G_2,\ell_{G_2})}(v_2)$. It is not hard to see that then $\WLhor{1}{(G_1,\ell_{G_1})}(v_1)=\WLhor{1}{(G_2,\ell_{G_2})}(v_2)$ and thus $\deg_{G_1}(v_1)=\deg_{G_2}(v_2)$. Then, similarly as in the case $k=1$, we have two situations. If $N_{G_1}(v_1)=\emptyset$, then we have that
$$\WLhor{k+1}{(G_1,\ell_{G_1})}(v_1)=\lc\WLhor{k}{(G_1,\ell_{G_1})}(v_1),\emptyset\rc=\lc\WLhor{k}{(G_2,\ell_{G_2})}(v_2),\emptyset\rc=\WLhor{k+1}{(G_2,\ell_{G_2})}(v_2).$$
If otherwise $N_{G_1}(v_1)\neq\emptyset$, then $\WLh{k+1}{(\mX_q(G_1),\ell_{G_1})}(v_1)=\WLh{k+1}{(\mX_q(G_2),\ell_{G_2})}(v_2)$ again implies that
 $$\frac{1-q}{\mathrm{deg}(v_1)}\sum_{v\in N_{G_1}(v_1)}\delta_{\WLh{k}{(\mX_q(G_1),\ell_{G_1})}(v)}=\frac{1-q}{\mathrm{deg}(v_2)}\sum_{v\in N_{G_2}(v_2)}\delta_{\WLh{k}{(\mX_q(G_2),\ell_{G_2})}(v)}.$$ 
 Hence, $\sum_{v\in N_{G_1}(v_1)}\delta_{\WLh{k}{(\mX_q(G_1),\ell_{G_1})}(v)}=\sum_{v\in N_{G_2}(v_2)}\delta_{\WLh{k}{(\mX_q(G_2),\ell_{G_2})}(v)}$ and thus 
$$\mbl \WLh{k}{(\mX_q(G_1),\ell_{G_1})}(v):v\in N_{G_1}(v_1)\mbr=\mbl \WLh{k}{(\mX_q(G_2),\ell_{G_2})}(v):v\in N_{G_2}(v_2)\mbr.$$ 
Then, by the induction assumption again, we have that $$\mbl \WLhor{k}{(G_1,\ell_{G_1})}(v):v\in N_{G_1}(v_1)\mbr=\mbl \WLhor{k}{(G_2,\ell_{G_2})}(v):v\in N_{G_2}(v_2)\mbr.$$
Therefore, $\WLhor{k+1}{(G_1,\ell_{G_1})}(v_1)=\WLhor{k+1}{(G_2,\ell_{G_2})}(v_2)$. This concludes the proof.
\end{proof}

Now, we are ready to prove that 
\[\mbox{the WL test cannot distinguish}\,\, (G_1,\ell_{G_1}^g)\,\, \mbox{and}\,\, (G_2,\ell_{G_2}^g) \,\,\mbox{iff}\,\, \dWL\!\lc \lc\mX_q(G_1),\ell_G^g\rc,\lc\mX_q(G_2),\ell_{G_2}^g\rc\rc = 0.\] It suffices to show that for any $k=0,1,\ldots$, 
\begin{equation}\label{eq:Lk vs lk}
    L_k\!\lc(G_1,\ell_1^g)\rc= L_k\!\lc(G_2,\ell_2^g)\rc\text{ iff }\!\lc\WLh{k}{(\mX_q(G_1),\ell_{G_1}^g)}\rc_\#\mu_{G_1}= \lc\WLh{k}{(\mX_q(G_2),\ell_{G_2}^g)}\rc_\#\mu_{G_2}
\end{equation}

Fix any $k=0,\ldots$. We first assume that $L_k\!\lc(G_1,\ell_1^g)\rc= L_k\!\lc(G_2,\ell_2^g)\rc$. Then, it is obvious that $|V_{G_1}|=|V_{G_2}|$ and moreover, there exists a bijection $\psi:V_{G_1}\rightarrow V_{G_2}$ such that $\WLhor{k}{(G_1,\ell_{G_1}^g)}(v)=\WLhor{k}{(G_2,\ell_{G_2}^g)}(\psi(v))$ for any $v\in V_{G_1}$. This implies the following facts:
\begin{enumerate}
\item By injectivity of $g$, $\deg_{G_1}(v)=\deg_{G_2}(\psi(v))$ for any $v\in V_{G_1}$. Hence, $\overline{\deg}_{G_1}(v)=\overline{\deg}_{G_2}(\psi(v))$ for any $v\in V_{G_1}$.
    
    \item By \Cref{lm:mathfrak l = varphi l}, for any $v\in V_{G_1}$ we have that
    \[\WLh{k}{(\mX_q(G_1),\ell_{G_1}^g)}(v)=\varphi_q^k\circ\WLhor{k}{(G_1,\ell_{G_1}^g)}(v)=\varphi_q^k\circ\WLhor{k}{(G_2,\ell_{G_2}^g)}(\psi(v))=\WLh{k}{(\mX_q(G_2),\ell_{G_2}^g)}(\psi(v)).\]
\end{enumerate}

Then,
\begin{align*}
    \lc \WLh{k}{(\mX_q(G_1),\ell_{G_1}^g)}\rc_\#\mu_{G_1}&=\sum_{v\in V_{G_1}}\frac{\overline{\deg}_{G_1}(v)}{\sum_{v'\in V_{G_1}}\overline{\deg}_{G_1}(v')}\delta_{\WLh{k}{(\mX_q(G_1),\ell_{G_1}^g)}(v)}\\
    &=\sum_{v\in V_{G_1}}\frac{\overline{\deg}_{G_2}(\psi(v))}{\sum_{v'\in V_{G_1}}\overline{\deg}_{G_2}(\psi(v'))}\delta_{\WLh{k}{(\mX_q(G_2),\ell_{G_2}^g)}(\psi(v))}\\
    &=\sum_{v\in V_{G_2}}\frac{\overline{\deg}_{G_2}(v)}{\sum_{v'\in V_{G_2}}\overline{\deg}_{G_2}(v')}\delta_{\WLh{k}{(\mX_q(G_2),\ell_{G_2}^g)}(v)}\\
    &= \lc \WLh{k}{(\mX_q(G_2),\ell_{G_2}^g)}\rc_\#\mu_{G_2}.
\end{align*}

Conversely, we assume that $\lc \WLh{k}{(\mX_q(G_1),\ell_{G_1}^g)}\rc_\#\mu_{G_1}= \lc \WLh{k}{(\mX_q(G_2),\ell_{G_2}^g)}\rc_\#\mu_{G_2}$. Then,
\begin{equation}\label{eq:sum = sum}
    \sum_{v\in V_{G_1}}\frac{\overline{\deg}_{G_1}(v)}{\sum_{v'\in V_{G_1}}\overline{\deg}_{G_1}(v')}\delta_{\WLh{k}{(\mX_q(G_1),\ell_{G_1}^g)}(v)}=\sum_{v\in V_{G_2}}\frac{\overline{\deg}_{G_2}(v)}{\sum_{v'\in V_{G_2}}\overline{\deg}_{G_2}(v')}\delta_{\WLh{k}{(\mX_q(G_2),\ell_{G_2}^g)}(v)}.
\end{equation}
Then, for any $v_1\in V_{G_1}$, there exists $v_2\in V_{G_2}$ such that $\WLh{k}{(\mX_q(G_1),\ell_{G_1}^g)}(v_1)=\WLh{k}{(\mX_q(G_2),\ell_{G_2}^g)}(v_2)$. If $k=0$, then 
$$\ell_{G_1}^{g}(v_1)=\WLh{0}{(\mX_q(G_1),\ell_{G_1}^g)}(v_1)=\WLh{0}{(\mX_q(G_2),\ell_{G_2}^g)}(v_2)=\ell_{G_2}^{g}(v_2).$$
Otherwise, we assume that $k>0$. Since $\frac{1}{2}<q<1$, we have that $\WLh{k-1}{(\mX_q(G_1),\ell_{G_1}^g)}(v_1)=\WLh{k-1}{(\mX_q(G_2),\ell_{G_2}^g)}(v_2)$. Inductively, we still obtain that
$$\ell_{G_1}^{g}(v_1)=\WLh{0}{(\mX_q(G_1),\ell_{G_1}^g)}(v_1)=\WLh{0}{(\mX_q(G_2),\ell_{G_2}^g)}(v_2)=\ell_{G_2}^{g}(v_2).$$
Hence, by injectivity of $g$, we have that $|V_{G_1}|=|V_{G_2}|$, $\deg_{G_1}(v_1)=\deg_{G_2}(v_2)$ and $\overline{\deg}_{G_1}(v_1)=\overline{\deg}_{G_2}(v_2)$. Then, it is easy to see from Equation \eqref{eq:sum = sum} that 
$$\left|\left\{v\in V_{G_1}:\WLh{k}{(\mX_q(G_1),\ell_{G_1}^g)}(v)=\WLh{k}{(\mX_q(G_1),\ell_{G_1}^g)}(v_1)\right\}\right|=\left|\left\{v\in V_{G_2}:\WLh{k}{(\mX_q(G_2),\ell_{G_2}^g)}(v)=\WLh{k}{(\mX_q(G_2),\ell_{G_2}^g)}(v_2)\right\}\right|.$$
 
It is obvious that $\ell^g_{G_i}$ for $i=1,2$ satisfy the condition in \Cref{lm:ptwise equivalence}. Then, by \Cref{lm:ptwise equivalence}, we have that $\WLhor{k}{(G_1,\ell_{G_1})}(v_1)=\WLhor{k}{(G_2,\ell_{G_2})}(v_2)$ and that
$$\left|\left\{v\in V_{G_1}:\WLhor{k}{(G_1,\ell_{G_1})}(v)=\WLhor{k}{(G_1,\ell_{G_1})}(v_1)\right\}\right|=\left|\left\{v\in V_{G_2}:\WLhor{k}{(G_2,\ell_{G_2})}(v)=\WLhor{k}{(G_2,\ell_{G_2})}(v_2)\right\}\right|.$$
Therefore, $L_k\!\lc(G_1,\ell_{G_1}^g)\rc=L_k\!\lc(G_2,\ell_{G_2}^g)\rc$.

\subsubsection{Proof of \Cref{cor:dWL convergence}}

It turns out that one only needs finite steps to determine whether the WL test can distinguish two labeled graphs \cite{krebs2015universal}. More precisely:
\begin{proposition}\label{prop:wl convergence}
For any labeled graphs $(G_1, \ell_{G_1})$ and $(G_2, \ell_{G_2})$, $L_k\!\lc(G_1,\ell_{G_1})\rc= L_k\!\lc(G_2,\ell_{G_2})\rc$ holds for all $k=0,\ldots, \left|V_{G_1}\right|+\left|V_{G_2}\right|$ if and only if $L_k\!\lc(G_1,\ell_{G_1})\rc= L_k\!\lc(G_2,\ell_{G_2})\rc$ holds for all $k\geq0$.
\end{proposition}
Hence, this corollary is a direct consequence of \Cref{prop:WL vs dwl} and \Cref{prop:wl convergence}.
\subsection{Proofs from \Cref{sec:NN}}


\subsubsection{Proof of \Cref{prop:zero set of NN}}\label{sec:proof zero set}
We need the following lemma:
\begin{lemma}\label{lm:lipschitz inherit}
For any $C$-Lipschitz function $\varphi:\R^i\rightarrow\R^j$, we have that the map $q_\varphi:\prob(\R^i)\rightarrow\R^j$ is $C$-Lipschitz.
\end{lemma}
\begin{proof}[Proof of \Cref{lm:lipschitz inherit}]
For any $\alpha,\beta\in\prob(\R^i)$, pick any $\gamma\in\cpl(\alpha,\beta)$. Then, we have that
\begin{align*}
    \left|q_\varphi(\alpha)-q_\varphi(\beta)\right|&=\left|\int_{\R^i}\varphi(x)\alpha(dx)-\int_{\R^i}\varphi(x)\beta(dx)\right|\\
    &=\left|\int_{\R^i\times\R^i}(\varphi(x)-\varphi(y))\,\gamma(dx\times dy)\right|\\
    &\leq \int_{\R^i\times\R^i}\left|\varphi(x)-\varphi(y)\right|\,\gamma(dx\times dy)\\
    &\leq C\cdot\int_{\R^i\times\R^i}\left|x-y\right|\,\gamma(dx\times dy).
\end{align*}
Since $\gamma\in\cpl(\alpha,\beta)$ is arbitrary, we have that
\[\left|q_\varphi(\alpha)-q_\varphi(\beta)\right|\leq C\cdot\dW(\alpha,\beta).\]
Hence $q_\varphi$ is $C$-Lipschitz.
\end{proof}

Now, we start to prove item 1. We introduce some notation. Given a $\LMMCNN{k}$ $h\coloneqq\psi\circ S_{\varphi_{k+1}}\circ F_{\varphi_k}\circ \cdots \circ F_{\varphi_1}$ and any $(\mX,\ell_X)\in\mathcal{M}^{L}(Z)$, we let
\begin{equation}\label{eq:notation ell}
    \lc \mX,\ell_X^{\scriptscriptstyle{(\varphi,i)}}\rc \coloneqq F_{\varphi_i}\circ\cdots\circ F_{\varphi_1}((\mX,\ell_X))
\end{equation}

Now, we assume that for $i=1,\ldots,k$, $\varphi_i$ is a $C_i$-Lipschitz map for some $C_i>0$. Then, by \Cref{lm:lipschitz inherit}, we have that $q_{\varphi_i}$ is also a $C_i$-Lipschitz map for $i=1,\ldots,k$.

Then, we prove that
\begin{equation}\label{eq:lip ineq}
    \dW\!\lc \lc \ell_X^{\scriptscriptstyle{(\varphi,k)}}\rc_\#\mu_X,\lc \ell_Y^{\scriptscriptstyle{(\varphi,k)}}\rc_\#\mu_Y\rc \leq
\Pi_{i=1}^k C_i\cdot\dWLk((\mathcal{X},\ell_X),(\mathcal{Y},\ell_Y)).
\end{equation}

Given Equation \eqref{eq:lip ineq}, if $\dWLk\!\lc(\mX,\ell_X),(\mY,\ell_Y)\rc =0$, then $\dW\!\lc \lc \ell_X^{\scriptscriptstyle{(\varphi,k)}}\rc_\#\mu_X,\lc \ell_Y^{\scriptscriptstyle{(\varphi,k)}}\rc_\#\mu_Y\rc =0$ and thus $\lc \ell_X^{\scriptscriptstyle{(\varphi,k)}}\rc_\#\mu_X=\lc \ell_Y^{\scriptscriptstyle{(\varphi,k)}}\rc_\#\mu_Y$. Hence, for any continuous $\psi$ and $\varphi_{k+1}$, the $\LMMCNN{}$ $h=\psi\circ S_{\varphi_{k+1}}\circ F_{\varphi_k}\circ \cdots \circ F_{\varphi_1}$ satisfies that
\begin{align*}
    h((\mX,\ell_X))=\psi\lc q_{\varphi_{k+1}}\!\lc \lc \ell_X^{\scriptscriptstyle{(\varphi,k)}}\rc_\#\mu_X\rc \rc =\psi\lc q_{\varphi_{k+1}}\!\lc \lc \ell_Y^{\scriptscriptstyle{(\varphi,k)}}\rc_\#\mu_Y\rc \rc =h((\mY,\ell_Y)).
\end{align*}

To prove Equation \eqref{eq:lip ineq}, it suffices to prove that for any $x\in X$ and $y\in Y$ (cf. \Cref{lm:push forward of coupling dW}), 
\begin{equation*}\label{eq:nn<dwl-k}
   \left\|\ell_X^{\scriptscriptstyle{(\varphi,k)}}(x)-\ell_Y^{\scriptscriptstyle{(\varphi,k)}}(y)\right\|\leq \Pi_{i=1}^k C_i\cdot \dW\!\lc \WLh{k}{(\mX,\ell_X)}(x),\WLh{k}{(\mY,\ell_Y)}(y)\rc.
\end{equation*}
We prove the above inequality by proving the following inequality inductively on $j=1,\ldots,k$:
\begin{equation}\label{eq:nn<dwl}
   \left\|\ell_X^{\scriptscriptstyle{(\varphi,j)}}(x)-\ell_Y^{\scriptscriptstyle{(\varphi,j)}}(y)\right\|\leq \Pi_{i=1}^j C_i\cdot \dW\!\lc \WLh{j}{(\mX,\ell_X)}(x),\WLh{j}{(\mY,\ell_Y)}(y)\rc.
\end{equation}

When $j=0$, we have that $\ell_X^{\scriptscriptstyle{(\varphi,0)}}=\ell_X=\WLh{0}{(\mX,\ell_X)}$ and $\ell_Y^{\scriptscriptstyle{(\varphi,0)}}=\ell_Y=\WLh{0}{(\mY,\ell_Y)}$. Therefore, \Cref{eq:nn<dwl} obviously holds (we let $\Pi_{i=1}^0 C_i\coloneqq 1$). We now assume that \Cref{eq:nn<dwl} holds for some $j\geq 0$. For $j+1$, we have that
\begin{align*}
    &\Pi_{i=1}^{j+1} C_i\cdot\dW\!\lc \WLh{j+1}{(\mX,\ell_X)}(x),\WLh{j+1}{(\mY,\ell_Y)}(y)\rc\\ &=\Pi_{i=1}^{j+1} C_i\cdot\dW\!\lc \lc \WLh{j}{(\mX,\ell_X)}\rc_\#m_x^X,\lc \WLh{j}{(\mY,\ell_Y)}\rc_\#m_y^Y\rc \\
    &=C_{j+1}\cdot\inf_{\gamma\in\mathcal{C}(m_x^X,m_y^Y)}\int\limits_{X\times Y}\Pi_{i=1}^j C_i\cdot\dW\!\lc \WLh{j}{(\mX,\ell_X)}(x'),\WLh{j}{(\mY,\ell_Y)}(y')\rc \gamma(dx'\times dy')\\
    &\geq C_{j+1} \inf_{\gamma\in\mathcal{C}(m_x^X,m_y^Y)}\int\limits_{X\times Y}\left\|\ell_X^{\scriptscriptstyle{(\varphi,j)}}(x')-\ell_Y^{\scriptscriptstyle{(\varphi,j)}}(y')\right\|\gamma(dx'\times dy')\\
    &=C_{j+1}\cdot\dW\!\lc \lc \ell_X^{\scriptscriptstyle{(\varphi,j)}}\rc_\#m_x^X,\lc \ell_Y^{\scriptscriptstyle{(\varphi,j)}}\rc_\#m_y^Y\rc \\
    &\geq \left\|q_{\varphi_{j+1}}\!\lc \lc \ell_X^{\scriptscriptstyle{(\varphi,j)}}\rc_\#m_x^X\rc -q_{\varphi_{j+1}}\!\lc \lc \ell_Y^{\scriptscriptstyle{(\varphi,j)}}\rc_\#m_y^Y\rc \right\|\\
    &=\left\|\ell_X^{\scriptscriptstyle{(\varphi,j+1)}}(x)-\ell_Y^{\scriptscriptstyle{(\varphi,j+1)}}(y)\right\|.
\end{align*}

Next, we prove item 2. The proof is based on the following basic result:
\begin{lemma}\label{lm:int ineq}
For any $d\in\mathbb{N}$ and any $\alpha,\beta\in\prob(\R^d)$, if $\alpha\neq \beta$, then there exists a Lipschitz function $\varphi:\R^d\rightarrow\R$ such that
\[\int_{\R^d}\varphi(x)\alpha(dx)\neq \int_{\R^d}\varphi(x)\beta(dx).\]
\end{lemma}
\begin{proof}[Proof of \Cref{lm:int ineq}]
By Kantorovich duality (see for example Remark 6.5 in \cite{villani2009optimal}), 
\[\dW(\alpha,\beta)=\sup\left\{\left|\int_{\R^d} \varphi(x) \alpha(dx)-\int_{\R^d} \varphi(x) \beta(dx)\right|:\varphi:{\R^d}\rightarrow\R\text{ is 1-Lipschitz}.\right\}\]
Since $\alpha\neq \beta$, we have that $\dW(\alpha,\beta)>0$, and thus there exists a 1-Lipschitz $\varphi:{\R^d}\rightarrow\R$ such that $\int_{\R^d}\varphi(x)\alpha(dx)\neq \int_{\R^d}\varphi(x)\beta(dx).$
\end{proof}

Now, given any $(\mathcal{X},\ell_X)$ and $(\mathcal{Y},\ell_Y)$ such that $\dWLk\!\lc(\mX,\ell_X),(\mY,\ell_Y)\rc >0$, we have that 
$$\dW\!\lc \lc \WLh{k}{(\mX,\ell_X)}\rc_\#\mu_X,\lc \WLh{k}{(\mY,\ell_Y)}\rc_\#\mu_Y\rc=\dWLk\!\lc(\mX,\ell_X),(\mY,\ell_Y)\rc  >0.$$

Then, we prove that for each $i=1,\ldots,k$, there exists a Lipschitz map $\varphi_i:\R^{d_{i-1}}\rightarrow \R^{d_i}$ for suitable dimensions $d_{i-1}$ and $d_i$ such that
\begin{equation}\label{eq:separate points}
    \forall x\in X,\,y\in Y,\,\, \ell_X^{\scriptscriptstyle{(\varphi,i)}}(x)=\ell_Y^{\scriptscriptstyle{(\varphi,i)}}(y)\mbox{ iff }\WLh{i}{(\mX,\ell_X)}(x)=\WLh{i}{(\mY,\ell_Y)}(y).
\end{equation}
Given \Cref{eq:separate points}, it is obvious that $\lc \WLh{k}{(\mX,\ell_X)}\rc_\#\mu_X\neq \lc \WLh{k}{(\mY,\ell_Y)}\rc_\#\mu_Y$ implies that $\lc \ell_X^{\scriptscriptstyle{(\varphi,k)}}\rc_\#\mu_X\neq \lc \ell_Y^{\scriptscriptstyle{(\varphi,k)}}\rc_\#\mu_Y$. Then, by \Cref{lm:int ineq} there exists a Lipschitz map $\varphi_{k+1}:\R^{d_k}\rightarrow\R$ such that 
\[\int_{\R^d}\varphi_{k+1}(t)\lc \ell_X^{\scriptscriptstyle{(\varphi,k)}}\rc_\#\mu_X(dt)\neq \int_{\R^d}\varphi_{k+1}(t)\lc \ell_Y^{\scriptscriptstyle{(\varphi,k)}}\rc_\#\mu_Y(dt).\]
Then, if we let $\psi:\R\rightarrow\R$ be the identity map, we have that
\begin{align*}
    \psi\circ S_{\varphi_{k+1}}\circ F_{\varphi_k}\circ \cdots F_{\varphi_1}(\mX)&=\int_{\R^d}\varphi_{k+1}(t)\lc \ell_X^{\scriptscriptstyle{(\varphi,k)}}\rc_\#\mu_X(dt)\\
    &\neq \int_{\R^d}\varphi_{k+1}(t)\lc \ell_Y^{\scriptscriptstyle{(\varphi,k)}}\rc_\#\mu_Y(dt)\\
    &=\psi\circ S_{\varphi_{k+1}}\circ F_{\varphi_k}\circ \cdots F_{\varphi_1}(\mY).
\end{align*}

To conclude the proof, we prove \Cref{eq:separate points} by induction on $i=1,\ldots,k$. When $i=1$, let
$$A_1\coloneqq\left\{(x,y)\in X\times Y:\,(\ell_X)_\#m_x^X\neq(\ell_Y)_\#m_y^Y\right\}.$$
Since $X$ and $Y$ are finite, $A_1$ is a finite set. We enumerate elements in $A_1$ and write $A_1=\{(x_1,y_1),\ldots,(x_{d_1},y_{d_1})\}$ By \Cref{lm:int ineq}, for each $j=1,\ldots,d_1$, there exists a Lipschitz map $\varphi^j_1:\R^d\rightarrow\R$ such that
\[\int_{\R^d}\varphi^j_1(t)(\ell_X)_\#m_{x_j}^X(dt)\neq \int_{\R^d}\varphi^j_1(t)(\ell_Y)_\#m_{y_j}^Y(dt).\]
We then let $\varphi_1\coloneqq\lc \varphi_1^1,\varphi_1^2,\ldots,\varphi_1^{d_1}\rc :\R^d\rightarrow\R^{d_1}$. $\varphi_1$ is obviously a Lipschitz map and it satisfies that
\[ \forall x\in X,\,y\in Y,\,\, (\ell_X)_\#m_x^X=(\ell_Y)_\#m_y^Y\mbox{ iff }\int_{\R^d}\varphi_1(t)(\ell_X)_\#m_x^X(dt)= \int_{\R^d}\varphi_1(t)(\ell_Y)_\#m_y^Y(dt).\]
Equivalent speaking, 
\[ \forall x\in X,\,y\in Y,\,\, \WLh{1}{(\mX,\ell_X)}(x)=\WLh{1}{(\mY,\ell_Y)}(y)\mbox{ iff }\ell_X^{\scriptscriptstyle{(\varphi,1)}}(x)=\ell_Y^{\scriptscriptstyle{(\varphi,1)}}(y).\]

Now, we assume that \Cref{eq:separate points} holds for some $i\geq 1$. For $i+1$, we let
$$A_{i+1}\coloneqq\left\{(x,y)\in X\times Y:\,\lc \ell_X^{\scriptscriptstyle{(\varphi,i)}}\rc_\#m_x^X\neq\lc \ell_Y^{\scriptscriptstyle{(\varphi,i)}}\rc_\#m_y^Y\right\}.$$
Since $X$ and $Y$ are finite, $A_{i+1}$ is a finite set. We enumerate elements in $A_{i+1}$ and write $A_{i+1}=\{(x_1,y_1),\ldots,(x_{d_{i+1}},y_{d_{i+1}})\}$ By \Cref{lm:int ineq}, for each $j=1,\ldots,d_{i+1}$, there exists a Lipschitz map $\varphi^j_{i+1}:\R^{d_i}\rightarrow\R$ such that
\[\int_{\R^{d_i}}\varphi^j_{i+1}(t)\lc \ell_X^{\scriptscriptstyle{(\varphi,i)}}\rc_\#m_{x_j}^X(dt)\neq \int_{\R^{d_i}}\varphi^j_{i+1}(t)\lc \ell_Y^{\scriptscriptstyle{(\varphi,i)}}\rc_\#m_{y_j}^Y(dt).\]
We then let $\varphi_{i+1}\coloneqq\lc \varphi_{i+1}^1,\varphi_{i+1}^2,\ldots,\varphi_{i+1}^{d_{i+1}}\rc :\R^{d_i}\rightarrow\R^{d_{i+1}}$. $\varphi_{i+1}$ is obviously a Lipschitz map and it satisfies that $\forall x\in X$ and $y\in Y,$
\[ \lc \ell_X^{\scriptscriptstyle{(\varphi,i)}}\rc_\#m_x^X=\lc \ell_Y^{\scriptscriptstyle{(\varphi,i)}}\rc_\#m_y^Y\mbox{ iff }\int_{\R^{d_i}}\varphi_{i+1}(t)\lc \ell_X^{\scriptscriptstyle{(\varphi,i)}}\rc_\#m_x^X(dt)= \int_{\R^{d_i}}\varphi_{i+1}(t)\lc \ell_Y^{\scriptscriptstyle{(\varphi,i)}}\rc_\#m_y^Y(dt).\]
Equivalent speaking, 
\[ \forall x\in X,\,y\in Y,\,\, \lc \ell_X^{\scriptscriptstyle{(\varphi,i)}}\rc_\#m_x^X=\lc \ell_Y^{\scriptscriptstyle{(\varphi,i)}}\rc_\#m_y^Y\mbox{ iff }\ell_X^{\scriptscriptstyle{(\varphi,i+1)}}(x)=\ell_Y^{\scriptscriptstyle{(\varphi,i+1)}}(y).\]
By the induction assumption, $\forall x\in X,\,y\in Y$ we have that
\[ \ell_X^{\scriptscriptstyle{(\varphi,i)}}(x)=\ell_Y^{\scriptscriptstyle{(\varphi,i)}}(y)\mbox{ iff }\WLh{i}{(\mX,\ell_X)}(x)=\WLh{i}{(\mY,\ell_Y)}(y).\]
This implies that
\[ \lc \ell_X^{\scriptscriptstyle{(\varphi,i)}}\rc_\#m_x^X=\lc \ell_Y^{\scriptscriptstyle{(\varphi,i)}}\rc_\#m_y^Y\mbox{ iff }\!\lc\WLh{i}{(\mX,\ell_X)}\rc_\#m_x^X=\lc\WLh{i}{(\mY,\ell_Y)}\rc_\#m_y^Y\mbox{ iff }\WLh{i+1}{(\mX,\ell_X)}(x)=\WLh{i+1}{(\mY,\ell_Y)}(y).\]
Therefore, 
\[ \ell_X^{\scriptscriptstyle{(\varphi,i+1)}}(x)=\ell_Y^{\scriptscriptstyle{(\varphi,i+1)}}(y)\mbox{ iff }\WLh{i+1}{(\mX,\ell_X)}(x)=\WLh{i+1}{(\mY,\ell_Y)}(y)\]
and we thus conclude the proof.


\subsubsection{Proof of \Cref{thm:universality}}\label{app:proof universal}

The proof of the theorem is based on the following Stone-Weierstrass theorem.

\begin{lemma}[Stone-Weierstrass]\label{lm:SW}
Let $X$ be a compact space. Let $\mathcal{F}\subseteq C(X,\mathbb{R})$ be a subalgebra containing the constant function 1. If moreover $\mathcal{F}$ separates points, then $\mathcal{F}$ is dense in $C(X,\mathbb{R})$.
\end{lemma}

\paragraph{$\mathcal{N\!N}_k(\R^d)$ contains 1.} Given any choice of $\varphi_i$s, we let $\psi:\R^{d_{k+1}}\rightarrow\mathbb{R}$ be the constant map $1$. Then, the corresponding function $h=\psi\circ S_{\varphi_{k+1}}\circ F_{\varphi_k}\circ \cdots F_{\varphi_1}\equiv 1\in \mathcal{N\!N}_k(\R^d)$. 

\paragraph{$\mathcal{N\!N}_k(\R^d)$ separates points.} This follows from item 2 in \Cref{prop:zero set of NN}.

\paragraph{$\mathcal{N\!N}_k(\R^d)$ is a subalgebra.}
By \Cref{eq:lip ineq}, we have that 
\[\mathcal{N\!N}_k(\R^d)\coloneqq\{\psi\circ S_{\varphi_{k+1}}\circ F_{\varphi_k}\circ \cdots \circ F_{\varphi_1}:\forall \psi, \varphi_i,i=1,\ldots,k+1\}\subseteq C(\mathcal{K},\mathbb{R}).\]  Next, we show that $\mathcal{N\!N}_k(\R^d)$ is, in fact, a subalgebra of $C(\mathcal{K},\mathbb{R})$. Given any constant $c$ and function $h=\psi\circ S_{\varphi_{k+1}}\circ F_{\varphi_k}\circ \cdots \circ F_{\varphi_1}$, we have that
\[c\cdot h=c\cdot \psi\circ S_{\varphi_{k+1}}\circ F_{\varphi_k}\circ \cdots F_{\varphi_1}=(c\cdot \psi)\circ S_{\varphi_{k+1}}\circ F_{\varphi_k}\circ \cdots F_{\varphi_1}\in\mathcal{N\!N}_k(\R^d).\]
Then, we show that the sum and the product of any $h_1=\psi\circ S_{\varphi_{k+1}}\circ F_{\varphi_k}\circ \cdots F_{\varphi_1}$ and $h_2=\tilde{\psi}\circ S_{\tilde{\varphi}_{k+1}}\circ F_{\tilde{\varphi}_k}\circ \cdots F_{\tilde{\varphi}_1}$ belongs to $\mathcal{N\!N}_k(\R^d)$. We define 
$$\Phi_1\coloneqq (\varphi_1,\tilde{\varphi}_1):\R^d\rightarrow\R^{d_1}\times \R^{\tilde{d}_1},$$
and for each $2\leq i\leq k+1$, we define 
$$\Phi_i=\varphi_i\times\tilde{\varphi}_i :\R^{d_{i-1}}\times \R^{\tilde{d}_{i-1}}\rightarrow \R^{d_{i}}\times \R^{\tilde{d}_{i}}.$$

Obviously, we have that for each $i=1,\ldots,k+1$, $\Phi_i$ inherits the Lipschitz property from $\varphi_i$ and $\tilde{\varphi}_i$:

\begin{claim}\label{claim:product map lip}
For each $i=1,\ldots,k$, assume that $\varphi_i$ is $C_i$-Lipschitz and $\tilde{\varphi}_i$ is $\tilde{C}_i$-Lipschitz, then $\Phi_i$ is $\max(C_i,\tilde{C}_i)$-Lipschitz.
\end{claim}

We let $P:\R^{d_{k+1}}\times \R^{\tilde{d}_{k+1}}\rightarrow \R^{d_{k+1}}$ and $\tilde{P}:\R^{d_{k+1}}\times \R^{\tilde{d}_{k+1}}\rightarrow \R^{\tilde{d}_{k+1}}$ denote projection maps. Then, we can rewrite $h_1$ and $h_2$ as follows
\begin{claim}\label{claim:projection}
$h_1=\psi\circ S_{\varphi_{k+1}}\circ F_{\varphi_k}\circ \cdots F_{\varphi_1}=\psi\circ P\circ S_{\Phi_{k+1}}\circ F_{\Phi_k}\circ \cdots F_{\Phi_1}$ and $h_2=\tilde{\psi}\circ S_{\tilde{\varphi}_{k+1}}\circ F_{\tilde{\varphi}_k}\circ \cdots F_{\tilde{\varphi}_1}=\tilde{\psi}\circ \tilde{P}\circ S_{\Phi_{k+1}}\circ F_{\Phi_k}\circ \cdots F_{\Phi_1}$.
\end{claim}

\begin{proof}[Proof of Claim \ref{claim:projection}]
Recall notation from \Cref{eq:notation ell}. Then, we first prove inductively on $i=1,\ldots,k$ that for any $\mathcal{X}\in\mathcal{K}$ 
\begin{equation}\label{eq:Q=(q,q)}
    \ell_X^{\scriptscriptstyle{(\Phi,i)}}(x)=\lc \ell_X^{\scriptscriptstyle{(\varphi,i)}}(x),\ell_X^{\scriptscriptstyle{(\tilde{\varphi},i)}}(x)\rc ,\quad\forall x\in X.
\end{equation}
When $i=1$,
$$\ell_X^{\scriptscriptstyle{(\Phi,1)}}(x)=q_{\Phi_1}((\ell_X)_\#m_x^X))=(q_{\varphi_1}((\ell_X)_\#m_x^X),q_{\tilde{\varphi}_1}((\ell_X)_\#m_x^X))=\lc \ell_X^{\scriptscriptstyle{(\varphi,1)}}(x),\ell_X^{\scriptscriptstyle{(\tilde{\varphi},1)}}(x)\rc.$$
Now, we assume that Equation \eqref{eq:Q=(q,q)} holds for some $i\geq 1$. Then, for $i+1$, we have that
\begin{align*}
    \ell_X^{\scriptscriptstyle{(\Phi,i+1)}}(x)&=q_{\Phi_{i+1}}\!\lc\lc \ell_X^{\scriptscriptstyle{(\Phi,i)}}\rc_\#m_x^X\rc=q_{\Phi_{i+1}}\!\lc\lc \lc \ell_X^{\scriptscriptstyle{(\varphi,i)}},\ell_X^{\scriptscriptstyle{(\tilde{\varphi},i)}}\rc \rc_\#m_x^X\rc\\
    &=q_{\Phi_{i+1}}\!\lc \lc \ell_X^{\scriptscriptstyle{(\varphi,i)}}\rc_\#m_x^X\otimes \lc \ell_X^{\scriptscriptstyle{(\tilde{\varphi},i)}}\rc_\#m_x^X\rc \\
    &=\lc q_{\varphi_{i+1}}\!\lc \lc \ell_X^{\scriptscriptstyle{(\varphi,i)}}\rc_\#m_x^X\rc ,q_{\tilde{\varphi}_{i+1}}\!\lc \lc \ell_X^{\scriptscriptstyle{(\tilde{\varphi},i)}}\rc_\#m_x^X\rc \rc \\
    &=\lc \ell_X^{\scriptscriptstyle{(\varphi,i+1)}}(x),\ell_X^{\scriptscriptstyle{(\tilde{\varphi},i+1)}}(x)\rc ,
\end{align*}
which concludes the proof of Equation \eqref{eq:Q=(q,q)}.

Similarly,
\begin{align*}
     S_{\Phi_{k+1}}\circ F_{\Phi_k}\circ \cdots F_{\Phi_1}((\mathcal{X},\ell_X))&= q_{\Phi_{k+1}}\!\lc\lc \ell_X^{\scriptscriptstyle{(\Phi,k)}}\rc_\#\mu_X\rc=q_{\Phi_{k+1}}\!\lc\lc \lc \ell_X^{\scriptscriptstyle{(\varphi,k)}},\ell_X^{\scriptscriptstyle{(\tilde{\varphi},k)}}\rc \rc_\#\mu_X\rc
     \\
    &=q_{\Phi_{k+1}}\!\lc \lc \ell_X^{\scriptscriptstyle{(\varphi,k)}}\rc_\#\mu_X\otimes \lc \ell_X^{\scriptscriptstyle{(\tilde{\varphi},k)}}\rc_\#\mu_X\rc \\
    &=\lc q_{\varphi_{k+1}}\!\lc \lc \ell_X^{\scriptscriptstyle{(\varphi,k)}}\rc_\#\mu_X\rc ,q_{\tilde{\varphi}_{k+1}}\!\lc \lc \ell_X^{\scriptscriptstyle{(\tilde{\varphi},k)}}\rc_\#\mu_X\rc \rc \\
    &=\lc S_{\varphi_{k+1}}\circ F_{\varphi_k}\circ \cdots F_{\varphi_1}((\mathcal{X},\ell_X)),S_{\tilde{\varphi}_{k+1}}\circ F_{\tilde{\varphi}_k}\circ \cdots F_{\tilde{\varphi}_1}((\mathcal{X},\ell_X))\rc.
\end{align*}
Therefore, $\psi\circ S_{\varphi_{k+1}}\circ F_{\varphi_k}\circ \cdots F_{\varphi_1}=\psi\circ P_\#\circ S_{\Phi_{k+1}}\circ F_{\Phi_k}\circ \cdots F_{\Phi_1}$ and similarly, $\tilde{\psi}\circ S_{\tilde{\varphi}_{k+1}}\circ F_{\tilde{\varphi}_k}\circ \cdots F_{\tilde{\varphi}_1}=\tilde{\psi}\circ \tilde{P}_\#\circ S_{\Phi_{k+1}}\circ F_{\Phi_k}\circ \cdots F_{\Phi_1}$.
\end{proof}

Given these claims, we then have that 
\[\psi\circ S_{\varphi_{k+1}}\circ F_{\varphi_k}\circ \cdots F_{\varphi_1}+\tilde{\psi}\circ S_{\tilde{\varphi}_{k+1}}\circ F_{\tilde{\varphi}_k}\circ \cdots F_{\tilde{\varphi}_1}=(\psi\circ P+\tilde{\psi}\circ \tilde{P})\circ S_{\Phi_{k+1}}\circ F_{\Phi_k}\circ \cdots F_{\Phi_1}\in\mathcal{N\!N}_k(\R^d)\]
and
\[\psi\circ S_{\varphi_{k+1}}\circ F_{\varphi_k}\circ \cdots F_{\varphi_1}\times\tilde{\psi}\circ S_{\tilde{\varphi}_{k+1}}\circ F_{\tilde{\varphi}_k}\circ \cdots F_{\tilde{\varphi}_1}=(\psi\circ P\times\tilde{\psi}\circ \tilde{P})\circ S_{\Phi_{k+1}}\circ F_{\Phi_k}\circ \cdots F_{\Phi_1}\in\mathcal{N\!N}_k(\R^d).\]

\subsection{Proofs from \Cref{sec:dGW}}\label{sec:proof iso}
\subsubsection{Proof of \Cref{prop:kGWmetric}}
The proof is rather lengthy, and we start with some preliminary definitions and lemmas.

\begin{definition}
Suppose two finite metric spaces $X$ and $Y$ are given. We say a sequence of measurable maps $\{(\nu_n)_{\bullet,\bullet}:X\times Y\rightarrow\prob(X\times Y)\}_{n\in\N}$ weakly converges to $\nu_{\bullet,\bullet}:X\times Y\rightarrow\prob(X\times Y)$ if  $\{(\nu_n)_{x,y}\}_{n\in\N}\subseteq\prob(X\times Y)$ weakly converges to $\nu_{x,y}\in\prob(X\times Y)$ for all $(x,y)\in X\times Y$.
\end{definition}

\begin{lemma}\label{lemma:productweakconv}
Suppose two finite metric spaces $X$ and $Y$ are given. If a sequence of probability measures $\{\gamma_n\}_{n\in\N}\subseteq\prob(X\times Y)$ weakly converges to $\gamma\in\prob(X\times Y)$ and a sequence of measurable maps $\{(\nu_n)_{\bullet,\bullet}:X\times Y\rightarrow\prob(X\times Y)\}_{n\in\N}$ weakly converges to $\nu_{\bullet,\bullet}:X\times Y\rightarrow\prob(X\times Y)$, then the sequence $\{(\nu_n)_{\bullet,\bullet}\odot\gamma_n\}_{n\in\N}$ also weakly converges to $\nu_{\bullet,\bullet}\odot\gamma$.
\end{lemma}
\begin{proof}
Fix an arbitrary continuous bounded map $\phi:X\times Y\longrightarrow\R$. Then,

\begin{align*}
    &\left\vert \int_{X\times Y}\phi(x,y)\,(\nu_n)_{\bullet,\bullet}\odot\gamma_n(dx\times dy)-\int_{X\times Y}\phi(x,y)\,\nu_{\bullet,\bullet}\odot\gamma(dx\times dy)\right\vert\\
    &\leq\left\vert \int_{X\times Y}\phi(x,y)\,(\nu_n)_{\bullet,\bullet}\odot\gamma_n(dx\times dy)-\int_{X\times Y}\phi(x,y)\,\nu_{\bullet,\bullet}\odot\gamma_n(dx\times dy)\right\vert\\
    &\quad+\left\vert \int_{X\times Y}\phi(x,y)\,\nu_{\bullet,\bullet}\odot\gamma_n(dx\times dy)-\int_{X\times Y}\phi(x,y)\,\nu_{\bullet,\bullet}\odot\gamma(dx\times dy)\right\vert.
\end{align*}

By the weak convergence of $\{(\nu_n)_{\bullet,\bullet}\}_{n\in\N}$ and by applying the bounded convergence theorem, we have that 
$$\int_{X\times Y}\phi(x,y)\,(\nu_n)_{\bullet,\bullet}\odot\gamma_n(dx\times dy)=\int_{X\times Y}\int_{X\times Y}\phi(x,y)\,(\nu_n)_{x',y'}(dx\times dy)\gamma_n(dx'\times dy')$$
converges to
$$\int_{X\times Y}\phi(x,y)\,\nu_{\bullet,\bullet}\odot\gamma_n(dx\times dy)=\int_{X\times Y}\int_{X\times Y}\phi(x,y)\,\nu_{x',y'}(dx\times dy)\gamma_n(dx'\times dy').$$

Also, since $\{\gamma_n\}_{n\in\N}$ weakly converges to $\gamma$, by finiteness of $X$ and $Y$ we have that
$$\int_{X\times Y}\phi(x,y)\,\nu_{\bullet,\bullet}\odot\gamma_n(dx\times dy)=\int_{X\times Y}\int_{X\times Y}\phi(x,y)\,\nu_{x',y'}(dx\times dy)\gamma_n(dx'\times dy')$$
converges to
$$\int_{X\times Y}\phi(x,y)\,\nu_{\bullet,\bullet}\odot\gamma(dx\times dy)=\int_{X\times Y}\int_{X\times Y}\phi(x,y)\,\nu_{x',y'}(dx\times dy)\gamma(dx'\times dy').$$

Hence, $\left\vert \int_{X\times Y}\phi(x,y)\,(\nu_n)_{\bullet,\bullet}\odot\gamma_n(dx\times dy)-\int_{X\times Y}\phi(x,y)\,\nu_{\bullet,\bullet}\odot\gamma(dx\times dy)\right\vert$ converges to zero as we required. This completes the proof.
\end{proof}

\begin{lemma}\label{lemma:kfoldweakconve}
Suppose two MCMSs $(\mX,d_X)$, $(\mY,d_Y)$, $k\geq 1$, and a sequence of $k$-step couplings $\{(\nu_n^{\scriptscriptstyle{(k)}})_{\bullet,\bullet}\}_{n\in\N}\subseteq\cpl^{\scriptscriptstyle{(k)}}\!\lc m_\bullet^X,m_\bullet^Y\rc$ are given. Then, there is a $k$-step coupling $\nu_{\bullet,\bullet}^{\scriptscriptstyle{(k)}}\in\cpl^{\scriptscriptstyle{(k)}}\!\lc m_\bullet^X,m_\bullet^Y\rc$ to which the sequence $\{(\nu_n^{\scriptscriptstyle{(k)}})_{\bullet,\bullet}\}_{n\in\N}$ converges.
\end{lemma}
\begin{proof}
The proof is by induction. $k=1$ case is obvious since $\cpl(m_x^X,m_y^Y)$ is compact w.r.t. the weak topology (see \cite[p.49]{villani2021topics}) for all $(x,y)\in X\times Y$.

Now, suppose the claim holds up to some $k\geq 1$. Consider $k+1$ case. By the definition, each $(k+1)$-step coupling $(\nu_n^{\scriptscriptstyle{(k+1)}})_{\bullet,\bullet}\in\cpl^{\scriptscriptstyle{(k+1)}}\!\lc m_\bullet^X,m_\bullet^Y\rc$ can be expressed in the following way:

$$(\nu_n^{\scriptscriptstyle{(k+1)}})_{x,y}=\int_{X\times Y}(\nu_n^{\scriptscriptstyle{(k)}})_{x',y'}(\mu_n^{\scriptscriptstyle{(1)}})_{x,y}(dx'\times dy')$$

for all $(x,y)\in X\times Y$ for some $(\nu_n^{\scriptscriptstyle{(k)}})_{\bullet,\bullet}\in\cpl^{\scriptscriptstyle{(k)}}\!\lc m_\bullet^X,m_\bullet^Y\rc$ and $(\mu_n^{\scriptscriptstyle{(1)}})_{\bullet,\bullet}\in\cpl^{\scriptscriptstyle{(1)}}\!\lc m_\bullet^X,m_\bullet^Y\rc$. Then, by the inductive assumption, there are $\nu_{\bullet,\bullet}^{\scriptscriptstyle{(k)}}\in\cpl^{\scriptscriptstyle{(k)}}\!\lc m_\bullet^X,m_\bullet^Y\rc$ and $\mu_{\bullet,\bullet}^{\scriptscriptstyle{(1)}}\in\cpl^{\scriptscriptstyle{(1)}}\!\lc m_\bullet^X,m_\bullet^Y\rc$ such that the sequence $\{(\nu_n^{\scriptscriptstyle{(k)}})_{\bullet,\bullet}\}_{n\in\N}$ weakly converges to $\nu_{\bullet,\bullet}^{\scriptscriptstyle{(k)}}$, and the sequence $\{(\mu_n^{\scriptscriptstyle{(1)}})_{\bullet,\bullet}\}_{n\in\N}$ weakly converges to $\mu_{\bullet,\bullet}^{\scriptscriptstyle{(1)}}$. Then, for each $(x,y)\in X\times Y$,
$$(\nu_n^{\scriptscriptstyle{(k+1)}})_{x,y}=\int_{X\times Y}(\nu_n^{\scriptscriptstyle{(k)}})_{x',y'}(\mu_n^{\scriptscriptstyle{(1)}})_{x,y}(dx'\times dy')=(\nu_n^{\scriptscriptstyle{(k)}})_{\bullet,\bullet}\odot (\mu_n^{\scriptscriptstyle{(1)}})_{x,y}$$
weakly converges to
$$\nu_{x,y}^{\scriptscriptstyle{(k+1)}}=\int_{X\times Y}\nu_{x',y'}^{\scriptscriptstyle{(k)}} \mu_{x,y}^{\scriptscriptstyle{(1)}}(dx'\times dy')=\nu_{\bullet,\bullet}^{\scriptscriptstyle{(k)}}\odot \mu_{x,y}^{\scriptscriptstyle{(1)}}$$
by \Cref{lemma:productweakconv}. This completes the proof.
\end{proof}

\begin{lemma}[{\cite[Lemma 10.3]{memoli2011gromov}}]\label{lemma:Facundolemma}
Let $(Z,d_Z)$ be a compact metric space and $\phi:Z\times Z\rightarrow\R$ be a Lipschitz map w.r.t. the $L^1$ metric on $Z\times Z$:
$$\hat{d}_{Z\times Z}((z_1,z_2),(z_1',z_2'))\coloneqq d_Z(z_1,z_1')+d_Z(z_2,z_2') \mbox{ for all }(z_1,z_2),(z_1',z_2')\in Z\times Z.$$
Also, for each $\gamma\in\prob(Z)$, we define a map $p_{\phi,\gamma}$ in the following way:
\begin{align*}
    p_{\phi,\gamma}:Z&\longrightarrow\R\\
    z&\longmapsto\int_Z \phi(z,z')\,\gamma(dz').
\end{align*}

If a sequence $\{\mu_n\}_{n\in\N}\subseteq\prob(Z)$ weakly converges to $\mu$, then $p_{\phi,\mu_n}$ uniformly converges to $p_{\phi,\mu}$.
\end{lemma}

\begin{corollary}\label{cor:kGWoptimalcplexist}
For any two MCMSs $(\mX,d_X)$, $(\mY,d_Y)$, and $k\geq 1$, there exist a coupling measure $\gamma\in\cpl(\mu_X,\mu_Y)$ and a $k$-step coupling $\nu_{\bullet,\bullet}^{\scriptscriptstyle{(k)}}\in\cpl^{\scriptscriptstyle{(k)}}\!\lc m_\bullet^X,m_\bullet^Y\rc$ such that
$$\dGW^{\scriptscriptstyle{(k)}}\!\lc(\mX,d_X),(\mY,d_Y)\rc=\mathrm{dis}^{\scriptscriptstyle{(k)}}\!\lc\gamma,\nu_{\bullet,\bullet}^{\scriptscriptstyle{(k)}}\rc.$$
\end{corollary}
\begin{proof}
First of all, we define $\phi:(X\times Y)\times (X\times Y)\rightarrow\R$ by sending any $((x,y),(x',y'))\in (X\times Y)\times (X\times Y)$ to $\vert d_X(x,x')-d_Y(y,y')\vert$.

By \Cref{def:MCMsGW}, there are a sequence of coupling measures $\{\gamma_n\}_{n\in\N}\subseteq\cpl(\mu_X,\mu_Y)$ and a sequence of $k$-step couplings $\{(\nu_n^{\scriptscriptstyle{(k)}})_{\bullet,\bullet}\}_{n\in\N}\subseteq\cpl^{\scriptscriptstyle{(k)}}\!\lc m_\bullet^X,m_\bullet^Y\rc$ such that

\begin{align*}
&\mathrm{dis}^{\scriptscriptstyle{(k)}}\!\lc\gamma,(\nu_n^{\scriptscriptstyle{(k)}})_{\bullet,\bullet}\rc\\
    &=\int\limits_{X\times Y}\int\limits_{X\times Y}\int\limits_{X\times Y}\vert d_X(x,x')-d_Y(y,y')\vert\,(\nu_n^{\scriptscriptstyle{(k)}})_{x'',y''}(dx'\times dy')\gamma_n(dx''\times dy'')\gamma_n(dx\times dy)\\
    &=\int\limits_{X\times Y}p_{\phi,(\nu_n^{\scriptscriptstyle{(k)}})_{\bullet,\bullet}\odot\gamma_n}\!(x,y)\,\gamma_n(dx\times dy)\\
    &\leq\dGW^{\scriptscriptstyle{(k)}}\!\lc(\mX,d_X),(\mY,d_Y)\rc+\frac{1}{n}
\end{align*}
for each $n\geq 1$.

Now, since $\cpl(\mu_X,\mu_Y)$ is compact w.r.t. the weak topology (see p.49 of \cite{villani2021topics}), there is a coupling measure $\gamma\in\cpl(\mu_X,\mu_Y)$ such that $\gamma_n\rightarrow\gamma$ weakly. Also, by \Cref{lemma:kfoldweakconve}, there is a $k$-step coupling $\nu_{\bullet,\bullet}^{\scriptscriptstyle{(k)}}\in\cpl^{\scriptscriptstyle{(k)}}\!\lc m_\bullet^X,m_\bullet^Y\rc$ such that $(\nu_n^{\scriptscriptstyle{(k)}})_{\bullet,\bullet}\rightarrow \nu_{\bullet,\bullet}^{\scriptscriptstyle{(k)}}$ weakly.

Now, let
\begin{align*}
    &A_n\coloneqq\int\limits_{X\times Y}p_{\phi,\nu_{\bullet,\bullet}^{\scriptscriptstyle{(k)}}\odot\gamma}\!(x,y)\,\gamma_n(dx\times dy)-\int\limits_{X\times Y}p_{\phi,(\nu_n^{\scriptscriptstyle{(k)}})_{\bullet,\bullet}\odot\gamma_n}\!(x,y)\,\gamma_n(dx\times dy),\\
    &B_n\coloneqq\int\limits_{X\times Y}p_{\phi,\nu_{\bullet,\bullet}^{\scriptscriptstyle{(k)}}\odot\gamma}\!(x,y)\,\gamma(dx\times dy)-\int\limits_{X\times Y}p_{\phi,\nu_{\bullet,\bullet}^{\scriptscriptstyle{(k)}}\odot\gamma}\!(x,y)\,\gamma_n(dx\times dy),\\
    &C_n\coloneqq\int\limits_{X\times Y}p_{\phi,\nu_{\bullet,\bullet}^{\scriptscriptstyle{(k)}}\odot\gamma}\!(x,y)\,\gamma(dx\times dy)-\int\limits_{X\times Y}p_{\phi,(\nu_n^{\scriptscriptstyle{(k)}})_{\bullet,\bullet}\odot\gamma_n}\!(x,y)\,\gamma_n(dx\times dy).\\
\end{align*}
It is easy to see that $C_n=A_n+B_n$ and thus $\vert C_n \vert\leq\vert A_n \vert+\vert B_n \vert$. By \Cref{lemma:productweakconv} and \Cref{lemma:Facundolemma}, $A_n$ converges to zero. Also, $B_n$ converges to zero by the assumption that $X,Y$ are finite and that $\gamma_n$ weakly converges to $\gamma$. Hence, $C_n$ converges to zero. Therefore,

\begin{align*}
    \mathrm{dis}^{\scriptscriptstyle{(k)}}\!\lc\gamma,\nu^{\scriptscriptstyle{(k)}}_{\bullet,\bullet}\rc&=\int\limits_{X\times Y}p_{\phi,\nu_{\bullet,\bullet}^{\scriptscriptstyle{(k)}}\odot\gamma}\gamma(dx\times dy)\\
    &=\lim_{n\rightarrow\infty}\int\limits_{X\times Y}p_{\phi,(\nu_n^{\scriptscriptstyle{(k)}})_{\bullet,\bullet}\odot\gamma_n}\gamma_n(dx\times dy)\\
    &\leq\dGW^{\scriptscriptstyle{(k)}}\!\lc(\mX,d_X),(\mY,d_Y)\rc
\end{align*}
Since we always have that
$\mathrm{dis}^{\scriptscriptstyle{(k)}}\!\lc\gamma,\nu^{\scriptscriptstyle{(k)}}_{\bullet,\bullet}\rc\geq \dGW^{\scriptscriptstyle{(k)}}\!\lc(\mX,d_X),(\mY,d_Y)\rc,$ we conclude that
$$\mathrm{dis}^{\scriptscriptstyle{(k)}}\!\lc\gamma,\nu^{\scriptscriptstyle{(k)}}_{\bullet,\bullet}\rc= \dGW^{\scriptscriptstyle{(k)}}\!\lc(\mX,d_X),(\mY,d_Y)\rc.$$
\end{proof}

\begin{lemma}[Gluing of $k$-step couplings]\label{lemma:kfoldglueing}
Suppose three MCMSs $(\mX,d_X),(\mY,d_Y),(\mZ,d_Z)$, $k\geq 1$, and $k$-step couplings $\mu_{\bullet,\bullet}^{\scriptscriptstyle{(k)}}\in\cpl^{\scriptscriptstyle{(k)}}(m_\bullet^X,m_\bullet^Z)$, $\eta_{\bullet,\bullet}^{\scriptscriptstyle{(k)}}\in\cpl^{\scriptscriptstyle{(k)}}(m_\bullet^Z,m_\bullet^Y)$ are given. Then, there are probability measures $\pi_{\bullet,\bullet,\bullet}^{\scriptscriptstyle{(k)}}:X\times Y\times Z\rightarrow\prob(X\times Y\times Z)$ and $k$-step coupling $\nu_{\bullet,\bullet}^{\scriptscriptstyle{(k)}}\in\cpl^{\scriptscriptstyle{(k)}}\!\lc m_\bullet^X,m_\bullet^Y\rc$ such that $\nu_{x,y}^{\scriptscriptstyle{(k)}}$, $\mu_{x,z}^{\scriptscriptstyle{(k)}}$, and $\eta_{z,y}^{\scriptscriptstyle{(k)}}$ are the marginals of $\pi_{x,y,z}^{\scriptscriptstyle{(k)}}$ for any $(x,y,z)\in X\times Y\times Z$.
\end{lemma}
\begin{proof}
The proof is by induction on $k$. First, consider $k=1$ case. Fix an arbitrary $(x,y,z)\in X\times Y\times Z$. Let $$\pi^{\scriptscriptstyle{(1)}}_{x,y,z}(x',y',z')\coloneqq\begin{cases}\frac{\mu_{x,z}^{\scriptscriptstyle{(1)}}(x',z')\mu_{z,y}^{\scriptscriptstyle{(1)}}(z',y')}{m_z^Z(z')},&m_z^Z(z')>0\\
0,&m_z^Z(z')=0\end{cases}$$
for each $(x',y',z')\in X\times Y\times Z$. Observe that

\begin{align*}
    &\sum_{(x',y',z')\in X\times Y\times Z}\pi^{\scriptscriptstyle{(1)}}_{x,y,z}(x',y',z')\\
    &=\sum_{z'\in Z,m_z^Z(z')>0}\sum_{y'\in Y}\frac{\mu_{z,y}^{\scriptscriptstyle{(1)}}(z',y')}{m_z^Z(z')}\sum_{x'\in X}\mu_{x,z}^{\scriptscriptstyle{(1)}}(x',z')\\
    &=\sum_{z'\in Z,m_z^Z(z')>0}\sum_{y'\in Y}\frac{\mu_{z,y}^{\scriptscriptstyle{(1)}}(z',y')}{m_z^Z(z')}\cdot m_z^Z(z')=\sum_{z'\in Z,m_z^Z(z')>0}\sum_{y'\in Y}\mu_{z,y}^{\scriptscriptstyle{(1)}}(z',y')=1.
\end{align*}

Hence, $\pi^{\scriptscriptstyle{(1)}}_{x,y,z}\in\prob(X\times Y\times Z)$. Now, let $\nu_{x,y}^{\scriptscriptstyle{(1)}}(x',y')\coloneqq\sum_{z'\in Z}\pi^{\scriptscriptstyle{(1)}}_{x,y,z}(x',y',z')$ for each $(x',y')\in X\times Y$. Then, for fixed $x'\in X$,

\begin{align*}
    \sum_{y'\in Y}\nu_{x,y}^{\scriptscriptstyle{(1)}}(x',y')&=\sum_{y'\in Y}\sum_{z'\in Z,m_z^Z(z')>0}\frac{\mu_{x,z}^{\scriptscriptstyle{(1)}}(x',z')\mu_{z,y}^{\scriptscriptstyle{(1)}}(z',y')}{m_z^Z(z')}=\sum_{z'\in Z,m_z^Z(z')>0}\frac{\mu_{x,z}^{\scriptscriptstyle{(1)}}(x',z')}{m_z^Z(z')}\sum_{y'\in Y}\mu_{z,y}^{\scriptscriptstyle{(1)}}(z',y')\\
    &=\sum_{z'\in Z,m_z^Z(z')>0}\frac{\mu_{x,z}^{\scriptscriptstyle{(1)}}(x',z')}{m_z^Z(z')}\cdot m_z^Z(z')=\sum_{z'\in Z,m_z^Z(z')>0}\mu_{x,z}^{\scriptscriptstyle{(1)}}(x',z')=m_x^X(x').
\end{align*}

Similarly, for each fixed $y'\in Y$, one can prove $\sum_{x'\in X}\nu_{x,y}^{\scriptscriptstyle{(1)}}(x',y')=m_y^Y(y')$. Hence, indeed $\nu_{\bullet,\bullet}^{\scriptscriptstyle{(1)}}\in\cpl^{\scriptscriptstyle{(1)}}\!\lc m_\bullet^X,m_\bullet^Y\rc$.

Now, suppose the claim holds up to some $k\geq 1$. We consider $k+1$ case. For a $(k+1)$-step coupling $\mu_{\bullet,\bullet}^{\scriptscriptstyle{(k+1)}}\in\cpl^{\scriptscriptstyle{(k+1)}}(m_\bullet^X,m_\bullet^Z)$, there are $k$-step coupling $\mu_{\bullet,\bullet}^{\scriptscriptstyle{(k)}}\in\cpl^{\scriptscriptstyle{(k)}}(m_\bullet^X,m_\bullet^Z)$ and $1$-step coupling $\mu_{\bullet,\bullet}^{\scriptscriptstyle{(1)}}\in\cpl^{\scriptscriptstyle{(1)}}(m_\bullet^X,m_\bullet^Z)$ such that
$$\mu_{x,z}^{\scriptscriptstyle{(k+1)}}(x',z')=\sum_{(x'',z'')\in X\times Y}\mu_{x'',z''}^{\scriptscriptstyle{(k)}}(x',z')\,\mu_{x,z}^{\scriptscriptstyle{(1)}}(x'',z'')$$
for any $(x,z),(x',z')\in X\times Z$. Similarly, for a $(k+1)$-step coupling $\eta_{\bullet,\bullet}^{\scriptscriptstyle{(k+1)}}\in\cpl^{\scriptscriptstyle{(k+1)}}(m_\bullet^Z,m_\bullet^Y)$, there are $k$-step coupling $\eta_{\bullet,\bullet}^{\scriptscriptstyle{(k)}}\in\cpl^{\scriptscriptstyle{(k)}}(m_\bullet^Z,m_\bullet^Y)$ and $1$-step coupling $\eta_{\bullet,\bullet}^{\scriptscriptstyle{(1)}}\in\cpl^{\scriptscriptstyle{(1)}}(m_\bullet^Z,m_\bullet^Y)$ such that
$$\eta_{z,y}^{\scriptscriptstyle{(k+1)}}(z',y')=\sum_{(z'',y'')\in Z\times Y}\eta_{z'',y''}^{\scriptscriptstyle{(k)}}(z',y')\,\eta_{z,y}^{\scriptscriptstyle{(1)}}(z'',y'')$$
for any $(z,y),(z',y')\in Z\times Y$.

Because of the inductive assumption, we have $\pi_{\bullet,\bullet,\bullet}^{\scriptscriptstyle{(k)}}:X\times Y\times Z\rightarrow\prob(X\times Y\times Z)$ and $\pi_{\bullet,\bullet,\bullet}^{\scriptscriptstyle{(1)}}:X\times Y\times Z\rightarrow\prob(X\times Y\times Z)$ satisfying the claim. Then, let $$\pi_{x,y,z}^{\scriptscriptstyle{(k+1)}}(x',y',z')\coloneqq\sum_{(x'',y'',z'')\in X\times Y\times Z}\pi_{x'',y'',z''}^{\scriptscriptstyle{(k)}}(x',y',z')\,\pi_{x,y,z}^{\scriptscriptstyle{(1)}}(x'',y'',z''),$$
and let $\nu_{x,y}^{\scriptscriptstyle{(k+1)}}(x',y')\coloneqq\sum_{z'\in Z}\pi_{x,y,z}^{\scriptscriptstyle{(k+1)}}(x',y',z')$. Then, we have that
\begin{align*}
    \nu_{x,y}^{\scriptscriptstyle{(k+1)}}(x',y')&=\sum_{z'\in Z}\sum_{(x'',y'',z'')\in X\times Y\times Z}\pi_{x'',y'',z''}^{\scriptscriptstyle{(k)}}(x',y',z')\,\pi_{x,y,z}^{\scriptscriptstyle{(1)}}(x'',y'',z'')\\
    &=\sum_{(x'',y'',z'')\in X\times Y\times Z}\pi_{x,y,z}^{\scriptscriptstyle{(1)}}(x'',y'',z'')\sum_{z'\in Z}\pi_{x'',y'',z''}^{\scriptscriptstyle{(k)}}(x',y',z')\\
    &=\sum_{(x'',y'',z'')\in X\times Y\times Z}\pi_{x,y,z}^{\scriptscriptstyle{(1)}}(x'',y'',z'')\,\nu_{x'',y''}^{\scriptscriptstyle{(k)}}(x',y')\\
    &=\sum_{(x'',y'')\in X\times Y}\nu_{x'',y''}^{\scriptscriptstyle{(k)}}(x',y')\sum_{z''\in Z}\pi_{x,y,z}^{\scriptscriptstyle{(1)}}(x'',y'',z'')\\
    &=\sum_{(x'',y'')\in X\times Y}\nu_{x'',y''}^{\scriptscriptstyle{(k)}}(x',y')\,\nu_{x,y}^{\scriptscriptstyle{(k)}}(x'',y'').
\end{align*}

Since the choice of $(x,y,z)\in x\times y\times z$ is arbitrary, now we have $\nu_{\bullet,\bullet}^{\scriptscriptstyle{(k+1)}}\in\cpl^{\scriptscriptstyle{(k+1)}}\!\lc m_\bullet^X,m_\bullet^Y\rc$ as we required. Hence, this concludes the proof.
\end{proof}

Now we start to prove \Cref{prop:kGWmetric}.

First of all, $\dGW^\mathrm{MCMS}$ is obviously symmetric. 

Next, we prove that $\dGW^\mathrm{MCMS}((\mX,d_X),(\mY,d_Y))=0$ happens if and only if $(\mX,d_X)$ and $(\mY,d_Y)$ are isomorphic. To do this, we first provide a precise definition of MCMS isomorphism.
\begin{definition}\label{def:MCMS isomorphism}
Two MCMSs $(\mX,d_X)$ and $(\mY,d_Y)$ are said to be \emph{isomorphic} if there exists an isometry $\psi:X\rightarrow Y$ such that $\psi_\#\mu_X=\mu_Y$ and
$\psi_\# m_x^X = m_{\psi(x)}^Y$ for all $x\in X.$
\end{definition}
When are $(\mX,d_X)$ and $(\mY,d_Y)$ are isomorphic, without loss of generality, we simply assume that $(\mY,d_Y)=(\mX,d_X)$.

\begin{claim}\label{claim:diagonal}
Let $\Delta_{\mu_X}$ denote the diagonal coupling between $\mu_X$ and itself, i.e.,
\[\Delta_{\mu_X}=\sum_{x\in X}\mu_X(x)\delta_{(x,x)}.\]
Then, for each $k\in\N$, there exists $\nu_{\bullet,\bullet}^{\scriptscriptstyle{(k)}}\in\cpl^{\scriptscriptstyle{(k)}}(m_\bullet^X,m_\bullet^X)$ such that $\Delta_{\mu_X}=\nu_{\bullet,\bullet}^{\scriptscriptstyle{(k)}}\odot\Delta_{\mu_X}$.
\end{claim}

Assume the claim for now. Then, we have that for each $k\in\N$
\begin{align*}
    \dGW^{\scriptscriptstyle{(k)}}((\mX,d_X),(\mY,d_Y))&\leq \mathrm{dis}\!\lc\Delta_{\mu_X},\nu_{\bullet,\bullet}^{\scriptscriptstyle{(k)}}\rc\\
    &=\int\limits_{X\times X}\int\limits_{X\times X}\!\!|d_X(x,x')-d_X(x_1,x_1')|
 \nu_{\bullet,\bullet}^{\scriptscriptstyle{(k)}}\odot\Delta_{\mu_X}(dx'\times dx_1')\Delta_{\mu_X}(dx\times dx_1)\\
 &=\int\limits_{X\times X}|d_X(x,x')-d_X(x,x')|\mu_X(dx)\mu_X(dx')=0.
\end{align*}
Hence, $\dGW^\mathrm{MCMS}((\mX,d_X),(\mY,d_Y))=0$,

\begin{proof}[Proof of \Cref{claim:diagonal}]
We prove inductively on $k\in\N$ that there exists $\nu_{\bullet,\bullet}^{\scriptscriptstyle{(k)}}\in\cpl^{\scriptscriptstyle{(k)}}(\mX,\mX)$ so that $\nu_{x,x}^{\scriptscriptstyle{(k)}}=\Delta_{m_x^{X,\otimes k}}$ is the diagonal coupling between $m_x^{X,\otimes k}$ and itself for each $x\in X$.

For $k=1$, we define $\nu_{\bullet,\bullet}^{\scriptscriptstyle{(1)}}$ as follows:
\[\nu_{x,x'}^{\scriptscriptstyle{(1)}}\coloneqq\begin{cases}m_x^X\otimes m_{x'}^X &x\neq x'\\
\Delta_{m_x^X} & x=x'
\end{cases}. \]
Since $X$ is finite, obviously we have that $\nu_{\bullet,\bullet}^{\scriptscriptstyle{(1)}}\in\cpl^{\scriptscriptstyle{(1)}}(\mX,\mX)$.

Assume that the statement holds for some $k\geq 0$. Now, for $k+1$, by the induction assumption, there exists $\nu_{\bullet,\bullet}^{\scriptscriptstyle{(k)}}\in\cpl^{\scriptscriptstyle{(k)}}(\mX,\mX)$ so that $\nu_{x,x}^{\scriptscriptstyle{(k)}}=\Delta_{m_x^{X,\otimes k}}$. We define $\nu_{\bullet,\bullet}^{\scriptscriptstyle{(k+1)}}\in\cpl^{\scriptscriptstyle{(k+1)}}(\mX,\mX)$ as follows
\[\nu_{x,x'}^{\scriptscriptstyle{(k+1)}}\coloneqq\int\limits_{X\times X}\nu_{x_1,x_1'}^{\scriptscriptstyle{(k)}}\nu_{x,x'}^{\scriptscriptstyle{(1)}}(dx_1\times dx_1'),\,\,\forall x,x'\in X.\]
Now, for any $x\in X$, we have that 
\begin{align*}
    \nu_{x,x}^{\scriptscriptstyle{(k+1)}}&=\int\limits_{X\times X}\nu_{x_1,x_1'}^{\scriptscriptstyle{(k)}}\nu_{x,x}^{\scriptscriptstyle{(1)}}(dx_1\times dx_1')\\
    &=\sum_{x'\in X}m_x^X(x')\sum_{x''\in X}m_{x'}^{X,\otimes k}(x'')\delta_{(x'',x'')}\\
    &=\sum_{x''\in X}\!\lc\sum_{x'\in X}m_{x'}^{X,\otimes k}(x'')m_x^X(x')\rc\delta_{(x'',x'')}\\
    &=\sum_{x''\in X}m_{x}^{X,\otimes (k+1)}(x'')\delta_{(x'',x'')}\\
    &=\Delta_{m_{x}^{X,\otimes (k+1)}}.
\end{align*}
Now, we turn to prove the claim. For each $k\in\N$, let $\nu_{\bullet,\bullet}^{\scriptscriptstyle{(k)}}\in\cpl^{\scriptscriptstyle{(k)}}(\mX,\mX)$ be such that $\nu_{x,x}^{\scriptscriptstyle{(k)}}=\Delta_{m_x^{X,\otimes k}}$ is the diagonal coupling between $m_x^{X,\otimes k}$ and itself for each $x\in X$. Then, 
\begin{align*}
    \int\limits_{X\times X}\nu_{x,x'}^{\scriptscriptstyle{(k)}}\,\Delta_{\mu_X}(dx\times dx')
    &=\sum_{x\in X}\nu_{x,x}^{\scriptscriptstyle{(k)}}\,\mu_X(x)\\
    &=\sum_{x,x'\in X}m_x^{X,\otimes k}(x')\delta_{(x',x')}\,\mu_X(x)\\
    &=\sum_{x'\in X}\!\lc\sum_{x\in X}m_x^{X,\otimes k}(x')\mu_X(x)\rc\delta_{(x',x')}\\
    &=\sum_{x'\in X}\mu_X(x')\delta_{(x',x')}\\
    &=\Delta_{\mu_X}.
\end{align*}
\end{proof}

Now, we assume that $\dGW^\mathrm{MCMS}((\mX,d_X),(\mY,d_Y))=0$ for some MCMSs $(\mX,d_X)$ and $(\mY,d_Y)$. Then, $\dGW^{\scriptscriptstyle{(1)}}((\mX,d_X),(\mY,d_Y))=0$. By \Cref{cor:kGWoptimalcplexist}, there exist optimal $\nu_{\bullet,\bullet}\in\cpl^{\scriptscriptstyle{(1)}}\!\lc m_\bullet^X,m_\bullet^Y\rc$ and $\gamma\in\cpl(\mu_X,\mu_Y)$ such that 
$$\int\limits_{X\times Y}\int\limits_{X\times Y}\int\limits_{X\times Y}|d_X(x,x')-d_Y(y,y')|
 \nu_{x'',y''}(dx'\times dy')\gamma(dx''\times dy'')\gamma(dx\times dy)=0.$$
We let $\gamma'\coloneqq\nu_{\bullet,\bullet}\odot\gamma$. Notice that $\gamma'\in\cpl(\mu_X,\mu_Y)$. Since $X$ and $Y$ are finite, we rewrite the integral above as finite sums:
$$\sum_{(x,y)\in X\times Y}\sum_{(x',y')\in X\times Y}\vert d_X(x,x')-d_Y(y,y')\vert\,\gamma'(x',y')\,\gamma(x,y)=0.$$
By \Cref{claim:isometry exist}, there exists an isometry $\phi:X\rightarrow Y$ such that
$$\{(x,\phi(x)):x\in X\}=\mathrm{supp}(\gamma)=\mathrm{supp}(\gamma').$$
Since $\gamma,\gamma'\in\cpl(\mu_X,\mu_Y)$, this immediately implies that for any Borel subset $A\subseteq X$, one has
$$\mu_X(A)=\gamma(A\times Y)=\gamma(A\times\phi(A))=\gamma(X\times\phi(A))=\mu_Y(\phi(A)).$$
Hence, $\phi_\#\mu_X=\mu_Y$. Moreover, we have that
\[\gamma=\gamma'=\sum_{x\in X}\mu_X(x)\delta_{(x,\phi(x))}.\]
Then, by the definition of $\gamma'$, we have that
\begin{align*}
    \sum_{x\in X}\mu_X(x)\delta_{(x,\phi(x))}=\gamma'&=\sum_{x_1\in X,y_1\in Y}\gamma(x_1,y_1)\nu_{x_1,y_1}\\
    &=\sum_{x\in X,y\in Y}\!\lc\sum_{x_1\in X,y_1\in Y}\gamma(x_1,y_1)\nu_{x_1,y_1}(x,y)\rc\delta_{(x,y)}\\
    &=\sum_{x\in X,y\in Y}\!\lc\sum_{x_1\in X}\gamma(x_1,\phi(x_1))\nu_{x_1,\phi(x_1)}(x,y)\rc\delta_{(x,y)}\\
    &=\sum_{x\in X,y\in Y}\!\lc\sum_{x_1\in X}\mu_X(x_1)\nu_{x_1,\phi(x_1)}(x,y)\rc\delta_{(x,y)}.
\end{align*}
By comparing coefficients for the Dirac delta measures above, one has that
\[\nu_{x_1,\phi(x_1)}(x,y)=\begin{cases}0,&\mbox{if }y\neq\phi(x)\\
\nu_{x_1,\phi(x_1)}(x,y),&\mbox{if }y=\phi(x)\end{cases}\]
for any $x_1\in X$. This means that $\{(x,\phi(x)):x\in X\}\supseteq\mathrm{supp}[\nu_{x_1,\phi(x_1)}].$ Since $\nu_{x_1,\phi(x_1)}\in\cpl\lc m_{x_1}^X,m_{\phi(x_1)}^Y\rc$, for any Borel subset $A\subseteq X$, one has that
$$m_{x_1}^X(A)=\nu_{x_1,\phi(x_1)}(A\times Y)=\nu_{x_1,\phi(x_1)}(A\times\phi(A))=\nu_{x_1,\phi(x_1)}(X\times\phi(A))=m_{\phi(x_1)}^Y(\phi(A)).$$
By an argument similar to the one for proving $\phi_\#\mu_X=\mu_Y$, we have that
\[\phi_\#m_{x_1}^X=m_{\phi(x_1)}^Y,\,\,\forall x_1\in X.\]
Therefore, $(\mX,d_X)$ is isomorphic to $(\mY,d_Y)$.

Finally, we prove that $\dGW^\mathrm{MCMS}$ satisfies the triangle inequality. It suffices to prove that for each $k\in\N$, $\dGW^{\scriptscriptstyle{(k)}}$ satisfies the triangle inequality. Fix arbitrary three MCMSs $(\mX,d_X)$, $(\mY,d_Y)$, and $(\mZ,d_Z)$. Recall the notation $\Gamma_{X,Y}$ which defines a function sending $(x,y,x',y')\in X\times Y\times X\times Y$ to $|d_X(x,x')-d_Y(y,y')|$. Then, for any $x,x'\in X$, $y,y'\in Y$, and $z,z'\in Z$, we obviously have that
$$\Gamma_{X,Y}(x,y,x',y')\leq\Gamma_{X,Z}(x,z,x',z')+\Gamma_{Z,Y}(z,y,z',y').$$

Now, fix arbitrary $\mu_{\bullet,\bullet}^{\scriptscriptstyle{(k)}}\in\cpl^{\scriptscriptstyle{(k)}}(m_\bullet^X,m_\bullet^Z)$, $\gamma_{X,Z}\in\cpl(\mu_X,\mu_Z)$, $\eta_{\bullet,\bullet}^{\scriptscriptstyle{(k)}}\in\cpl^{\scriptscriptstyle{(k)}}(m_\bullet^Z,m_\bullet^Y)$, and $\gamma_{Z,Y}\in\cpl(\mu_Z,\mu_Y)$. Then, by the Gluing Lemma (see \cite[Lemma 7.6]{villani2021topics}), there exists a probability measure $\alpha\in\prob(X\times Y\times Z)$ with marginals $\gamma_{X,Z}$ and $\gamma_{Z,Y}$ on $X\times Z$ and $Z\times Y$, respectively. Let $\gamma_{X,Y}$ be the marginal of $\pi$ on $X\times Y$ which belongs to $\cpl(\mu_X,\mu_Y)$. By \Cref{lemma:kfoldglueing}, there exists $\nu_{\bullet,\bullet}^{\scriptscriptstyle{(k)}}\in\cpl^{\scriptscriptstyle{(k)}}\!\lc m_\bullet^X,m_\bullet^Y\rc$ such that $\nu_{x,y}^{\scriptscriptstyle{(k)}}$, $\mu_{x,z}^{\scriptscriptstyle{(k)}}$, $\eta_{z,y}^{\scriptscriptstyle{(k)}}$ are the marginals of some probability measure $\pi_{x,y,z}^{\scriptscriptstyle{(k)}}\in\prob(X\times Y\times Z)$ for any $x,y,z\in X\times Y\times Z$. Then, because of the triangle inequality for $L^1$-norm,
\begin{align*}
    &\dGW^{\scriptscriptstyle{(k)}}\!\lc(\mX,d_X),(\mY,d_Y)\rc\\ &\leq\int\limits_{X\times Y}\int\limits_{X\times Y}\int\limits_{X\times Y}\Gamma_{X,Y}(x,y,x',y')
 \nu_{x'',y''}^{\scriptscriptstyle{(k)}}(dx'\times dy')\gamma_{X,Y}(dx''\times dy'')\gamma_{X,Y}(dx\times dy)\\ &=\int\limits_{X\times Y\times Z}\int\limits_{X\times Y\times Z}\int\limits_{X\times Y\times Z}\Gamma_{X,Y}(x,y,x',y')
 \pi^{\scriptscriptstyle{(k)}}_{x'',y'',z''}(dx'\times dy'\times dz')\alpha(dx''\times dy''\times dz'')\alpha(dx\times dy\times dz)\\ &\leq\int\limits_{X\times Y\times Z}\int\limits_{X\times Y\times Z}\int\limits_{X\times Y\times Z}\Gamma_{X,Z}(x,z,x',z')
 \pi^{\scriptscriptstyle{(k)}}_{x'',y'',z''}(dx'\times dy'\times dz')\alpha(dx''\times dy''\times dz'')\alpha(dx\times dy\times dz)\\ &+\int\limits_{X\times Y\times Z}\int\limits_{X\times Y\times Z}\int\limits_{X\times Y\times Z}\Gamma_{Z,Y}(z,y,z',y')
 \pi^{\scriptscriptstyle{(k)}}_{x'',y'',z''}(dx'\times dy'\times dz')\alpha(dx''\times dy''\times dz'')\alpha(dx\times dy\times dz)\\ &=\int\limits_{X\times Z}\int\limits_{X\times Z}\int\limits_{X\times Z}\Gamma_{X,Z}(x,z,x',z')
 \mu^{\scriptscriptstyle{(k)}}_{x'',z''}(dx'\times dz')\gamma_{X,Z}(dx''\times dz'')\gamma_{X,Z}(dx\times dz)\\ &+\int\limits_{ Z\times Y}\int\limits_{Z\times Y}\int\limits_{Z\times Y}\Gamma_{Z,Y}(z,y,z',y')
 \eta^{\scriptscriptstyle{(k)}}_{z'',y''}(dx'\times dy'\times dz')\gamma_{Z,Y}( dz''\times dy'')\gamma_{Z,Y}( dz\times dy).
\end{align*}
Since the choice of $\mu_{\bullet,\bullet}^{\scriptscriptstyle{(k)}},\gamma_{X,Z},\eta_{\bullet,\bullet}^{\scriptscriptstyle{(k)}},\gamma_{Z,Y}$ are arbitrary, by taking the infimum one concludes that
$$\dGW^{\scriptscriptstyle{(k)}}\!\lc(\mX,d_X),(\mY,d_Y)\rc \leq\dGW^{\scriptscriptstyle{(k)}}((\mX,d_X),(\mZ,d_Z))+\dGW^{\scriptscriptstyle{(k)}}((\mZ,d_Z),(\mY,d_Y))$$
as we required. Then, we have that $\dGW^\mathrm{MCMS}\coloneqq\sup_{k\geq 0}\dGW^{\scriptscriptstyle{(k)}}$ satisfies the triangle inequality.

\subsubsection{Proof of the claim in \Cref{prop:Haibin}}

The proof is based on the following lemma.

\begin{lemma}\label{lemma:mmskfoldcpl}
Given two MMSs $\mathbf{X}$ and $\mathbf{Y}$, for their corresponding MCMSs $\mathcal{M}(\mathbf{X})$ and $\mathcal{M}(\mathbf{Y})$ we have that
\begin{enumerate}
    
    \item $\cpl^{\scriptscriptstyle{(k+1)}}\!\lc m_\bullet^X,m_\bullet^Y\rc\subseteq\cpl^{\scriptscriptstyle{(k)}}\!\lc m_\bullet^X,m_\bullet^Y\rc$ for all $k\geq 1$.
    
    \item For any coupling measure $\gamma\in\cpl(\mu_X,\mu_Y)$, the constant map $\nu_{\bullet,\bullet}\equiv\gamma$ belongs to $\cpl^{\scriptscriptstyle{(k)}}\!\lc m_\bullet^X,m_\bullet^Y\rc$ for all $k\geq 1$.
\end{enumerate}
\end{lemma}
\begin{proof}
We first prove item 1. Since $m_x^X=\mu_X$ and $m_y^Y=\mu_Y$ for all $x\in X$ and $y\in Y$, observe that $m_x^{X,\otimes k}=\mu_X$ and $m_y^{Y,\otimes k}=\mu_Y$ for all $x\in X$, $y\in Y$, and $k\geq 1$. Hence, $\cpl^{\scriptscriptstyle{(k)}}\!\lc m_\bullet^X,m_\bullet^Y\rc\subseteq\cpl^{\scriptscriptstyle{(1)}}\!\lc m_\bullet^X,m_\bullet^Y\rc$ for all $k\geq 1$ by \Cref{lm:k-fold coupling well defined}.
Now, for any $k\geq 2$, fix an arbitrary $\nu_{\bullet,\bullet}^{\scriptscriptstyle{(k+1)}}\in\cpl^{\scriptscriptstyle{(k+1)}}\!\lc m_\bullet^X,m_\bullet^Y\rc$. Then, by the definition, there are $\nu_{\bullet,\bullet}^{\scriptscriptstyle{(k)}}\in\cpl^{\scriptscriptstyle{(k)}}\!\lc m_\bullet^X,m_\bullet^Y\rc$ and $\nu_{\bullet,\bullet}^{\scriptscriptstyle{(1)}}\in\cpl^{\scriptscriptstyle{(1)}}\!\lc m_\bullet^X,m_\bullet^Y\rc=\cpl(\mu_X,\mu_Y)$ such that

$$\nu_{x,y}^{\scriptscriptstyle{(k+1)}}=\int_{X\times Y}\nu_{x',y'}^{\scriptscriptstyle{(k)}}\,\nu_{x,y}^{\scriptscriptstyle{(1)}}(dx'\times dy')$$

for all $(x,y)\in X\times Y$. Again, by the definition, there are $\nu_{\bullet,\bullet}^{\scriptscriptstyle{(k-1)}}\in\cpl^{\scriptscriptstyle{(k)}}\!\lc m_\bullet^X,m_\bullet^Y\rc$ and $\mu_{\bullet,\bullet}^{\scriptscriptstyle{(1)}}\in\cpl^{\scriptscriptstyle{(1)}}\!\lc m_\bullet^X,m_\bullet^Y\rc=\cpl(\mu_X,\mu_Y)$ such that

$$\nu_{x,y}^{\scriptscriptstyle{(k)}}=\int_{X\times Y}\nu_{x',y'}^{\scriptscriptstyle{(k-1)}}\,\mu_{x,y}^{\scriptscriptstyle{(1)}}(dx'\times dy')$$

for all $(x,y)\in X\times Y$. Therefore,

\begin{align*}
    \nu_{x,y}^{\scriptscriptstyle{(k+1)}}&=\int_{X\times Y}\int_{X\times Y}\nu_{x'',y''}^{\scriptscriptstyle{(k-1)}}\,\mu_{x',y'}^{\scriptscriptstyle{(1)}}(dx''\times dy'')\,\nu_{x,y}^{\scriptscriptstyle{(1)}}(dx'\times dy')\\
    &=\int_{X\times Y}\nu_{x'',y''}^{\scriptscriptstyle{(k-1)}}\,\pi_{x,y}^{\scriptscriptstyle{(1)}}(dx''\times dy'')
\end{align*}

for all $(x,y)\in X\times Y$ where $\pi_{\bullet,\bullet}^{\scriptscriptstyle{(1)}}\coloneqq\int_{X\times Y}\mu_{x',y'}^{\scriptscriptstyle{(1)}}\,\nu_{\bullet,\bullet}^{\scriptscriptstyle{(1)}}(dx'\times dy')\in\cpl^{\scriptscriptstyle{(2)}}\!\lc m_\bullet^X,m_\bullet^Y\rc\subseteq\cpl^{\scriptscriptstyle{(1)}}\!\lc m_\bullet^X,m_\bullet^Y\rc$. Hence, $\nu_{\bullet,\bullet}^{\scriptscriptstyle{(k+1)}}\in\cpl^{\scriptscriptstyle{(k)}}\!\lc m_\bullet^X,m_\bullet^Y\rc$ by the definition. The first item is proved.

Next, we prove the second item. The proof is by induction on $k$. Fix a coupling measure $\gamma\in\cpl(\mu_X,\mu_Y)$ and the constant map $\nu_{\bullet,\bullet}\equiv\gamma$. Obviously, $\nu_{\bullet,\bullet}\in\cpl^{\scriptscriptstyle{(1)}}\!\lc m_\bullet^X,m_\bullet^Y\rc$. Then, we also have that the constant map $\mu_X\otimes\mu_Y\in\cpl^{\scriptscriptstyle{(1)}}\!\lc m_\bullet^X,m_\bullet^Y\rc$. Now, suppose the claim holds up to some $k\geq 1$. Consider $k+1$ case. Observe that
\begin{align*}
    \gamma=\int_{X\times Y}\gamma\,\mu_X\otimes\mu_Y(dx'\times dy')
\end{align*}
where $\gamma\in\cpl^{\scriptscriptstyle{(k)}}\!\lc m_\bullet^X,m_\bullet^Y\rc$ by the inductive assumption and the constant map $\mu_X\otimes\mu_Y\in\cpl^{\scriptscriptstyle{(1)}}\!\lc m_\bullet^X,m_\bullet^Y\rc$. Hence, $\gamma\in\cpl^{\scriptscriptstyle{(k+1)}}\!\lc m_\bullet^X,m_\bullet^Y\rc$ by the definition. This completes the proof.
\end{proof}

Now, fix arbitrary couplings $\gamma,\gamma'\in\cpl(\mu_X,\mu_Y)$ and consider the constant map $\nu_{\bullet,\bullet}\equiv\gamma'$. Then, by the second item of \Cref{lemma:mmskfoldcpl}, $\nu_{\bullet,\bullet}\in\cpl^{\scriptscriptstyle{(k)}}\!\lc m_\bullet^X,m_\bullet^Y\rc$. Hence,

\begin{align*}
    &\dGW^{\scriptscriptstyle{(k)}}\!\lc\mathcal{M}(\mathbf{X}),\mathcal{M}(\mathbf{Y})\rc\\
    &\leq\int\limits_{X\times Y}\int\limits_{X\times Y}\int\limits_{X\times Y}\vert d_X(x,x')-d_Y(y,y')\vert\,\nu_{x'',y''}(dx'\times dy')\gamma(dx''\times dy'')\gamma(dx\times dy)\\
    &=\int\limits_{X\times Y}\int\limits_{X\times Y}\int\limits_{X\times Y}\vert d_X(x,x')-d_Y(y,y')\vert\,\gamma'(dx'\times dy')\gamma(dx''\times dy'')\gamma(dx\times dy)\\
    &=\int\limits_{X\times Y}\int\limits_{X\times Y}\vert d_X(x,x')-d_Y(y,y')\vert\,\gamma'(dx'\times dy')\gamma(dx\times dy).
\end{align*}

Since the choice of $\gamma,\gamma'$ are arbitrary, one concludes that $\dGW^{\scriptscriptstyle{(k)}}\!\lc(\mathcal{M}(\mathbf{X}),\mathcal{M}(\mathbf{Y})\rc\leq \dGW^{\mathrm{bi}}(\mathbf{X},\mathbf{Y})$.

For the reverse direction, choose arbitrary $\gamma\in\cpl(\mu_X,\mu_Y)$ and a $k$-step coupling $\nu_{\bullet,\bullet}^{\scriptscriptstyle{(k)}}\in\cpl^{\scriptscriptstyle{(k)}}\!\lc m_\bullet^X,m_\bullet^Y\rc$. Let
$$\gamma'\coloneqq\int_{X\times Y}\nu_{x'',y''}^{\scriptscriptstyle{(k)}}\,\gamma(dx''\times dy'').$$ Then,
\begin{align*}
    &\int\limits_{X\times Y}\int\limits_{X\times Y}\int\limits_{X\times Y}\vert d_X(x,x')-d_Y(y,y')\vert\,\nu_{x'',y''}^{\scriptscriptstyle{(k)}}(dx'\times dy')\gamma(dx''\times dy'')\gamma(dx\times dy)\\
    =&\int\limits_{X\times Y}\int\limits_{X\times Y}\vert d_X(x,x')-d_Y(y,y')\vert\,\gamma'(dx'\times dy')\gamma(dx\times dy)\\
\geq& \dGW^{\mathrm{bi}}(\mathbf{X},\mathbf{Y}).
\end{align*}

Since the choice of $\gamma$ and $\nu_{\bullet,\bullet}^{\scriptscriptstyle{(k)}}$ are arbitrary, one concludes that $\dGW^{\scriptscriptstyle{(k)}}\!\lc\mathcal{M}(\mathbf{X}),\mathcal{M}(\mathbf{Y})\rc\geq \dGW^{\mathrm{bi}}(\mathbf{X},\mathbf{Y})$.

Hence, $\dGW^{\scriptscriptstyle{(k)}}\!\lc\mathcal{M}(\mathbf{X}),\mathcal{M}(\mathbf{Y})\rc= \dGW^{\mathrm{bi}}(\mathbf{X},\mathbf{Y})$ as we required.


\subsubsection{Proof of \Cref{prop:lower bound for dGW}}

For any $x\in X$, we let $\ell_X^x\coloneqq d_X(x,\bullet)$. 
For $i=1,\ldots,k$, let $\mathfrak{l}^{\scriptscriptstyle{(i)}}_x:X\rightarrow\prob^{\circ i}(\R)$ be the shorthand for the $i$th WL measure hierarchy $\WLh{k}{(\mX,\ell_X^x)}$ generated from the label $d_X(x,\bullet)$. We similarly define $\ell_Y^y$ and $\mathfrak{l}^{\scriptscriptstyle{(i)}}_y:Y\rightarrow\prob^{\circ i}(\R)$ for any $y\in Y$ and each $i=1,\ldots,k$.

For any $\nu^{\scriptscriptstyle{(k)}}_{\bullet,\bullet}\in\cpl^{\scriptscriptstyle{(k)}}\!\lc m_\bullet^X,m_\bullet^Y\rc$, there exist $(\nu_i)_{\bullet,\bullet}\in\cpl^{\scriptscriptstyle{(1)}}\!\lc m_\bullet^X,m_\bullet^Y\rc$ for $i=1,\ldots,k$ such that
\[\nu_{x,y}^{\scriptscriptstyle{(k)}}=\int\limits_{X\times Y}\cdots\int\limits_{X\times Y}(\nu_k)_{x_{k-1},y_{k-1}}\,(\nu_{k-1})_{x_{k-2},y_{k-2}}(dx_{k-1}\times dy_{k-1})\cdots(\nu_1)_{x,y}(dx_1\times dy_1)\]
for any $x\in X$ and $y\in Y$. Hence, for any $\gamma\in\cpl(\mu_X,\mu_Y)$, we have that
\begin{align*}
    & \int\limits_{X\times Y}\int\limits_{X\times Y}\int\limits_{X\times Y}|d_X(x,x')-d_Y(y,y')|
 \nu_{x'',y''}^{\scriptscriptstyle{(k)}}(dx'\times dy')\gamma(dx''\times dy'')\gamma(dx\times dy)\\
 =& \int\limits_{X\times Y}\cdots\int\limits_{X\times Y}|d_X(x,x')-d_Y(y,y')|\\
 \quad\quad& (\nu_k)_{x_{k-1},y_{k-1}}(dx'\times dy')\cdots(\nu_1)_{x'',y''}(dx_1\times dy_1)\gamma(dx''\times dy'')\gamma(dx\times dy)\\
 \geq&\int\limits_{X\times Y}\cdots\int\limits_{X\times Y}\dW\!\lc (\ell_X^x)_\#m_{x_{k-1}}^X,(\ell_Y^y)_\#m_{y_{k-1}}^Y\rc\\
 \quad\quad&(\nu_{k-1})_{x_{k-2},y_{k-2}}(dx_{k-1}\times dy_{k-1})\cdots(\nu_1)_{x'',y''}(dx_1\times dy_1)\gamma(dx''\times dy'')\gamma(dx\times dy)\\
 =& \int\limits_{X\times Y}\cdots\int\limits_{X\times Y}\dW\!\lc \mathfrak{l}^{\scriptscriptstyle{(1)}}_x(x_{k-1}),\mathfrak{l}^{\scriptscriptstyle{(1)}}_y(y_{k-1})\rc
 \\
 \quad\quad&(\nu_{k-1})_{x_{k-2},y_{k-2}}(dx_{k-1}\times dy_{k-1})\cdots(\nu_1)_{x'',y''}(dx_1\times dy_1)\gamma(dx''\times dy'')\gamma(dx\times dy)\\
 \geq&\cdots\\
 \geq& \int\limits_{X\times Y}\int\limits_{X\times Y}\dW\!\lc \mathfrak{l}^{\scriptscriptstyle{(k)}}_x(x''),\mathfrak{l}^{\scriptscriptstyle{(k)}}_y(y'')\rc
 \gamma(dx''\times dy'')\gamma(dx\times dy)\\
 \geq & \int\limits_{X\times Y}\dWLk\!\lc (\mX,\ell_X^x),(\mY,\ell_Y^y)\rc
 \gamma(dx\times dy)
\end{align*}

Therefore,
\begin{align*}
    &\dGW^{\scriptscriptstyle{(k)}}\!\lc(\mX,d_X),(\mY,d_Y)\rc\\
    &=
 \inf_{\substack{\gamma\in \cpl(\mu_X,\mu_Y)\\\nu^{\scriptscriptstyle{(k)}}_{\bullet,\bullet}\in\cpl^{\scriptscriptstyle{(k)}}\!\lc m_\bullet^X,m_\bullet^Y\rc}}\int\limits_{X\times Y}\int\limits_{X\times Y}\int\limits_{X\times Y}|d_X(x,x')-d_Y(y,y')| \nu_{x'',y''}^{\scriptscriptstyle{(k)}}(dx'\times dy')\gamma(dx''\times dy'')\gamma(dx\times dy)\\
 &\geq \inf_{\gamma\in\cpl(\mu_X,\mu_Y)}\int\limits_{X\times Y}\dWLk\!\lc (\mX,\ell_X^x),(\mY,\ell_Y^y)\rc
 \gamma(dx\times dy).
\end{align*}

\subsubsection{Proof of the statement in \Cref{rmk:TLB}}\label{sec: rmk tlb}
We first recall the third lower bound (TLB) from \cite{memoli2011gromov}:
\[\mathrm{TLB}(\mathbf{X},\mathbf{Y})\coloneqq\inf_{\gamma\in\cpl(\mu_X,\mu_Y)}\int\limits_{X\times Y}\!\lc\inf_{\gamma'\in\cpl(\mu_X,\mu_Y)}\int\limits_{X\times Y}|d_X(x,x')-d_Y(y,y')|\gamma'(dx'\times dy')\rc\gamma(dx\times dy).\]
where we omit the $\frac{1}{2}$ factor from \cite{memoli2011gromov} for simplicity of presentation.

We adopt notation from the previous section. Notice that 
$$\inf_{\gamma'\in\cpl(\mu_X,\mu_Y)}\int\limits_{X\times Y}|d_X(x,x')-d_Y(y,y')|\gamma'(dx'\times dy')=\dWL^{\scriptscriptstyle{(1)}}((\mX,\ell_X^x),(\mY,\ell_Y^y)).$$
We hence have that
\[\mathrm{TLB}(\mathbf{X},\mathbf{Y})=\inf_{\gamma\in\cpl(\mu_X,\mu_Y)}\int\limits_{X\times Y} \dWL^{\scriptscriptstyle{(1)}}((\mX,\ell_X^x),(\mY,\ell_Y^y))\gamma(dx\times dy).\]

For any $k\in\N$, we show that
\[\mathrm{TLB}(\mathbf{X},\mathbf{Y})=\inf_{\gamma\in\cpl(\mu_X,\mu_Y)}\int\limits_{X\times Y} \dWLk((\mX,\ell_X^x),(\mY,\ell_Y^y))\gamma(dx\times dy)\]
by the lemma below.
\begin{lemma}
For any $k\in\N$ we have that $\dWLk((\mX,\ell_X^x),(\mY,\ell_Y^y))=\dWL^{\scriptscriptstyle{(1)}}((\mX,\ell_X^x),(\mY,\ell_Y^y))$.
\end{lemma}
\begin{proof}
By item 2 in \Cref{lemma:mmskfoldcpl} and \Cref{thm:dwl= dwk}, we have that
\begin{align*}
    &\dWLk((\mX,\ell_X^x),(\mY,\ell_Y^y))\\
    &=\inf_{\gamma'\in\cpl^{\scriptscriptstyle{(k)}}(\mu_X,\mu_Y)}\int\limits_{X\times Y}|d_X(x,x')-d_Y(y,y')|\gamma'(dx'\times dy')\\
    &=\inf_{\nu_{\bullet,\bullet}\in\cpl^{\scriptscriptstyle{(k)}}\!\lc m_\bullet^X,m_\bullet^Y\rc,\mu\in\cpl(\mu_X,\mu_Y)}\int\limits_{X\times Y}\int\limits_{X\times Y}|d_X(x,x')-d_Y(y,y')|\nu_{x_1,y_1}(dx'\times dy')\mu(dx_1\times dy_1)\\
    &\leq\inf_{\mu\in\cpl(\mu_X,\mu_Y)}\int\limits_{X\times Y}\int\limits_{X\times Y}|d_X(x,x')-d_Y(y,y')|\mu(dx'\times dy')\mu(dx_1\times dy_1)\\
    &=\inf_{\mu\in\cpl(\mu_X,\mu_Y)}\int\limits_{X\times Y}|d_X(x,x')-d_Y(y,y')|\mu(dx'\times dy')\\
    &=\dWL^{\scriptscriptstyle{(1)}}((\mX,\ell_X^x),(\mY,\ell_Y^y)).
\end{align*}
By \Cref{coro:hierarchy dwlk}, we conclude that $\dWLk((\mX,\ell_X^x),(\mY,\ell_Y^y))=\dWL^{\scriptscriptstyle{(1)}}((\mX,\ell_X^x),(\mY,\ell_Y^y))$.
\end{proof}

\subsubsection{Proof of the statement in \Cref{ex:eccentricity}}\label{sec:proof ecc}
Given $\mX$ and $\mY$, we have for any $\gamma\in\cpl(\mu_X,\mu_Y)$ and $\nu_{\bullet,\bullet}^{\scriptscriptstyle{(k)}}$ that
\begin{align*}
    &\int\limits_{X\times Y}\int\limits_{X\times Y}|\mathrm{ecc}_X(x')-\mathrm{ecc}_Y(y'))|\,\nu_{x,y}^{\scriptscriptstyle{(k)}}(dx'\times dy')\gamma(dx\times dy)\\
    =&\int\limits_{X\times Y}\int\limits_{X\times Y}\left|\int_Xd_X(x',x'')\mu_X(dx'')-\int_Yd_Y(y',y'')\mu_Y(y'')\right|\,\nu_{x,y}^{\scriptscriptstyle{(k)}}(dx'\times dy')\gamma(dx\times dy)\\
    =&\int\limits_{X\times Y}\int\limits_{X\times Y}\left|\int\limits_{X\times Y}d_X(x',x'')\gamma(dx''\times dy'')-\int\limits_{X\times Y}d_Y(y',y'')\gamma(dx''\times dy'')\right|\,\nu_{x,y}^{\scriptscriptstyle{(k)}}(dx'\times dy')\gamma(dx\times dy)\\
    \leq&\int\limits_{X\times Y}\int\limits_{X\times Y}\int\limits_{X\times Y}\left|d_X(x',x'')-d_Y(y',y'')\right|\,\nu_{x,y}^{\scriptscriptstyle{(k)}}(dx'\times dy')\gamma(dx\times dy)\gamma(dx''\times dy'').
\end{align*}
Hence, we conclude that $\mathrm{ecc}_\bullet$ is stable.

\subsubsection{Proof of \Cref{prop:stable WL}}
By \Cref{thm:dwl= dwk} we have that
\begin{align*}
    &\dWL^{\scriptscriptstyle{(k)}}\!\lc(\mX,\ell_X),(\mY,\ell_Y)\rc\\
    &= \inf_{\gamma^{\scriptscriptstyle{(k)}}\in \cpl^{\scriptscriptstyle{(k)}}(\mu_X,\mu_Y)} \int\limits_{X\times Y}d_Z(\ell_X(x),\ell_Y(y))\gamma^{\scriptscriptstyle{(k)}}(dx\times dy)\\
    &=\inf_{\gamma\in \cpl(\mu_X,\mu_Y),\nu^{\scriptscriptstyle{(k)}}_{\bullet,\bullet}\in\cpl^{\scriptscriptstyle{(k)}}\!\lc m_\bullet^X,m_\bullet^Y\rc} \int\limits_{X\times Y}\int\limits_{X\times Y}d_Z(\ell_X(x'),\ell_Y(y'))\,\nu_{x,y}^{\scriptscriptstyle{(k)}}(dx'\times dy')\gamma(dx\times dy)\\
    &\leq\inf_{\gamma\in \cpl(\mu_X,\mu_Y),\nu^{\scriptscriptstyle{(k)}}_{\bullet,\bullet}\in\cpl^{\scriptscriptstyle{(k)}}\!\lc m_\bullet^X,m_\bullet^Y\rc}\mathrm{dis}^{\scriptscriptstyle{(k)}}\!\lc\gamma,\nu^{\scriptscriptstyle{(k)}}_{\bullet,\bullet}\rc\\
    &=\dGW^{\scriptscriptstyle{(k)}}\!\lc(\mX,d_X),(\mY,d_Y)\rc.
\end{align*}
The inequality follows from the fact that $\ell_\bullet$ is stable. Hence we conclude the proof.

\end{document}